%% file: main.tex
\documentclass{article}

\usepackage[final]{corl_2026} 

\usepackage{amsmath}
\usepackage{amssymb}
\usepackage{mathtools}
\usepackage{amsthm}

\usepackage{algorithm}
\usepackage{algpseudocode}

\usepackage{graphicx}
\usepackage{wrapfig}

\usepackage{booktabs}
\usepackage{multirow}

\usepackage{caption}      
\usepackage{subcaption}   

\input{math_commands.tex}

\input{preamble}

\input{macros}

\theoremstyle{plain}
\newtheorem{theorem}{Theorem}[section]
\newtheorem{proposition}[theorem]{Proposition}

\newtheorem{corollary}[theorem]{Corollary}
\theoremstyle{definition}

\theoremstyle{remark}

\usepackage{thm-restate}

\definecolor{darkcyan}{HTML}{0080ff} 
\definecolor{lavender}{HTML}{CC33FF} 
\definecolor{cyan_accent}{HTML}{33E6FF}
\definecolor{electric_green}{HTML}{33CC00}
\definecolor{green_}{HTML}{76B900}
\hypersetup{
    colorlinks=true,  
    linkcolor=lavender,  
    citecolor=green_,  
    urlcolor=black    
}

\usepackage{fancyhdr}

\makeatletter
\newcommand{\listofappendices}{%
  \section*{Appendix Contents}%
  \@starttoc{app}%
}

\newcommand{\appsection}[1]{%
  \section{#1}%
  \addcontentsline{app}{section}{\protect\numberline{\thesection}#1}%
}
\makeatother

\title{BayesFP: Posterior Estimation for Flow-Based Policies via Feynman-Kac Sampling}

\makeatletter
\newcommand{\blfootnote}[1]{%
  \begingroup
  \renewcommand{\@makefntext}[1]{\noindent##1}%
  \renewcommand{\thefootnote}{}%
  \let\@footnotetext\H@@footnotetext%
  \footnotetext{#1}%
  \endgroup
}
\makeatother
\author{%
  Sreevardhan Sirigiri$^{1}$
  \And
  Weiming Zhi$^{1,2,3}$
  \And
  Fabio Ramos$^{1,4}$
}

\usepackage{paper-design}

\begin{document}
\maketitle
\blfootnote{$^{1}$School of Computer Science, The University of Sydney, Australia. \qquad $^{2}$Australian Centre For Robotics, The University of Sydney, Australia \qquad $^{3}$College of Connected Computing, Vanderbilt University,  USA. \qquad $^{4}$NVIDIA, USA.\\
Correspondence to \texttt{ssir4919@uni.sydney.edu.au}.}


\input{Sections/abstract}
\input{Sections/introduction}
\input{Sections/background}
\input{Sections/problem_statement}
\input{Sections/methodology}
\input{Sections/results}
\input{Sections/discussion_and_ablations}
\input{Sections/limitation_and_conclusion}
\newpage




\bibliography{references}  


\clearpage

\appendix
\crefalias{section}{appendix}
\crefalias{subsection}{appendix}
\crefalias{subsubsection}{appendix}
\crefalias{subappendix}{appendix}
\crefalias{subsubappendix}{appendix}
\listofappendices
\input{Appendix/related_work}
\input{Appendix/background}
\input{Appendix/flow_to_diffusion}
\input{Appendix/resampling}
\input{Appendix/algo_discription}
\input{Appendix/proofs}
\input{Appendix/experimental_deatils}
\input{Appendix/notation}

\end{document}

%% file: math_commands.tex
\usepackage{amsmath,amsfonts,bm}

\def\eqref#1{equation~\ref{#1}}

\def\1{\bm{1}}

\DeclareMathAlphabet{\mathsfit}{\encodingdefault}{\sfdefault}{m}{sl}
\SetMathAlphabet{\mathsfit}{bold}{\encodingdefault}{\sfdefault}{bx}{n}

\DeclareMathOperator*{\argmin}{arg\,min}

\usepackage[framemethod=TikZ]{mdframed}
\mdfdefinestyle{MyFrame}{%
    linecolor=black,
    outerlinewidth=0.5pt,
    roundcorner=1pt,
    innertopmargin=2pt, 
    innerbottommargin=2pt, 
    innerrightmargin=7pt,
    innerleftmargin=7pt,
    backgroundcolor=black!0!white}

\mdfdefinestyle{MyFrame2}{%
    linecolor=white,
    outerlinewidth=1pt,
    roundcorner=2pt,
    innertopmargin=10pt,
    innerbottommargin=0.5pt,
    innerrightmargin=10pt,
    innerleftmargin=10pt,
    backgroundcolor=black!3!white}

\mdfdefinestyle{MyFrameEq}{%
    linecolor=white,
    outerlinewidth=0pt,
    roundcorner=0pt,
    innertopmargin=0pt,
    innerbottommargin=0pt,
    innerrightmargin=7pt,
    innerleftmargin=7pt,
    backgroundcolor=black!3!white}

\mdfdefinestyle{FrameAdded}{%
    linecolor=black,
    outerlinewidth=0pt,
    roundcorner=0pt,
    innertopmargin=0pt,
    innerbottommargin=0pt,
    innerrightmargin=0pt,
    innerleftmargin=0pt,
    backgroundcolor=PastelGreen}

\definecolor{customwhite}{HTML}{FCFBF7}
\definecolor{customturq}{HTML}{1D9D79}
\definecolor{customorange}{HTML}{D96002} 
\definecolor{custombeige}{HTML}{d6c9b1} 

\definecolor{custompurple}{HTML}{AEADF0}
\definecolor{customwhite2}{HTML}{fbf9f4}
\definecolor{customblue}{HTML}{4D9DE0}

\definecolor{myred}{RGB}{215,48,39}
\definecolor{mygreen}{RGB}{26,152,80}
\newcommand{\maybe}{\textcolor{gray}{\checkmark\kern-1.1ex\raisebox{.7ex}{\rotatebox[origin=c]{125}{--}}}}

%% file: preamble.tex
\usepackage{silence}
\usepackage[utf8]{inputenc}
\usepackage[T1]{fontenc}
\usepackage{tabularx}
\usepackage{lipsum}
\usepackage{url}
\usepackage{placeins}
\usepackage{booktabs}
\usepackage{soul}
\usepackage{amsfonts}
\usepackage{nicefrac}
\usepackage{microtype}
\usepackage{wrapfig}
\usepackage{caption}
\usepackage{amssymb}
\usepackage{pifont}

\usepackage{subcaption}
\usepackage{float}
\usepackage{rotating}
\usepackage{etoolbox}
\usepackage{xparse}
\usepackage{grffile}

\usepackage{amsmath}
\usepackage{xspace}
\usepackage{euscript}
\usepackage{enumitem}
\usepackage{graphicx}
\usepackage{mathtools}
\usepackage{tikz}
\usepackage{tablefootnote}
\usepackage[flushleft]{threeparttable}
\usepackage{multirow}
\usepackage[title]{appendix}
\usepackage{longtable} 
\usepackage{array, booktabs}
\usepackage{graphicx}
\usepackage{glossaries}
\usepackage{amsmath,amsthm}
\usepackage{dsfont}
\usepackage{balance}
\usepackage{multicol}
\usepackage{nameref}
\usepackage{pdfpages}
\usepackage{braket}
\usepackage[normalem]{ulem}
\usepackage{bbding}
\usepackage[capitalise, noabbrev, nameinlink]{cleveref}   
\usepackage{xfrac}
\usepackage[usestackEOL]{stackengine}
\usepackage{indentfirst}
\usepackage{csquotes}
\usepackage{pgf}
\usepackage{varwidth}
\usepackage{verbatimbox}
\usepackage{verbatim}
\usepackage{bm}
\usepackage{anyfontsize}
\usepackage{tcolorbox}
\usepackage{nccmath}
\usepackage{makecell}
\usepackage{colortbl}

\definecolor{plotgreen}{RGB}{72, 164, 135}   
\definecolor{plotorange}{RGB}{222, 111, 68}  
\definecolor{plotpurple}{RGB}{111, 130, 173}

\usepackage{empheq}
\definecolor{myblue}{rgb}{.8, .8, 1}

\definecolor{pastelblue}{RGB}{76,113,175}
\definecolor{pastelgreen}{RGB}{144,238,144}
\definecolor{pastelred}{RGB}{196,78,82}
\definecolor{pastelgrey}{RGB}{230,230,230}
\definecolor{pastelbeige}{RGB}{243,236,221}
\definecolor{pastelpurple}{RGB}{154,139,192}

\definecolor{salmon}{RGB}{250, 128, 114}
\definecolor{darkgreen}{rgb}{0,0.6,0}
\definecolor{darkred}{rgb}{0.5,0,0}
\definecolor{verylightgreen}{HTML}{F6FFF9}
\definecolor{verylightred}{HTML}{FFF4F3}
\definecolor{verylightgray}{HTML}{F4F6F6}
\definecolor{babyblueeyes}{rgb}{0.63, 0.79, 0.95}
\definecolor{lightpink}{rgb}{1.00, 0.714, 0.757}

\definecolor{customwhite}{HTML}{FCFBF7}
\definecolor{customturq}{HTML}{1D9D79}
\definecolor{customorange}{HTML}{D96002} 
\definecolor{custombeige}{HTML}{d6c9b1} 

\definecolor{custompurple}{HTML}{AEADF0}
\definecolor{highlightturq}{HTML}{92CEBD}
\definecolor{highlightpurple}{HTML}{AEADF0}
\definecolor{highlightorange}{HTML}{ECAF80}
\definecolor{highlightgray}{HTML}{E6E6E6}
\definecolor{customblue}{HTML}{4D9DE0} 

\usepackage{soul}

\definecolor{some_color}{HTML}{86CFDA}
\colorlet{PastelGreen}{some_color!50!white}
\sethlcolor{PastelGreen}

\usepackage{tikzpeople}
\usetikzlibrary{shapes,decorations,arrows,calc,arrows.meta,fit,positioning}

\tikzset{
    -Latex,auto,node distance =1 cm and 1 cm,semithick,
    state/.style ={ellipse, draw, minimum width = 0.7 cm},
    point/.style = {circle, draw, inner sep=0.04cm,fill,node contents={}},
    bidirected/.style={Latex-Latex,dashed},
    el/.style = {inner sep=2pt, align=left, sloped}
}

\makeatletter
\def\thmt@refnamewithcomma #1#2#3,#4,#5\@nil{%
	\@xa\def\csname\thmt@envname #1utorefname\endcsname{#3}%
	\ifcsname #2refname\endcsname
	\csname #2refname\expandafter\endcsname\expandafter{\thmt@envname}{#3}{#4}%
	\fi}
\makeatother
\Crefname{conjecture}{Conjecture}{Conjectures}
\Crefname{definition}{Definition}{Definitions}
\Crefname{observation}{Observation}{Observations}
\Crefname{assumption}{Assumption}{Assumptions}
\Crefname{axiom}{Axiom}{Axioms}
\Crefname{case}{Case}{Cases}
\Crefname{claim}{Claim}{Claims}
\Crefname{conclusion}{Conclusion}{Conclusions}
\Crefname{condition}{Condition}{Conditions}
\Crefname{criterion}{Criterion}{Criteria}
\Crefname{exercise}{Exercise}{Exercises}
\Crefname{example}{Example}{Examples}
\Crefname{notation}{Notation}{Notations}
\Crefname{problem}{Problem}{Problems}
\Crefname{property}{Property}{Properties}
\Crefname{remark}{Remark}{Remarks}
\Crefname{solution}{Solution}{Solutions}
\Crefname{summary}{Summary}{Summaries}
\Crefname{motivation}{Motivation}{Motivations}
\Crefname{query}{Query}{Queries}

\newcommand*\dbar[1]{\overline{\overline{\lower0.2ex\hbox{$#1$}}}}

\def\cX{{\mathcal{X}}}

\DeclareFontFamily{U}{BOONDOX-calo}{\skewchar\font=45 }
\DeclareFontShape{U}{BOONDOX-calo}{m}{n}{
  <-> s*[1.05] BOONDOX-r-calo}{}
\DeclareFontShape{U}{BOONDOX-calo}{b}{n}{
  <-> s*[1.05] BOONDOX-b-calo}{}
\DeclareMathAlphabet{\mathcalb}{U}{BOONDOX-calo}{m}{n}
\SetMathAlphabet{\mathcalb}{bold}{U}{BOONDOX-calo}{b}{n}
\DeclareMathAlphabet{\mathbcalb}{U}{BOONDOX-calo}{b}{n}

\usepackage{cancel}

\renewcommand{\paragraph}[1]{{\noindent \textbf{#1.}}}

%% file: macros.tex
\let\originalleft\left
\let\originalright\right
\renewcommand{\left}{\mathopen{}\mathclose\bgroup\originalleft}
\renewcommand{\right}{\aftergroup\egroup\originalright}

\global\long\def\norm#1{\left\lVert #1\right\rVert }
\global\long\def\inner#1#2{\left\langle \vphantom{1^2} #1, #2\right\rangle}

\newacronym{SNIS}{\textsc{snis}}
{Self-Normalized Importance Sampling}
\newacronym{SMC}{\textsc{smc}}
{Sequential Monte Carlo}
\newacronym{CFG}{\textsc{cfg}}
{classifier-free guidance}
\newacronym{PDE}{\textsc{pde}}
{Partial Differential Equations}
\newacronym{sde}{\textsc{sde}}
{Stochastic Differential Equations}
\newacronym{ode}{\textsc{ode}}
{Ordinary Differential Equations}
\makeglossaries 
\glsdisablehyper

\definecolor{antiquefuchsia}{rgb}{0.57, 0.36, 0.51}
\definecolor{amethyst}{rgb}{0.6, 0.4, 0.8}

\newcommand{\deriv}[2]{\frac{\partial #1}{\partial #2}}

\newcommand{\mean}{\mathbb{E}}
\newcommand{\var}{{\rm I\kern-.3em D}}
\newtheorem*{theorem*}{Theorem}
\newtheorem*{lemma*}{Lemma}
\newtheorem*{proposition*}{Proposition}
\newtheorem*{corollary*}{Corollary}

\newtheorem*{example*}{Example}
\DeclareMathSymbol{\shortminus}{\mathbin}{AMSa}{"39}

\crefname{equation}{Eq.}{Eqs.}
\crefname{section}{Sec.}{Secs.}
\crefname{corollary}{Cor.}{Cors.}
\crefname{proposition}{Prop.}{Props.}
\crefname{appendix}{App.}{Apps.}
\Crefname{appendix}{App.}{Apps.}
\crefformat{appendix}{#2App.\,#1#3}
\Crefformat{appendix}{#2App.\,#1#3}
\crefname{theorem}{Thm.}{Thms.}
\crefname{figure}{Fig.}{Figs.}
\crefname{tabular}{Tab.}{Tabs.}

\newcommand{\ifbm}[1]
{#1}

\newcommand{\bbR}{\mathbb{R}}

\tcbuselibrary{theorems}
\newtcbtheorem[number within=section]{myprop}{Proposition}%
{colback=blue!5,colframe=blue!35!black,fonttitle=\bfseries}{th}

\definecolor{mybabyblue}{HTML}{89CFF0} 
\colorlet{babybluelight}{mybabyblue!25!white} 

\tcbset{mybox/.style={colback=lavender, width =\textwidth, boxrule=0.0pt, top=5pt,bottom=5pt,left=2pt, right=2pt},
  before upper=\setlength{\parindent}
  {0em}\everypar{{\setbox0\lastbox}\everypar{}}
}
\newtcolorbox{mybox}[1][]{mybox,#1}

\tcbset{halfmybox/.style={colback=olive!3, width =0.49\textwidth, boxrule=0.0pt, top=5pt,bottom=5pt,left=2pt, right=2pt},
  before upper=\setlength{\parindent}
  {0em}\everypar{{\setbox0\lastbox}\everypar{}}
}
\newtcolorbox{halfmybox}[1][]{halfmybox,#1}

\definecolor{pastel_purple}{HTML}{756FB3}
\definecolor{pastel_green}{HTML}{1D9D79}

\colorlet{PastelPurpleLight}{pastel_purple!15!white}
\colorlet{PastelGreenLight}{pastel_green!15!white}
\sethlcolor{PastelPurpleLight}
\sethlcolor{PastelGreenLight}

\makeatletter
\let\save@mathaccent\mathaccent
\newcommand*\if@single[3]{%
  \setbox0\hbox{${\mathaccent"0362{#1}}^H$}%
  \setbox2\hbox{${\mathaccent"0362{\kern0pt#1}}^H$}%
  \ifdim\ht0=\ht2 #3\else #2\fi
  }
\newcommand*\rel@kern[1]{\kern#1\dimexpr\macc@kerna}
\newcommand*\widebar[1]{\@ifnextchar^{{\wide@bar{#1}{0}}}{\wide@bar{#1}{1}}}
\newcommand*\wide@bar[2]{\if@single{#1}{\wide@bar@{#1}{#2}{1}}{\wide@bar@{#1}{#2}{2}}}
\newcommand*\wide@bar@[3]{%
  \begingroup
  \def\mathaccent##1##2{%
    \let\mathaccent\save@mathaccent
    \if#32 \let\macc@nucleus\first@char \fi
    \setbox\z@\hbox{$\macc@style{\macc@nucleus}_{}$}%
    \setbox\tw@\hbox{$\macc@style{\macc@nucleus}{}_{}$}%
    \dimen@\wd\tw@
    \advance\dimen@-\wd\z@
    \divide\dimen@ 3
    \@tempdima\wd\tw@
    \advance\@tempdima-\scriptspace
    \divide\@tempdima 10
    \advance\dimen@-\@tempdima
    \ifdim\dimen@>\z@ \dimen@0pt\fi
    \rel@kern{0.6}\kern-\dimen@
    \if#31
      \overline{\rel@kern{-0.6}\kern\dimen@\macc@nucleus\rel@kern{0.4}\kern\dimen@}%
      \advance\dimen@0.4\dimexpr\macc@kerna
      \let\final@kern#2%
      \ifdim\dimen@<\z@ \let\final@kern1\fi
      \if\final@kern1 \kern-\dimen@\fi
    \else
      \overline{\rel@kern{-0.6}\kern\dimen@#1}%
    \fi
  }%
  \macc@depth\@ne
  \let\math@bgroup\@empty \let\math@egroup\macc@set@skewchar
  \mathsurround\z@ \frozen@everymath{\mathgroup\macc@group\relax}%
  \macc@set@skewchar\relax
  \let\mathaccentV\macc@nested@a
  \if#31
    \macc@nested@a\relax111{#1}%
  \else
    \def\gobble@till@marker##1\endmarker{}%
    \futurelet\first@char\gobble@till@marker#1\endmarker
    \ifcat\noexpand\first@char A\else
      \def\first@char{}%
    \fi
    \macc@nested@a\relax111{\first@char}%
  \fi
  \endgroup
}
\makeatother

\definecolor{customwhite}{HTML}{FCFBF7}
\definecolor{customturq}{HTML}{1D9D79}
\definecolor{customorange}{HTML}{D96002} 
\definecolor{custombeige}{HTML}{d6c9b1} 
\definecolor{customblue}{HTML}{4D9DE0}

\newcommand{\cut}[1]{}

\AtBeginEnvironment{theorem}{\crefalias{theorem}{theorem}}
\AtBeginEnvironment{proposition}{\crefalias{theorem}{proposition}}
\AtBeginEnvironment{corollary}{\crefalias{theorem}{corollary}}
\AtBeginEnvironment{lemma}{\crefalias{theorem}{lemma}}

%% file: Sections/abstract.tex
\begin{abstract}
Robots must generate trajectories that remain faithful to learned expert behavior while satisfying safety constraints and task-specific objectives specified only at inference time. We formulate constrained trajectory generation for pretrained diffusion and flow-matching policies as Bayesian posterior sampling, with the learned demonstration distribution as a prior and an inference-time, cost-derived likelihood tilting it toward feasible, optimal trajectories. To sample from this posterior without any retraining of the base policy, we leverage the Feynman--Kac corrector framework, originally formulated for diffusion models, and extend it to deterministic flow-matching policies. 
The result is a unified, inference-time, retraining-free sampler for diffusion and flow policies. We validate the approach on pretrained Diffusion Policy, GR00T-N1.6, and $\pi_{0.5}$ checkpoints across simulated and real-world manipulation tasks, including planning around non-convex obstacles introduced at inference time, and show improvements over the base \(\pi_{0.5}\) on zero-shot tasks. We open sourced the implementation to provide a reference in this \href{https://github.com/S-sirigiri/BayesFP}{\textcolor{PaperTeal}{GitHub repository}}.
\end{abstract}

\keywords{Feynman-Kac PDEs, Constrained Generative Policies, Robot Trajectory Generation}

%% file: Sections/introduction.tex
\section{Introduction}


Robots operating in the real world must generate motions that are both 
expressive enough to capture multimodal task distributions and reactive 
enough for closed-loop control. Recent advances in generative modeling 
have made imitation learning a compelling paradigm for synthesizing such 
motions, with diffusion-based 
policies~\cite{janner2022diffuser, chi2024diffusionpolicy} and flow-matching 
policies~\cite{lipman2023flow,fmPolicy2,jiang2025streaming} demonstrating 
strong performance across manipulation and locomotion tasks. By learning 
to denoise or transport noise into expert demonstrations, these models 
faithfully capture the rich, multimodal structure of human or expert 
behavior.

However, deployment in the real world demands more than imitation 
fidelity. Robots must respect safety constraints, feasibility limits, 
and task-specific objectives that are often unknown at training time and 
may change between deployments. Joint limits, obstacle avoidance, 
self-collision, workspace boundaries, energy consumption costs and goal-reaching costs are 
typically specified at inference, and a pretrained policy must be steered 
to satisfy them without sacrificing the behavioral richness it learned 
from data. Visuomotor policies dominate because directly conditioning on 
point clouds tends to degrade the control frame rate; however, depth 
cameras are commonly available in robotics setups, making environment 
geometry from depth sensing a natural fit for specifying such constraints 
at inference. The central question we address is therefore: \emph{given a 
pretrained diffusion/flow-matching policy or VLAs with flow-based heads, how can we sample 
trajectories that respect a constrained optimization problem specified 
only at inference time, while remaining faithful to the learned 
demonstration distribution?}


Existing approaches to this problem fall broadly into two categories. 
\emph{Trajectory optimization and post-hoc filtering} methods rectify 
unconstrained rollouts using gradient-based planners or control-theoretic 
safety filters~\cite{CHOMP,trajopt,cbf_intV}; while effective in some 
settings, decoupling generation from correction tends to produce 
out-of-distribution trajectories, suffer from local minima, and incur 
high computational overhead. \emph{Generative guidance} methods, 
including classifier-based 
guidance~\cite{dhariwal2021diffusion}, control barrier functions based guidance~\cite{xiao2023safediffuser}, 
projection~\cite{DPwConstarints25}, reflected dynamics~\cite{lou2023reflected, 
liu2023mirror}, and joint sampling with auxiliary 
modules~\cite{jung2025joint}, modify the sampling process directly. 
While these are an improvement over post-hoc approaches, they all share 
a fundamental limitation: they are heuristic modifications to the 
sampling dynamics. None of them samples from a well-defined posterior 
distribution that has a principled relationship to both the data prior 
and the constraint.


In this work, we take a different perspective. We pose constrained 
trajectory generation as \emph{posterior sampling}: given a pretrained 
diffusion or flow policy whose marginal $p_{\text{data}}(x)$\footnote{
We abuse notation: $p_{\text{data}}(x)$ stands for 
$p_{\text{data}}(x \mid o)$, where $o$ is the current sensor observation 
(typically RGB cameras). This conditioning is always present and we suppress it 
throughout for brevity.} represents 
the demonstration distribution (over control inputs; such as action deltas 
or joint positions), and given a cost function $\mathcal{J}(x)$ 
encoding the constrained optimization problem at inference time, we 
target to sample from the posterior
\begin{equation}
\label{eq:main_posterior}
p(x \mid O = 1) \propto p_{\text{data}}(x)\,
\exp(\beta\, \mathcal{J}(x)),
\end{equation}

\begin{wrapfigure}{r}{0.60\textwidth}
  \vspace{-1.5em}
  \centering
  \PaperFrame{PaperTeal}{\includegraphics[width=\linewidth]{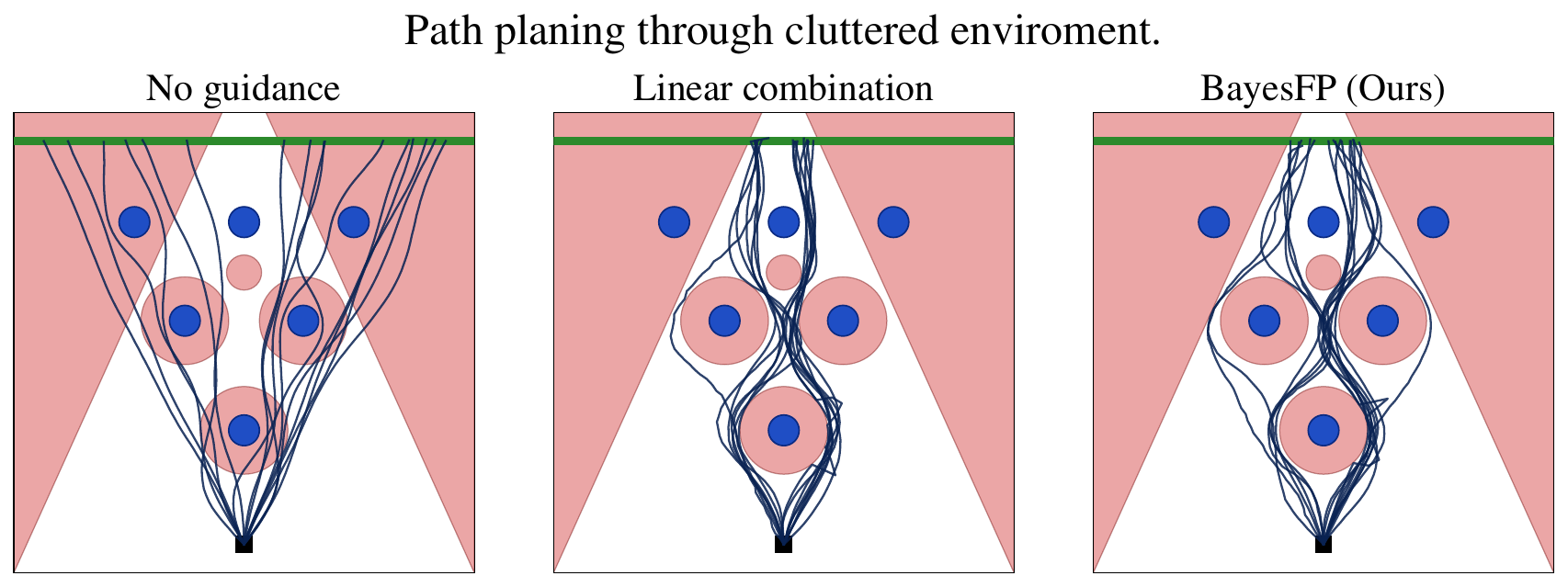}}\\[0.5em]
  \PaperFrame{PaperTeal}{\includegraphics[width=\linewidth]{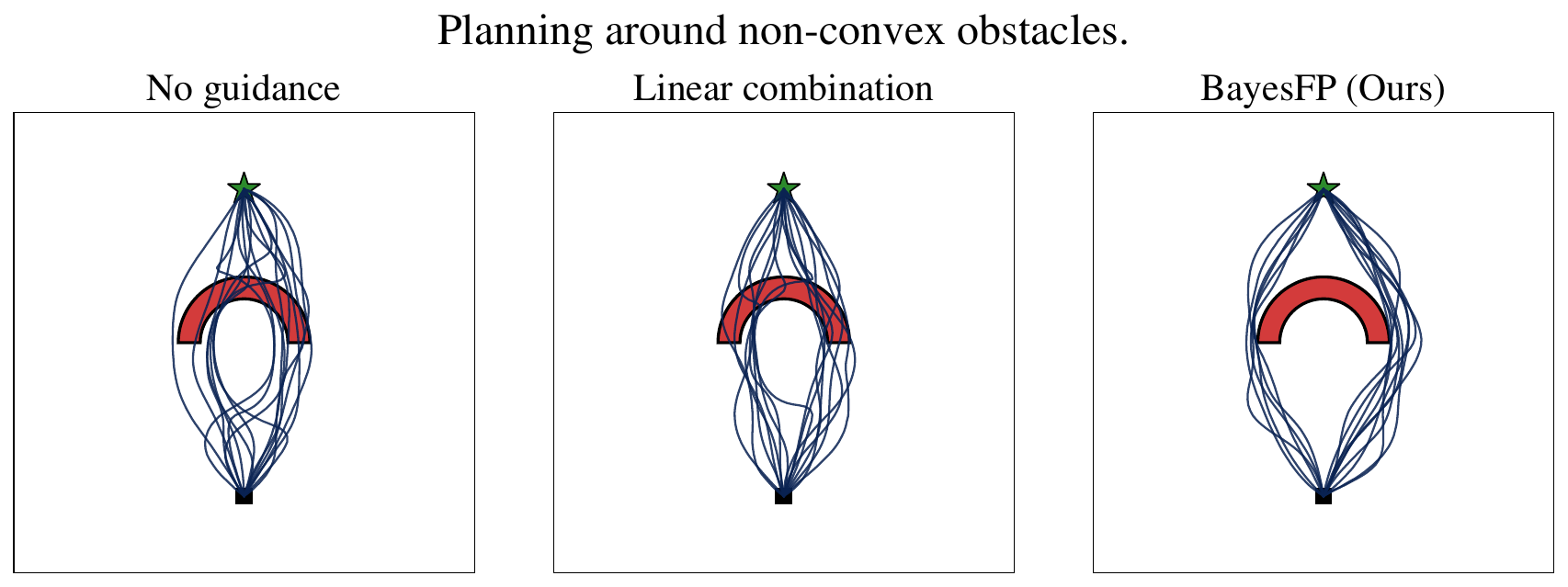}}
  \caption{%
    Toy 2D environments with obstacle constraints. Top: the diffusion model is trained on demonstrations that avoid the blue circular regions, while the red obstacle constraints are introduced only at inference time. We compare no guidance, a naive linear-combination baseline that replaces the learned score $\nabla\log p_t(x)$ by $\nabla\log p_t(x) + \lambda_t \nabla \mathcal J(x)$, and our BayesFP (Bayes Flow-based Policies). Bottom: the model is trained to avoid a small square but is evaluated with a non-convex obstacle at inference time. BayesFP handles these constraints by sampling from the cost-tilted posterior in \cref{eq:main_posterior} via the corresponding Feynman--Kac weighted dynamics, rather than applying a post-hoc correction or an unprincipled score--cost-gradient linear combination.%
  }
  \label{fig:toy-2d}
  \vspace{-1.5em}
\end{wrapfigure}

where ${O}$ is a binary optimality indicator and $\beta < 0$ 
controls the strength of the cost-induced tilt. This is precisely the 
distribution one obtains by treating the pretrained policy as a Bayesian 
prior over expert trajectories and the cost as a likelihood on 
optimality. Sampling from this posterior favors trajectories that 
already lie under the learned demonstration distribution while 
exponentially preferring those that better satisfy the constrained 
optimization problem. Crucially, this is not heuristic guidance: it is 
the unique distribution consistent with the Bayesian interpretation of 
constraint-aware imitation.

\textbf{The contributions of this paper are as follows:}
    (1) We formulate constrained trajectory generation from a 
    pretrained generative policy as Bayesian inference, with 
    the demonstration distribution as prior and a cost-derived 
    likelihood encoding the inference-time constrained optimization 
    problem---offering a principled alternative to heuristic guidance 
    and post-hoc projection.
    (2) We instantiate this posterior as a Feynman-Kac (FK) weighted 
    SDE (a special case of~\citet{skreta2025fkc}) whose reweighting 
    term is derived from the Fokker-Planck PDE of the target, so that 
    the weighted particle population is an asymptotically unbiased 
    estimator of the posterior even when the drift correction is 
    imperfect. The FK weights are not heuristic 
    corrections but principled terms derived from the Fokker-Planck PDE.
    We further extend this construction (originally 
    formulated only for stochastic diffusion processes) to 
    deterministic flow-matching policies by recasting the flow ODE as 
    a marginal-preserving SDE from~\citet{singh2024stochastic} with the 
    score imputed in closed form from the learned velocity field, 
    yielding a unified inference-time sampler for both model classes.
    (3) We provide a theoretical guarantee that, with sufficiently 
    large penalty weights and inverse temperature, the resulting 
    posterior concentrates on trajectories that are arbitrarily close 
    to feasible and arbitrarily close to optimal, removing the need 
    for backup planner. We further establish a non-asymptotic guarantee 
    for its practical Monte Carlo approximation, showing that, for 
    fixed algorithmic parameters, the additional error 
    due to a finite sampling budget decreases inversely 
    with the number of samples used.
    (4) We empirically demonstrate, across a range of simulated and 
    real-world robotic manipulation tasks, that our sampler produces 
    trajectories that satisfy inference-time constraints. 

\textbf{Paper Organization:} The remainder of the numbered sections constitutes the main paper, where we focus on the core problem formulation, methodology, and empirical evaluation. The rest is deferred to the appendix, which contains supplementary material such as proofs, implementation details, additional experimental information, related work(s), and a detailed notation guide/glossary.

%% file: Sections/background.tex
\section{Background \& Preliminaries}
\subsection{Diffusion Policy} \label{sec:back_diff_policy}
In imitation learning, diffusion models \cite{chi2024diffusionpolicy} can be used to represent the policy as a 
generative model over expert demonstrations, formulated via the simulation of the 
Stochastic Differential Equation (SDE) corresponding to the reverse-time process.
In particular, during training, one gradually corrupts samples from the expert 
demonstration distribution $p_{\text{data}}(x)$ by simulating the following 
noising SDE:
\begin{align}
    dx_\tau = f_\tau(x_\tau)d\tau + \sigma_\tau d\widebar{W}_\tau\,, 
    ~~ x_{\tau = 0} \sim p_{\text{data}}(x)\,,
    ~~ \tau \in [0,1],
    \label{eq:noising}
\end{align}
where $f_\tau(x_\tau)$ is usually some linear drift function 
$f_\tau(x_\tau) = \alpha_\tau x_\tau$, $\sigma_\tau$ defines the scale of 
injected noise through time, and $d\widebar{W}_\tau$ is the standard Wiener 
process. The drift $f_\tau$ and the diffusion coefficient $\sigma_\tau$ are 
chosen so that the final density is close to the standard normal distribution 
$p_{\tau=1} \approx \mathcal{N}(0, I_d)$. At inference time, the learned 
reverse-time process is simulated conditioned on the current observation to 
produce a trajectory for the robot to execute.

The trajectory generation process can then be defined as the family of denoising 
SDEs in the opposite time direction ($t = 1-\tau$),
\begin{align}
    dx_t =~& \left(-f_t(x_t) + \sigma_t^2\nabla\log p_t(x_t)\right)dt 
    + \sigma_t dW_t\,,
    ~~ t \in [0,1], 
    \label{eq:denoising}
\end{align}
where $p_{t} = p_{1-\tau}$ is the density of the marginals induced by the 
noising process in \cref{eq:noising}; hence, the process starts with 
$x_0 \sim \mathcal{N}(x \mid 0, I)$ and is simulated forward in time to 
produce a denoised trajectory $x_{t=1}$ conditioned on the current robot 
observation. By training a model of the score functions $\nabla \log p_t(\cdot)$, 
one can generate new samples from $p_{\text{data}}(x)$ using \cref{eq:denoising} \citep{song2021scorebased}.

\subsection{(Rectified) Flow Matching Policy}
\label{sec:rect_flow}

Once again, let $x \sim p_{\text{data}}(x)$ 
denote expert demonstrations 
(e.g., end-effector trajectories or joint-space paths) drawn from the 
distribution we wish to learn and sample from. Let $x \sim p_{\text{src}}(x) 
\in \mathbb{R}^d$ be a tractable source distribution, typically chosen as a 
standard Gaussian. The key idea behind flow-based models in the context of 
imitation learning is to learn a deterministic mapping that iteratively 
transforms samples from $p_{\text{src}}$ into demonstrations consistent with 
$p_{\text{data}}$, conditioned on the current robot observation \cite{braun2024riemannianflowmatchingpolicy}.

Let $\nu(x_{\text{data}}, x_{\text{src}})$ be a coupling distribution for 
demonstrations and source samples such that 
$p_{\text{data}}(x) = \int \nu(x_{\text{data}}, x_{\text{src}})\,
dx_{\text{src}}$ and $p_{\text{src}}(x) = \int \nu(x_{\text{data}}, 
x_{\text{src}})\,dx_{\text{data}}$. A simple choice is the independent 
coupling $\nu(x_{\text{data}}, x_{\text{src}}) \equiv 
p_{\text{data}}(x_{\text{data}})p_{\text{src}}(x_{\text{src}})$.

To construct a rectified flow, one first defines a time-differentiable 
interpolation $x_t \equiv h(x_{\text{data}}, x_{\text{src}}, t)$, $t \in 
[0, 1]$, between the two endpoints. The default interpolation 
\citep{liu2022flow} is: \(x_t = (1-t)\,x_{\text{data}} + t\, x_{\text{src}}, \quad t \in [0, 1]\,,\)
so that $x_{t=0} = x_{\text{data}}$ corresponds to an expert demonstration 
and $x_{t=1} = x_{\text{src}} \sim \mathcal{N}(0, I_d)$ corresponds to pure 
noise. Rectified flow then learns a vector field $v_t(x_t)$ by minimizing:
\begin{equation}
\label{eq:rectflow_obj}
v_t(x) = \argmin_{v'}\, \mathbb{E}_{(x_{\text{data}},\, x_{\text{src}})
\sim \nu} \left[ \int_0^1 \left\lVert \frac{dx_t}{dt} - v_t'(x_t) 
\right\rVert^2 dt \right].
\end{equation}
The solution is $v_t(x) \equiv \mathbb{E}[x_{\text{src}} - x_{\text{data}} 
\mid x, t]$, referred to as 1-Rectified flow. Since $v_t(x)$ is not 
available in closed form in general, $v$ is parameterized with parameters 
$\theta$ and \cref{eq:rectflow_obj} is optimized with respect to $\theta$; 
we drop this dependence in notation henceforth and assume a learned model 
$v_t(x)$ to be given.

At inference time, a demonstration trajectory for the robot to execute is 
produced by drawing $x_{t=1} \sim \mathcal{N}(0, I_d)$ and simulating the 
flow \textit{backward} in $t$, i.e., from $t=1$ to $t=0$, via: \(dx = v_t(x)\,dt\,.\)

\subsection{Feynman-Kac Correctors}
\label{sec:fkc}

While \cref{eq:denoising} describes the standard denoising process that 
generates trajectories distributed according to $p_{\text{data}}(x)$, in 
practice we may wish to steer the generation process toward a 
\textit{modified} target distribution at inference time — for instance, 
to compose multiple pretrained policies or to bias generation toward 
high-reward regions of the trajectory space. The Feynman-Kac Corrector 
(FKC) framework \citep{skreta2025fkc} provides a principled mechanism for 
doing so by augmenting the denoising SDE with additional reweighting terms 
derived from the Feynman-Kac formula.

The key insight is that the time-evolution of a density under an SDE can be 
decomposed into three distinct mechanisms, each governed by a corresponding 
PDE: a \textit{continuity equation} capturing deterministic transport, a 
\textit{diffusion equation} capturing stochastic spreading, and a 
\textit{reweighting equation} capturing importance reweighting of samples
(see~\cref{app:add_background}). 
Together, these yield the \emph{Feynman-Kac PDE},
\begin{equation}
    \deriv{p_t^{\textsc{fk}}(x)}{t} =
    -\inner{\nabla}{p_t^{\textsc{fk}}(x)v_t(x)} 
    + \frac{\sigma_t^2}{2}\Delta p_t^{\textsc{fk}}(x) \nonumber\\
    + \bar{g}_t(x)p_t^{\textsc{fk}}(x)\,,
    \label{eq:fk_pde}
\end{equation}
which governs the evolution of some \textit{weighted} density  
$p_t^{\textsc{fk}}(x)$ (in our case a modified target distribution). 
To sample from $p_t^{\textsc{fk}}$, one simulates 
the following weighted SDE,
\begin{align}
    dx_t = v_t(x_t)dt + \sigma_t dW_t\,,\quad 
    dw_t = \bar{g}_t(x_t)\,dt\,,
\label{eq:weighted_sde}
\end{align}
where $w_t$ are log-weights carried by each simulated trajectory, and 
$\bar{g}_t(x) = g_t(x) - \int g_t(x) p_t^{\textsc{fk}}(x)\,dx$ is a 
centered reweighting function that ensures normalization is preserved,
and the evolution of \(dx_t\) is given by the usual diffusion SDE. 
The reweighting term $g_t(x)$ encodes the discrepancy between the 
SDE being simulated and the true dynamics of the target distribution 
$p_t^{\textsc{fk}}$; by accumulating these corrections into weights $w_t$, 
the weighted particle population approximates the desired target.

In practice, the weighted samples from \cref{eq:weighted_sde} can be used 
to estimate expectations under $p_t^{\textsc{fk}}$ via Self-Normalized 
Importance Sampling (SNIS). Concretely, given a batch of $K$ simulated 
trajectories $\{x_t^{(k)}, w_t^{(k)}\}_{k=1}^K$, expectations of an arbitrary test function \(\varphi\) (usually, \(\varphi(x) \equiv x\)) under the 
target distribution can be approximated as
\begin{align}
    \mathbb{E}_{p_t^{\textsc{fk}}}[\varphi(x_t)] \approx \sum_{k=1}^K 
    \frac{\exp(w_T^{(k)})}{\sum_j \exp(w_T^{(j)})} \varphi(x_T^{(k)})\,,
    \label{eq:snis}
\end{align}
which converges to the true expectation as $K \to \infty$ 
\citep{skreta2025fkc, naesseth2019elements} (also see~\cref{app:expecation_estimation}). 
However, in practice, this estimator suffers from 
high variance due to weight degeneracy — over long simulation horizons, a 
small number of particles tend to accumulate the majority of the weight, 
leading to poor coverage of the target distribution. To address this, we 
introduce some recommended resampling methods in \cref{sec:resampling}.
Given parallel simulation of K particles and efficient resampling, FKCs 
add negligible slowdown.

%% file: Sections/problem_statement.tex
\section{Problem Statement}

During inference (i.e., the reverse-time flow/diffusion process), we seek to guide sampling so that the generated samples satisfy the constrained optimization problem
\begin{equation}
    \underset{x \in \mathcal{X}}{\text{minimize}} \quad \mathcal{L}(x) \qquad
    \text{subject to} \quad h_1(x) = 0, \quad h_2(x) \le 0 .
\label{eq:optimzation_prob}
\end{equation}  
We can introduce a binary indicator
\(
O:\mathcal{X}\to\{0,1\},
\)
which flags samples that satisfy the desired optimality criterion as defined in~\cref{eq:optimzation_prob}. We define a penalty-augmented cost as \footnote{Here \(\text{softmax}\{0,h_2(x)\}\) is any \(C^2\) approximation of \(\max\{0,h_2(x)\}\). We use a softmax relaxation for analytical and mathematical tractability; empirically, using the maximum operator performs without issue.}
\begin{equation}
\label{eq:agumented_cost_function}
    \mathcal J(x)
    :=\mathcal L(x)
    +\rho(x)
    +c_1\,h_1(x)^2
    +c_2\,\text{softmax}\{0,h_2(x)\},
\end{equation}
where \(\rho(x)\) is any other additional penalty.
The likelihood of observing $O=1$ given a trajectory $x$ is defined as
\(
p\bigl(O=1 \mid x\bigr)
=\exp{\!\bigl(\beta\,\mathcal{J}(x)\bigr)},
\)
where $\beta<0$ controls the strength of the preference for low-cost samples. Combining this likelihood with a prior $p(x)$ over trajectories via Bayes' rule yields the posterior distribution
\(
p(x \mid O=1)\;\propto\; p(x)\,p(O=1\mid x).
\)
The formulation above is an instance of \emph{control as (variational) inference}, which recasts constrained optimization as posterior inference. A key advantage is that it naturally captures multi-modality in the solution space (usually arising due to non-convexity of constraints) while keeping samples close to the prior \cite{lambert2021steinvariationalmodelpredictive, DBLP:journals/corr/abs-1805-00909, rawlik2013}.
Now given a pretrained diffusion or flow policy with marginals \(q_t(x)\), our goal is to sample from the time-\(t\) posterior marginals (or often referred to as the Gibbs tilted marginals):
\begin{equation}
    p_t(x \mid O=1)=Z_t^{-1} \cdot q_t(x)\exp\bigl(\beta\mathcal{J}(x)\bigr),
    \qquad
    Z_t:=\int q_t(\tilde{x})\exp\bigl(\beta\mathcal{J}(\tilde{x})\bigr)\,d\tilde{x}.
    \label{eq:postior_marginals}
\end{equation}

%% file: Sections/methodology.tex
\section{Methodology}

\subsection{Guiding Diffusion Policy to Satisfy the Optimization Problem}

In the context of imitation learning, the FKC framework is particularly 
attractive: given a pretrained diffusion policy generating trajectories from 
$p_{\text{data}}(x)$, FKC allows one to sample at inference time from 
modified distributions without any retraining. Therefore, we can use the FKC framework 
in order to sample from the marginals \(p_t(x)\) from~\cref{eq:postior_marginals}
using the following Feynman-Kac PDE.\footnote{These continuous-time dynamics/PDEs can be readily discretized and simulated using any suitable numerical integration scheme (e.g., Euler--Maruyama discretization).}

\PaperStatementBegin{PaperPlum}
\begin{theorem}[Gibbs Tilted Sampling; see~\cref{thm:gibbs_tilted_full} in~\cref{app:proofs} for proof]
\label{thm:cost_guided_sde}
    Consider the following SDE
    \begin{align}
        dx_t = \left(-f_t(x_t) + \sigma_t^2 \nabla \log q_t(x_t)\right)dt + \sigma_tdW_t\,,
        \label{eq:base_stochastic_PDE}
    \end{align}
    which samples from the marginals $q_t(x)$. The samples from the Gibbs titled marginals $p_t(x \mid O=1) \propto q_t(x) \exp{(\beta_t \mathcal{J}(x))}$ from \cref{eq:postior_marginals} can be simulated according to the following SDE
    \begin{subequations}\label{eq:fkc-weight}
    \begin{align}
    \label{eq:fkc-weight-x}
    dx_t =~& (-f_t(x_t) + \sigma_t^2 \nabla \log q_t(x_t) + \beta_t \frac{\sigma_t^2}{2}\nabla \mathcal J(x_t))dt + \sigma_t dW_t\,,\\
    \label{eq:fkc-weight-w}
    dw_t =~& \left[\deriv{\beta_t}{t}\mathcal J(x) + \inner{\beta_t\nabla \mathcal J(x)}{\frac{\sigma_t^2}{2}\nabla\log q_t(x) -f_t(x)}\right]dt.
    \end{align}
    \end{subequations}
\PaperStatementEnd
\end{theorem}    
\PaperStatementClose

\subsection{Guiding Flow Matching Policy to Satisfy the Optimization Problem}
\label{sec:flow_guidance}


The reader may observe that only the (reverse-time) diffusion SDE in 
\cref{eq:denoising} conforms to the class of SDEs defined in 
\cref{eq:base_stochastic_PDE}. Naively identifying the flow ODE 
\(dx = -v_{1-t}(x)\,dt\,\) with \cref{eq:base_stochastic_PDE} corresponds to
setting $\sigma_t = 0$, making drift guidance disappear in  
\cref{thm:cost_guided_sde} (while reweighting remains valid); both the $\nabla \mathcal{J}$ 
correction in the drift and the score-correction in the weight evolution 
vanish identically.

Remarkably, this limitation can be circumvented by recasting the 
deterministic flow as an equivalent stochastic process with identical 
marginal densities, following \citet{singh2024stochastic}. Specifically, 
for any scalar function $\tilde{\sigma}(t) \geq 0$, the flow ODE 
$dx_t = -v(x_t, 1-t)\,dt$ shares its time-marginals with the family of SDEs,
\begin{equation}
dx_t = \left( -v(x_t, 1-t) + \tfrac{\tilde{\sigma}^2(t)}{2}\nabla \log q_t(x_t)
\right)dt + \tilde{\sigma}(t)\,dW_t, 
\label{eq:flow_sde}
\end{equation}
where $\nabla \log q_t$ is the score of the flow's marginals, available in 
closed form from the velocity field itself when $p_{\text{src}}$ is Gaussian 
(see \cref{app:flow_to_sde} for full details). \cref{eq:flow_sde} 
now matches the form of \cref{eq:base_stochastic_PDE} with non-zero 
diffusion $\tilde{\sigma}(t)$, so we may directly apply the FKC framework 
of \cref{thm:cost_guided_sde}.

\subsection{Concentration on the Constrained Optimum}

A natural question is whether sampling from the Gibbs tilted posterior in 
\cref{eq:postior_marginals} actually recovers solutions to the original constrained 
optimization problem in \cref{eq:optimzation_prob}. The following result answers 
this in the affirmative: for sufficiently large penalty weights and inverse 
temperature $\beta$, the posterior concentrates arbitrarily tightly on 
feasible, near-optimal trajectories.

\PaperStatementBegin{PaperTeal}
\begin{corollary}[Concentration on the constrained optimum; see~\cref{prop:approx-constrained-full} in~\cref{app:proofs}] \label{col:postirior_bound}
For any tolerance $\delta>0$, there exist constants $\varepsilon, C>0$ (depending on 
the penalty weights $c$) such that the Gibbs-tilted posterior satisfies
\begin{equation}
\label{eq:concentration}
    \mathbb{P}_{p_{\beta,c}}\!\left(\,|h_1(x)|\le\delta,\; h_2(x)\le\delta,\; \mathcal{L}(x)\le \mathcal{L}^\star+\delta\,\right) \;\ge\; 1 - C\, e^{\beta\varepsilon/2}.
\end{equation}
That is, the probability of $\delta$-infeasibility or $\delta$-suboptimality decays exponentially as $\beta\to-\infty$.
\PaperStatementEnd
\end{corollary}
\PaperStatementClose
In essence, $\beta$ tunes how strongly we weight constraint satisfaction and cost minimization relative to the baseline (data) distribution.

\subsection{Finite-Particle Guarantee}
\label{sec:finite-particle-guarantee}

\cref{col:postirior_bound} describes concentration of the Gibbs posterior.
BayesFP approximates this posterior using a finite population of
weighted trajectories. The following result relates the particle
budget to the probability that the returned trajectory is feasible
and near-optimal.

\PaperStatementBegin{PaperTeal}
\begin{corollary}[Finite-particle concentration; see \cref{thm:bayesfp-finite-particle} in \cref{app:proofs}]
Consider BayesFP with  $K$ particles. Fix $\delta>0$ and penalties $c$ as
in \cref{col:postirior_bound}, together with $\beta<0$, a numerical grid of $N$
steps, and a deterministic resampling schedule, independently of $K$.
Let $X_{\mathrm{out}}^K$ denote the trajectory selected using the
final normalized particle weights.
There exists a finite constant $B_N\geq0$ and $d_N$, independent of $K$, 
such that
\begin{equation}
\begin{aligned}
& \mathbb P\!\left(
 |h_1(X_{\mathrm{out}}^K)|\leq\delta,\;
 h_2(X_{\mathrm{out}}^K)\leq\delta,\;
 \mathcal L(X_{\mathrm{out}}^K)\leq\mathcal L^\star+\delta
\right)
& \geq
1-Ce^{\beta\varepsilon/2}-d_N-\frac{B_N}{K}.
\end{aligned}
\label{eq:finite-particle-main}
\end{equation}
The probability includes all particle-system randomness and the
final trajectory selection.
\PaperStatementEnd
\end{corollary}
\PaperStatementClose

For fixed temperature, penalties, grid, and resampling schedule,
the finite-particle contribution to the bound decreases as $K^{-1}$.

%% file: Sections/results.tex
\section{Experiments and Results}
\label{sec:experiments}

In this section, we provide analysis and validation of our proposed 
approach. While our framework applies to any problem expressible in the 
form of \cref{eq:optimzation_prob} or \cref{eq:main_posterior}, we focus our motivation and 
empirical validation on some practically important cases: 
    (1) Can our approach improve the performance of pretrained 
    Vision-Language-Action (VLA) models such as 
    GR00T-N1.6~\cite{nvidia2025gr00tn1openfoundation} or 
    $\pi_{0.5}$~\cite{intelligence2025pi05visionlanguageactionmodelopenworld} 
    by introducing collision constraints at inference time? Additionally, 
    can our approach plan around non-convex obstacles?
    (2)  How does our approach compare to other baselines?

\paragraph{Baselines}
We compare against three baselines spanning the main families of 
inference-time constrained generation.
\textbf{Linear-Combination Guidance:} A simple baseline that adds the 
gradient of the augmented cost \(\mathcal{J}(x)\) to the unconstrained 
generative signal: \(\nabla \log q_t(x) + \lambda_t \nabla \mathcal{J}(x)\) 
for diffusion policies, and \(v_t(x) + \lambda_t \nabla \mathcal{J}(x)\) 
for flow policies. It is simple but lacks any principled correction for 
the change of measure, however it is common choice used in highly effective
guidance methods in flow/diffusion based generative policies (e.g. OmniGuide~\cite{song2026omniguideuniversalguidancefields}, Guided Diffusion in Inference-Time Policy Steering~\cite{wang2025inferencetimepolicysteeringhuman}).
\textbf{CASF~\cite{long2026constrainingstreamingflowmodels}:} Constraint-Aware Streaming 
Flow models each constraint as a differentiable distance function and 
converts it into a Riemannian metric that locally reshapes the learned 
velocity field near constraint boundaries.\footnote{CASF is flow-specific and 
\emph{cannot be applied to diffusion policies}; we evaluate it only in 
the flow setting.}
\textbf{JM2D~\cite{jung2025joint}:} Joint Model-based Model-free 
Diffusion frames constrained planning as joint sampling over \((x, k)\), 
where \(k\) is the output of an auxiliary optimization module, with 
compatibility encouraged by an interaction potential \(V(x, k)\). The 
joint score is approximated via importance sampling on denoised samples 
at each reverse step.\footnote{JM2D is originally proposed for diffusion; for 
flow policies, we first convert the flow into an equivalent SDE 
(\cref{app:flow_to_sde}) and then apply JM2D to the resulting SDE.}
We additionally note that JM2D is directly benchmarked against several 
closely related constrained diffusion planners, including SafeDiffuser~\cite{xiao2023safediffuser}, 
DPCC~\cite{DPwConstarints25}, MPD~\cite{carvalho2024motion}, etc., in its original evaluation. Thus, we use JM2D as a strong 
representative baseline here and refer the reader to \citet{jung2025joint} (JM2D) 
for broader comparisons with these methods.

\subsection{Collision Avoidance in Simulation} \label{sec:sim_exps}

\begin{table}[h]
\centering
\caption{Collision ($\downarrow$) and success ($\uparrow$) rates (\%) across simulation tasks; best per task in \textbf{bold}. For $\pi_{0.5}$, collision is reported per-timestep rather than per-trajectory because zero-shot evaluation yields high collision rates.}
\label{tab:main_col_results}
\setlength{\tabcolsep}{4pt}
\renewcommand{\arraystretch}{1.05}
\footnotesize
\resizebox{\linewidth}{!}{%
\begin{tabular}{@{}l l cccc
  >{\columncolor{PaperLilac}}c cccc
  >{\columncolor{PaperLilac}[\tabcolsep][0pt]}c@{}}
\toprule
\PaperTableHead
\multicolumn{1}{@{}>{\columncolor{PaperIce}[0pt][\tabcolsep]}l}{} & & \multicolumn{5}{c}{Collision Rate $\downarrow$} & \multicolumn{5}{c}{Success Rate $\uparrow$} \\
\cmidrule(lr){3-7} \cmidrule(lr){8-12}
\PaperTableHead
\multicolumn{1}{@{}>{\columncolor{PaperIce}[0pt][\tabcolsep]}l}{Policy} & Task
 & No Guid. & Lin. Comb. & CASF & JM2D & \textbf{BayesFP (Ours)}
 & No Guid. & Lin. Comb. & CASF & JM2D & \textbf{BayesFP (Ours)} \\
\midrule
\multirow{2}{*}{\shortstack[l]{Diffusion Policy}}
 & Can + cyl. obstacle            & 30\% & 2\% & \emph{n/a} & 2\% & \textbf{0\%}  & \textbf{96\%} & \textbf{96\%} & \emph{n/a} & \textbf{96\%} & \textbf{96\%} \\
 & Transport + cyl. obstacle      & 94\% & \textbf{4\%} & \emph{n/a} & 6\% & \textbf{4\%}  & 12\% & 38\% & \emph{n/a} & 46\% & \textbf{56\%} \\
\midrule
\multirow{2}{*}{\shortstack[l]{GR00T-N1.6}}
 & LIBERO-Obj. + cyl. obstacle    & 100\% & 2\% & 4\% & \textbf{0\%} & \textbf{0\%}  & 58\% & 72\% & 76\% & \textbf{84\%} & \textbf{84\%} \\
 & LIBERO-Obj. + V-shape obstacle & 48\%    & 43\% & 44\% & 25\% & \textbf{14\%}  & 42\%   & 72\%   & 69\% & \textbf{75\%} & \textbf{75\%}   \\
\midrule
\multirow{4}{*}{$\pi_{0.5}$}
 & FoodPacking2Cans               & 4.7\% & 2.3\% & \textbf{0.9\%} & 2.7\% & 3.0\% & 14\% & 22\% & 8\% & 22\% & \textbf{28\%} \\
 & BBQSauceInBin                  & 20.2\% & 19.1\% & 18.0\% & 19.0\% & \textbf{17.9\%} & 8\% & 8\% & 8\% & 13\% & \textbf{16\%} \\
 & HammersInLeftBin               & 4.1\% & 3.9\% & 3.4\% & 3.2\% & \textbf{2.8\%} & 13\% & 27\% & 15\% & 28\% & \textbf{36\%} \\
 & TakeMugsOffOfShelf             & \textbf{1.1\%} & 7.7\% & 4.2\% & 7.1\% & 9.9\% & 0\% & 2\% & 1\% & 6\% & \textbf{10\%} \\
\bottomrule
\end{tabular}%
}
\end{table}

Across all experiments in this 
section, we enforce collision constraints by sampling points on the 
robot body and penalizing trajectories whose sampled body points fall 
within an obstacle region. The experiments below differ only in how 
this obstacle region is specified. In the first two cases 
(Diffusion Policy, GR00T-N1.6), we assume access to 
the ground-truth pose and dimensions of the obstacle, and the obstacle 
region is defined directly from this geometry. In the case of \(\pi_{0.5}\), we 
do not assume any such access; instead, we construct a signed distance 
function (SDF)~\cite{millane2024nvbloxgpuacceleratedincrementalsigned} 
of the environment online, treat every object that is 
not the target of manipulation or grasping as a collidable obstacle, 
and define the obstacle region as the zero-level set of this SDF.

\paragraph{Guiding Diffusion Policies}
\label{sec:diffusion_policy}
We begin with experiments on vanilla diffusion policies, where we 
demonstrate the ability of our method to enforce collision constraints 
on pretrained models without any retraining. We consider two settings 
of increasing complexity. First, in \cref{fig:toy-2d}, we present 
two illustrative experiments that highlight the qualitative behavior of 
our Gibbs tilted posterior sampler under collision constraints, 
including the case of planning around non-convex obstacles. 
Second, we introduce 
an additional cylindrical obstacle into the \emph{Can and Transport 
tasks} from RoboMimic~\cite{mandlekar2021what}, again using the 
pretrained diffusion policy checkpoints from 
\citet{chi2024diffusionpolicy}; the results are summarized in 
\cref{tab:main_col_results} and he scenes are visualized in \cref{fig:Libero_robomimic_scene}. 
Across all experiments in this subsection, the policies predict 
\emph{end-effector position trajectories in the world frame} (i.e., 
absolute Cartesian poses).

\paragraph{Guiding GR00T-N1.6}
\label{sec:groot}
We fine-tune GR00T-N1.6~\cite{nvidia2025gr00tn1openfoundation} on the 
\emph{LIBERO object tasks}~\cite{liu2023liberobenchmarkingknowledgetransfer}, 
and then evaluate under two separate inference-time settings: one in 
which we introduce a translucent cylindrical obstacle, and another in 
which we introduce a translucent non-convex V-shaped obstacle formed by 
two thick arms joined at an angle (a "V" in the XY plane, elongated 
along the Z axis). The policy outputs end-effector action deltas in 
the world frame. Quantitative results are reported in \cref{tab:main_col_results}, 
and the scenes are visualized in \cref{fig:Libero_robomimic_scene}.

\begin{figure}[htbp]
    \centering
    \begin{subfigure}[t]{0.20\linewidth}
        \centering
        \PaperPhotoFrame{PaperTeal}{\includegraphics[width=\linewidth]{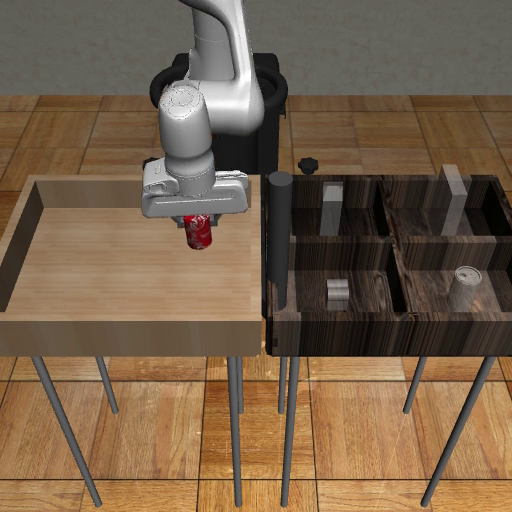}}
        \caption{Robomimic Can task}
        \label{fig:sub1}
    \end{subfigure}
    \hfill
    \begin{subfigure}[t]{0.20\linewidth}
        \centering
        \PaperPhotoFrame{PaperTeal}{\includegraphics[width=\linewidth]{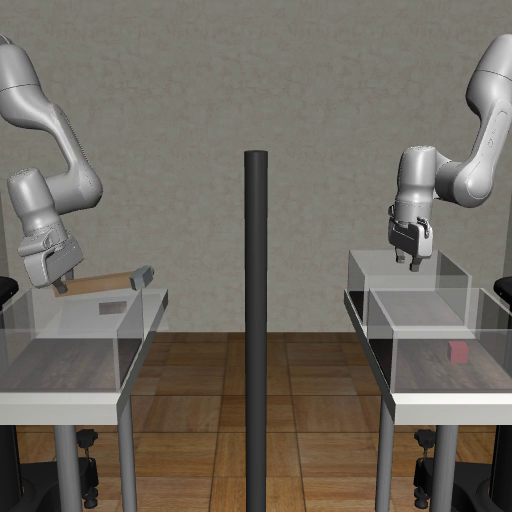}}
        \caption{Robomimic Transport task}
        \label{fig:sub2}
    \end{subfigure}
    \hfill
    \begin{subfigure}[t]{0.20\linewidth}
        \centering
        \PaperPhotoFrame{PaperTeal}{\includegraphics[width=\linewidth]{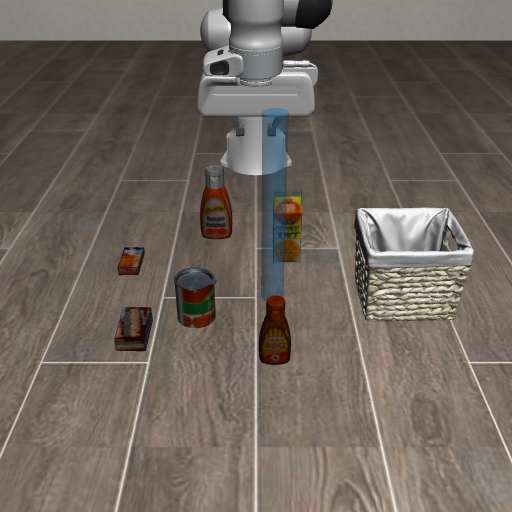}}
        \caption{LIBERO-Obj. + cyl. obstacle}
        \label{fig:sub3}
    \end{subfigure}
    \hfill
    \begin{subfigure}[t]{0.20\linewidth}
        \centering
        \PaperPhotoFrame{PaperTeal}{\includegraphics[width=\linewidth]{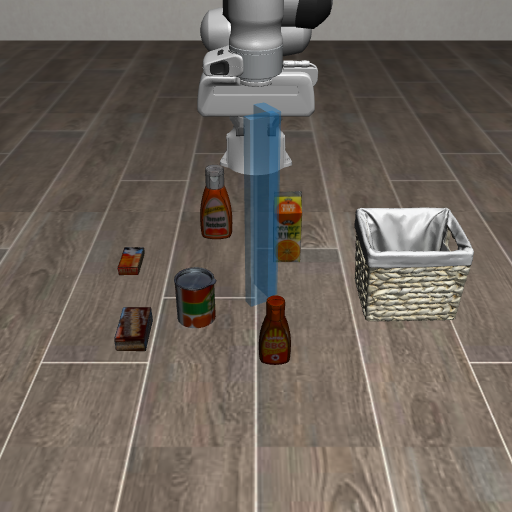}}
        \caption{LIBERO-Obj. + V-shape obstacle}
        \label{fig:sub4}
    \end{subfigure}
    \caption{Scene visualization of the Robomimic and LIBERO tasks. }
    \label{fig:Libero_robomimic_scene}
\end{figure}

\paragraph{Guiding $\pi_{0.5}$}
\label{sec:pi05}
We use the pretrained 
$\pi_{0.5}$~\cite{intelligence2025pi05visionlanguageactionmodelopenworld} 
checkpoint fine-tuned on the DROID 
dataset~\cite{khazatsky2025droidlargescaleinthewildrobot}, and evaluate 
it on four tasks from the RoboLab~\cite{yang2026robolab} benchmark that 
were never seen during $\pi_{0.5}$ training: \emph{FoodPacking2CansTask}, 
\emph{BBQSauceInBinTask}, \emph{HammersInLeftBinTask}, and 
\emph{TakeMugsOffOfShelfTask}. Here, the policy outputs joint-position 
commands rather than end-effector poses or deltas as in the previous 
experiments. Quantitative results are reported in \cref{tab:main_col_results}, and 
the scenes are visualized in \cref{fig:robolab_scene}. 

\begin{figure}[htbp]
    \centering
    \begin{subfigure}[t]{0.22\linewidth}
        \centering
        \PaperPhotoFrame{PaperTeal}{\includegraphics[width=\linewidth]{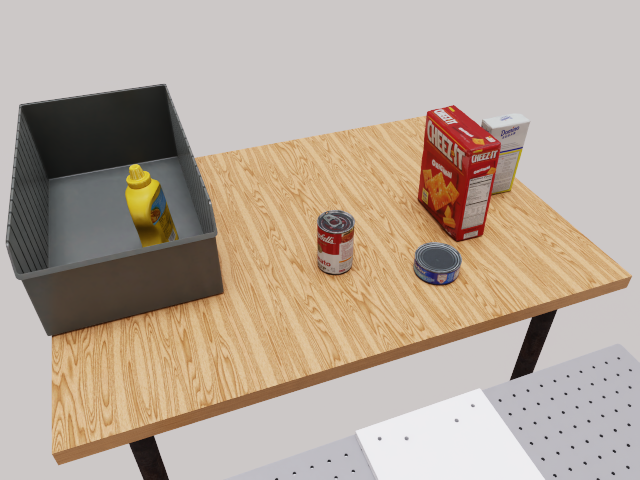}}
        \caption{FoodPacking2CansTask}
        \label{fig:robolab:foodpacking}
    \end{subfigure}
    \hfill
    \begin{subfigure}[t]{0.22\linewidth}
        \centering
        \PaperPhotoFrame{PaperTeal}{\includegraphics[width=\linewidth]{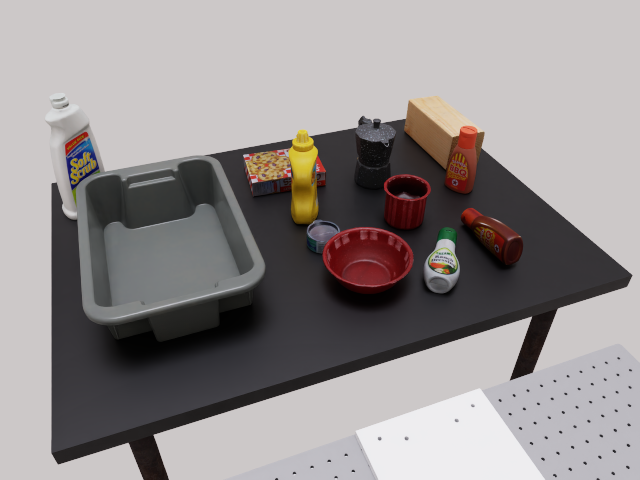}}
        \caption{BBQSauceInBinTask}
        \label{fig:robolab:bbq}
    \end{subfigure}
    \hfill
    \begin{subfigure}[t]{0.22\linewidth}
        \centering
        \PaperPhotoFrame{PaperTeal}{\includegraphics[width=\linewidth]{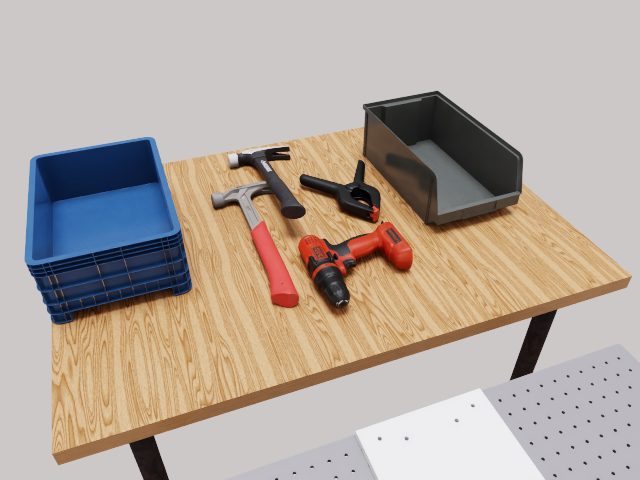}}
        \caption{HammersInLeftBinTask}
        \label{fig:robolab:hammers}
    \end{subfigure}
    \hfill
    \begin{subfigure}[t]{0.22\linewidth}
        \centering
        \PaperPhotoFrame{PaperTeal}{\includegraphics[width=\linewidth]{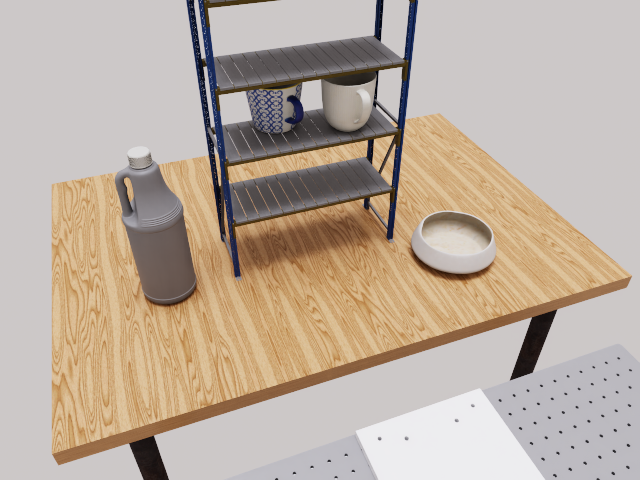}}
        \caption{TakeMugs\allowbreak OffOfShelf\allowbreak Task}
        \label{fig:robolab:mugs}
    \end{subfigure}
    \caption{Scene visualization of the Robolab tasks. }
    \label{fig:robolab_scene}
\end{figure}

\subsection{Real-world Experiments}
\begin{figure}[h]
  \centering
  \begin{minipage}[t]{0.55\linewidth}
    \centering
    \vspace{0pt}
    \begin{subfigure}[t]{0.45\linewidth}
      \centering
      \PaperPhotoFrame{PaperTeal}{\includegraphics[width=\linewidth]{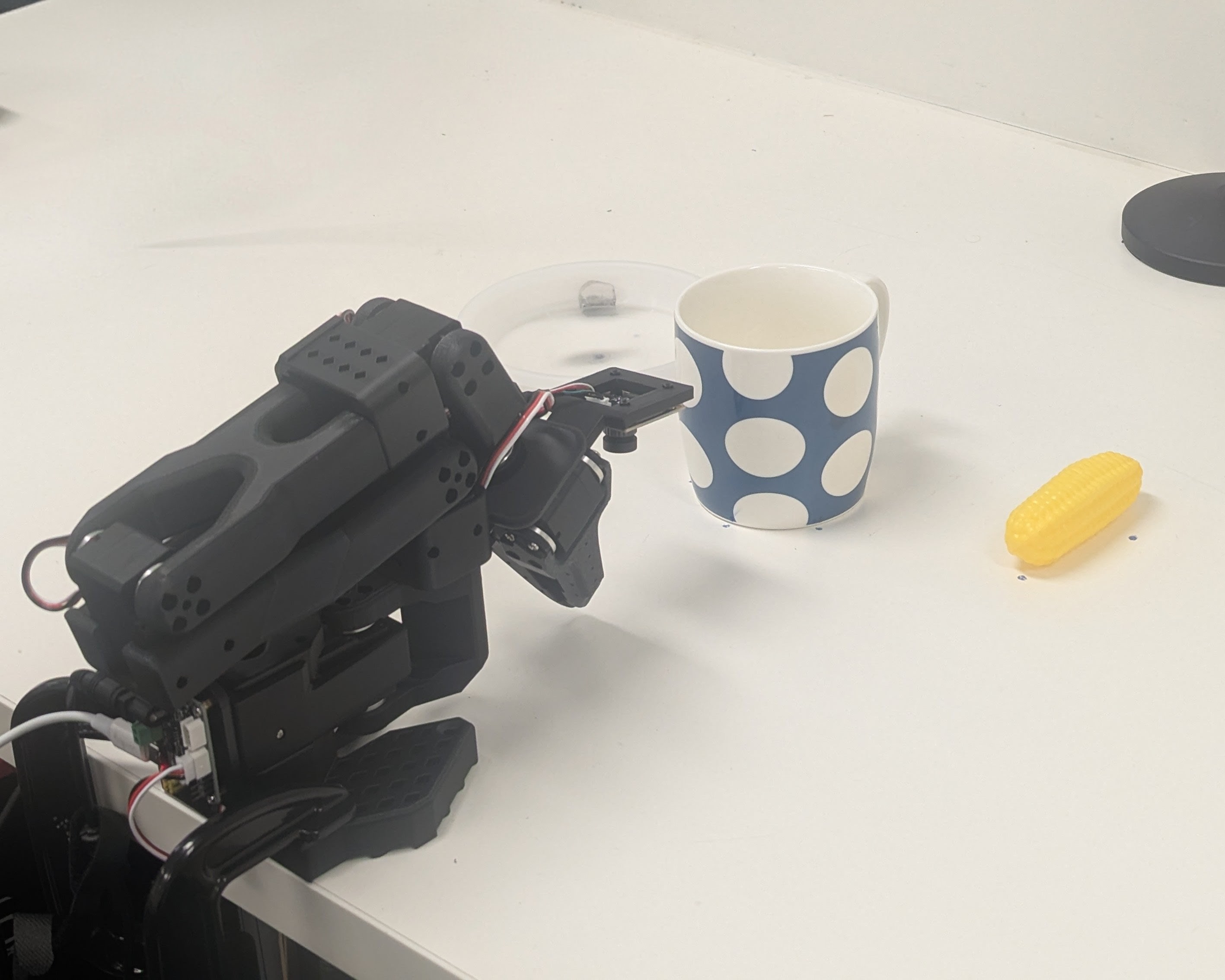}}
      \caption{PickAndPlace}
      \label{fig:real:mug}
    \end{subfigure}\hfill
    \begin{subfigure}[t]{0.45\linewidth}
      \centering
      \PaperPhotoFrame{PaperTeal}{\includegraphics[width=\linewidth]{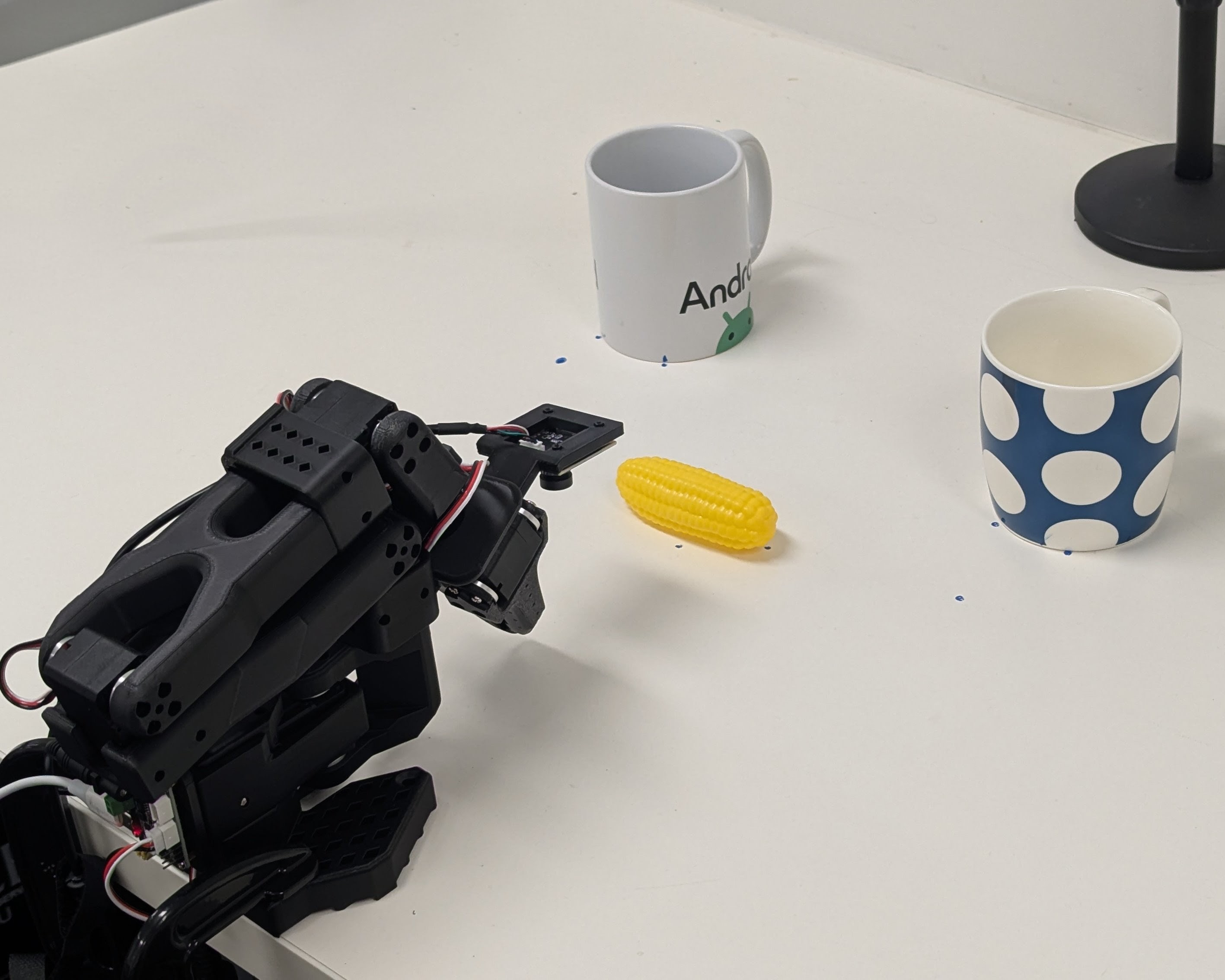}}
      \caption{BimodalMugSelection}
      \label{fig:real:yellow}
    \end{subfigure}
    \caption{Scene visualization of the real-world tasks}
    \label{fig:real_experiments}
  \end{minipage}\hfill
  \begin{minipage}[t]{0.36\linewidth}
    \centering
    \vspace{0pt}
    \small
    \setlength{\tabcolsep}{2pt}
    \renewcommand{\arraystretch}{1.1}
    \begin{tabular}{@{}>{\columncolor{white}[0pt][\tabcolsep]}lc>{\columncolor{PaperLilac}[\tabcolsep][0pt]}c@{}}
      \toprule
      \PaperTableHead
       & No Guid. & BayesFP \\
      \midrule
      \multicolumn{3}{@{}>{\columncolor{PaperTeal!4}[0pt][0pt]}l@{}}{\textit{PickAndPlace}} \\
      \midrule
      Collision rate $\downarrow$ & 98\% & \textbf{6\%} \\
      Success rate  $\uparrow$    & 100\% & \textbf{100\%} \\
      \midrule
      \multicolumn{3}{@{}>{\columncolor{PaperTeal!4}[0pt][0pt]}l@{}}{\textit{BimodalMugSelection}} \\
      \midrule
      Correct-mug rate $\uparrow$ & 72\% & \textbf{100\%} \\
      Success rate     $\uparrow$ & 64\% & \textbf{92\%} \\
      \bottomrule
    \end{tabular}
    \captionof{table}{Real-world results}
    \label{tab:real_results}
  \end{minipage}
\end{figure}

We additionally conduct two real-world experiments on an SO101 arm~\citep{so101_github}, each fine-tuned \(\pi_{0.5}\)with only 20 demonstrations: (i) a \emph{PickAndPlace} task into a saucer, with a mug introduced as an obstacle at inference time; and (ii) a \emph{BimodalMugSelection} task whose training data places the object into either of two mugs (left and right), but at inference we language-condition on the right mug and add an inference-time constraint forbidding the left. Quantitative results for all tasks are reported in \cref{tab:real_results}, and the scenes are visualized in \cref{fig:real_experiments}.

%% file: Sections/discussion_and_ablations.tex
\section{Discussion and Ablations}
\label{sec:discussion}

\textbf{Ablations.} We ablate two knobs of our sampler (\cref{fig:ablations}). Increasing the number of FK particles $K$ improves success rate as the SNIS estimator tightens, but the gain saturates by $K\!\approx\!4$. Increasing $|\beta|$ monotonically reduces constraint violations, matching the exponential concentration predicted by \cref{eq:concentration}.

\begin{figure}[!htbp]
  \centering
  \begin{minipage}[t]{0.725\linewidth}
  \vspace{0pt}
  \centering
  \begin{subfigure}[t]{0.47\linewidth}
  \centering
  \PaperFrame{PaperTeal}{\includegraphics[width=\linewidth]{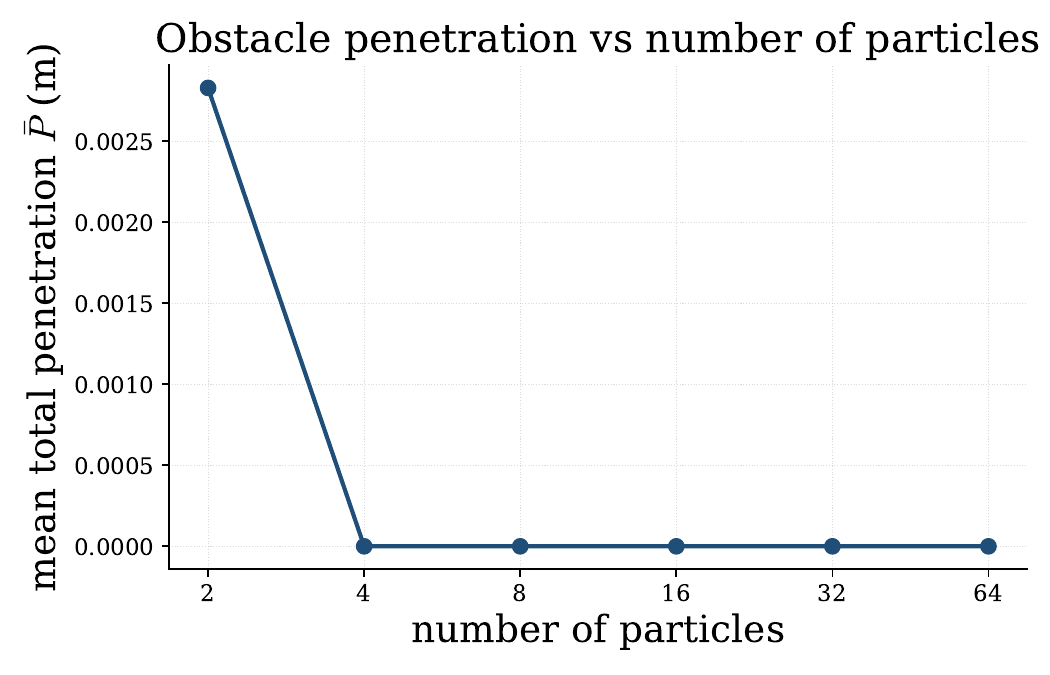}}
  \caption{Particle count $K$: gains saturate by $K=4$ here, but we recommend $K \geq 8$ for non-convex obstacles and real-world tasks.}
  \label{fig:ex:A}
\end{subfigure}\hfill
\begin{subfigure}[t]{0.47\linewidth}
  \centering
  \PaperFrame{PaperTeal}{\includegraphics[width=\linewidth]{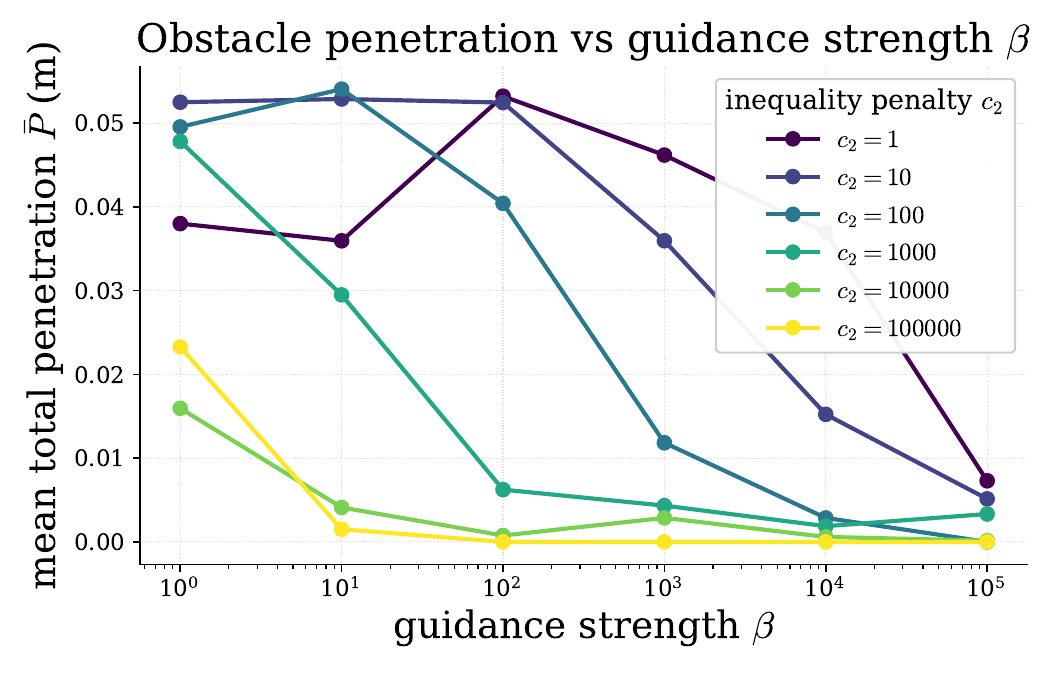}}
  \caption{Guidance strength $|\beta|$: violations decay exponentially, as predicted by \cref{eq:concentration}.}
  \label{fig:ex:B}
\end{subfigure}
  \end{minipage}\hfill
  \begin{minipage}[t]{0.245\linewidth}
  \vspace{0pt}
  \captionsetup{justification=raggedright,singlelinecheck=false}
  \caption{Ablations on LIBERO-Object with a cylindrical obstacle (radius $0.04$ m). $\bar{P}$ 
  a constraint-violation metric defined in \cref{app:constraint_violation_metric}.}
  \label{fig:ablations}
  \end{minipage}
\end{figure}

\textbf{Overlap assumption.} Our weighted SDEs (\cref{eq:fkc-weight}) combines a cost-aware drift (\cref{eq:fkc-weight-x}) 
with the FK reweighting term (\cref{eq:fkc-weight-w}); cost-aware the drift alone coincides with linear-combination guidance, and the reweighting is what makes the procedure a principled posterior sampler. This presumes that $p_{\text{data}}$ and $\exp(\beta\mathcal{J})$ overlap. When they do not, \cref{eq:fkc-weight-w} carries little signal and the dynamics reduce to \cref{eq:fkc-weight-x}, so our method degrades gracefully to the linear-combination and may consequently produce samples outside the learned data distribution. This behavior is unavoidable in the absence of overlap: there are no trajectories with appreciable probability under $p_{\text{data}}$ that also satisfy the specified optimization objective.

\textbf{Whole-trajectory sampling vs.\ post-hoc pushing.} The Transport gains in \cref{tab:main_col_results} are not purely a collision effect: the cylinder is placed at the bimanual handoff point, and post-hoc gradient-based methods routinely push the two arms to \emph{opposite} sides of the obstacle, breaking the handoff. Because we sample whole trajectories from the posterior, joint feasibility of the two-arm motion is enforced at the population level. Population-level reweighting lets our method handle convex and non-convex obstacles uniformly (\cref{fig:toy-2d}, V-shape in \cref{tab:main_col_results}); JM2D also handles non-convex geometry to some extent via its auxiliary module, whereas purely gradient-based baselines struggle. For $\pi_{0.5}$, encoding the scene as an online SDF that treats every non-target object and the table as collidable yields two side benefits beyond raw collision avoidance: fewer incorrect-object grasps and fewer table strikes.


%% file: Sections/limitation_and_conclusion.tex
\section{Limitations and Future Work}

\paragraph{Limitations} Our method introduces three main limitations 
relative to vanilla diffusion or flow sampling. First, while the FK 
particles can be simulated in parallel on GPU and the reweighting step 
adds negligible cost, constructing the SDF used to evaluate the 
collision penalty $\mathcal{J}$ is inherently sequential and must be 
redone at every replanning step. In practice this yields a 
$1.25$--$2\times$ slowdown over the base policy (e.g.\ $\pi_{0.5}$ 
runs at $\sim$$30$\,Hz, ours at $\sim$$15$\,Hz); when the policy outputs 
joint positions and forward kinematics must be differentiated through to map 
actions into Cartesian space, the slowdown grows to $1.5$--$2\times$. Second, 
the FKC-weighted SDE is conceptually heavier and harder to implement 
than a naive linear-combination guidance scheme. Finally, the augmented 
cost function $\mathcal{J}$ must be (almost) differentiable.

\textbf{Future Work.} Our experiments focus on collision constraints, 
and the quality of guidance ultimately depends on the fidelity of the 
constraint representation; richer SDF constructions such as those of 
OmniGuide~\cite{song2026omniguideuniversalguidancefields} are a natural 
drop-in. More broadly, our framework is not specific to collisions: 
it can be instantiated with any differentiable cost, opening the door 
to extending the recent suite of inference-time steering 
objectives---semantic grounding and human-demonstration 
attractors~\cite{song2026omniguideuniversalguidancefields}, and point, 
sketch, and physical-correction 
inputs~\cite{wang2025inferencetimepolicysteeringhuman}---to a 
principled, posterior-sampling treatment. Finally, applying the
marginal-preserving SDE construction to streaming flow 
policies~\cite{jiang2025streaming} would extend FKC-style guidance to 
the low-latency regime required for reactive closed-loop control. 

%% file: Appendix/related_work.tex
\appsection{Related Work}

\paragraph{Generative Robot Policies}
Imitation learning has progressively shifted from deterministic behavior 
cloning to expressive generative models. Diffusion-based 
policies~\cite{janner2022diffuser, chi2024diffusionpolicy, streamingDP} 
have demonstrated strong performance by representing multimodal action 
distributions, while more recent flow-matching 
formulations~\cite{lipman2023flow,fmPolicy2,jiang2025streaming} retain 
this expressivity while reducing inference latency through deterministic 
transport. Our method is agnostic to which of these is used as the base 
policy.

\paragraph{Inference-Time Guided Generation}
A growing body of work seeks to enforce constraints on pretrained 
generative policies at inference time. \emph{Post-hoc} approaches 
project or filter generated trajectories using trajectory 
optimization~\cite{CHOMP,trajopt} or control-theoretic safety 
filters~\cite{cbf_intV}, but decouple generation from correction and 
often produce out-of-distribution outputs. \emph{Generative guidance} 
methods bias the sampling process directly, via classifier-based 
guidance~\cite{dhariwal2021diffusion}, control barrier functions 
based guidance~\cite{xiao2023safediffuser},
projection~\cite{DPwConstarints25}, reflected 
dynamics~\cite{lou2023reflected,liu2023mirror}, or progressive 
barriers~\cite{mishra2025eb}. Closest to our work, 
JM2D~\cite{jung2025joint} formulates constrained planning as joint 
sampling over a model-free diffusion variable and an auxiliary 
model-based variable, using an interaction potential to encourage mutual 
compatibility. While JM2D is a principled step beyond projection, its 
target distribution is defined over an augmented pair $(x, k)$, and its 
practical sampler relies on a Monte Carlo approximation of the joint 
score together with a final model-based repair step for infeasible 
samples. In contrast, we formulate constrained generation directly as 
posterior sampling over the trajectory distribution learned by the 
pretrained policy, with an explicit Feynman-Kac density evolution and 
weighted SDE sampler. This avoids introducing an auxiliary model-based 
variable or a post-hoc projection operator, and applies uniformly to 
both diffusion and flow policies.

ITPS~\cite{wang2025inferencetimepolicysteeringhuman} steers diffusion 
policies via human interactions (points, sketches, physical corrections) 
and introduces two diffusion-specific samplers: Guided Diffusion (GD), 
which is exactly the linear-combination update  
and therefore samples from the \emph{sum} rather than the product of 
the policy and objective distributions, and Stochastic Sampling (SS), 
which targets the product via annealed unadjusted Langevin (ULA) MCMC
by running $M$ inner Langevin steps at \emph{every} denoising noise 
level---a procedure that multiplies the cost of each reverse step 
by $M$ and is therefore expensive in practice. 
OmniGuide~\cite{song2026omniguideuniversalguidancefields} aggregates 
collision, semantic, and human-demonstration cues into attractive/repulsive 
energy fields (\(\mathcal E\)) for VLAs via the same 
linear-combination update $v(x) + \lambda \nabla \mathcal{E}(x)$, 
and therefore inherits GD's sum-of-distributions limitation. 
In contrast, our Feynman--Kac weighted SDE carries explicit importance 
weights derived from the Fokker--Planck PDE of the cost-tilted 
posterior, requires only a single network evaluation per reverse 
step (the $K$ FK particles are simulated in parallel on GPU), 
and extends uniformly to both diffusion and flow-matching 
policies via the marginal-preserving SDE recast of \cref{app:flow_to_sde}.

\paragraph{Feynman-Kac Methods for Diffusion and Flow Models}
Feynman-Kac correctors~\cite{skreta2025fkc} were recently introduced as 
a principled framework for inference-time control of diffusion models, 
with applications to annealing, classifier-free guidance, and product 
of experts. Sequential Monte Carlo techniques have similarly been used 
for reward-guided generation~\cite{uehara2024understanding}, conditional 
generation~\cite{wu2024practical}, and inverse 
problems~\cite{cardoso2024monte} in diffusion models. Our work brings 
this machinery to constrained imitation learning for robotics, and 
extends it to flow-matching policies via the marginal-preserving 
SDE construction of \citet{singh2024stochastic}.
Singhal et al.~\cite{singhal2025generalframeworkinferencetimescaling}
concurrently introduce \emph{Feynman--Kac steering}, a related but distinct 
framework that uses discrete-time FK potentials together with SMC-style 
particle resampling for reward-guided inference in image and text 
diffusion models; in contrast, we adopt the continuous-time 
FK PDE formulation, derive a weighted SDE with an explicit cost-aware 
drift correction, and target constrained robot trajectory generation 
across both diffusion and flow-matching policies. Feynman-Kac-Flow~\citet{mark2025feynmankacflowinferencesteeringconditional} 
applies the discrete time FK steering of 
\citet{singhal2025generalframeworkinferencetimescaling} to
Flow Matching. The drift term is unmodified and the tilt 
is carried entirely by SMC resampling, with potentials 
from a heuristic one-shot endpoint estimate. We instead use 
the continuous-time FK PDE, which splits the tilt into an 
explicit cost-aware drift and a weight update derived in 
closed form from the Fokker–Planck PDE, which is exact at every \(t\).

%% file: Appendix/background.tex
\appsection{Additional Background}
\label{app:add_background}

\subsection{Transport Equation (for Densities)}
While \cref{eq:denoising} describes a procedure for simulating individual 
trajectories, we can also characterize the time-evolution of the entire 
distribution of samples $p_t(x)$\footnote{Throughout this Appendix, 
we overload the notation \(p_t(\cdot)\) to refer to any generic marginal 
distribution. It may also denote the marginal posterior distribution 
from the main text, but this will be explicitly mentioned whenever 
we refer to the posterior.} using Partial Differential Equations (PDEs). 
We describe three fundamental PDEs that will be central to our development, 
each capturing a distinct mechanism by which a density can evolve over time.

\textbf{(1) Continuity Equation.} When samples evolve according to a 
deterministic flow with drift $v_t$, the corresponding density evolves 
according to the continuity equation:
\begin{align}
    dx_t = v_t(x_t)dt \implies \deriv{p_t^{\text{ode}}(x)}{t} = 
    -\inner{\nabla}{p^{\text{ode}}_t(x)v_t(x)}\,,
\label{eq:continuity}
\end{align}
where $p_t^{\text{ode}}$ denotes the density evolving purely under the 
deterministic flow. Intuitively, this equation states that probability mass 
is transported by the drift $v_t$ without any creation or destruction.

\looseness=-1
\textbf{(2) Diffusion Equation.} When samples undergo pure Brownian motion 
with diffusion coefficient $\sigma_t$, the density evolves according to:
\begin{align}
    dx_t = \sigma_t dW_t \implies \deriv{p_t^{\text{diff}}(x)}{t} = 
    \frac{\sigma^2_t}{2}\Delta p_t^{\text{diff}}(x)\,,
\label{eq:diffusion}
\end{align}
where $p_t^{\text{diff}}$ denotes the density evolving purely under the 
stochastic diffusion term. This is simply a heat equation, describing how 
stochastic noise progressively spreads probability mass.

\looseness=-1
Since the denoising SDE in \cref{eq:denoising} combines both a deterministic 
drift and a stochastic diffusion term, the joint evolution of its density is 
governed by the \emph{Fokker-Planck Equation}, which is the superposition 
of the two mechanisms above:
\begin{align}
    \deriv{p_t^{\text{sde}}(x)}{t} = -\inner{\nabla}{p_t^{\text{sde}}(x)
    v_t(x)} + \frac{\sigma_t^2}{2}\Delta p_t^{\text{sde}}(x)\,.
\label{eq:fpe}
\end{align}

\looseness=-1
However, the Fokker-Planck equation alone is insufficient for our purposes: 
we additionally require a mechanism to \textit{reweight} samples, which will 
be key to simulating a sequence of marginals other than those induced by the 
forward noising process $p_{1-\tau}$ (see \cref{sec:resampling}). This motivates 
the third type of PDE.

\looseness=-1
\textbf{(3) Reweighting Equation.} When samples carry time-dependent 
log-weights $w_t$ that are updated as a function of sample positions $x_t$, 
the resulting weighted density evolves as:
\begin{equation}
    dw_t = \bar{g}_t(x_t)dt \implies \deriv{p_t^w(x)}{t} = 
    \bar{g}_t(x)p_t^w(x)\,, \quad
    \text{where}~~\bar{g}_t(x) = g_t(x) - \int g_t(x)p_t^w(x)\, dx\,,
\label{eq:reweighting}
\end{equation}
where the centering term in $\bar{g}_t$ ensures that the normalization 
constant is conserved, i.e. $\int \bar{g}_t(x)p_t^w(x)\,dx = 0$.

\textbf{Feynman-Kac Formula.} Combining all three mechanisms --- 
deterministic drift, stochastic diffusion, and sample reweighting --- yields 
the \textit{Feynman-Kac PDE}, which governs the evolution of the weighted 
density $p_t^{\textsc{fk}}$:
\begin{equation}
    \deriv{p_t^{\textsc{fk}}(x)}{t} =
    -\inner{\nabla}{p_t^{\textsc{fk}}(x)v_t(x)} 
    + \frac{\sigma_t^2}{2}\Delta p_t^{\textsc{fk}}(x) + \bar{g}_t(x)p_t^{\textsc{fk}}(x)\,.
\label{eq:fk_pde_appendix}
\end{equation}
To draw samples from $p_t^{\textsc{fk}}$, one simulates the following 
weighted SDE:
\begin{align}
    dx_t = v_t(x_t)dt + \sigma_t dW_t\,,\quad dw_t = \bar{g}_t(x_t)\,dt\,,
\label{eq:weighted_sde_app}
\end{align}
and then reweights the resulting samples using $w_t$. Crucially, 
$p_t^{\textsc{fk}}(x)$ reflects the density of \textit{weighted} samples, 
which in general differs from the unweighted density $p_t^{\text{sde}}(x)$ 
governed by the Fokker-Planck PDE in \cref{eq:fpe} --- the distinction 
arising precisely from the additional reweighting term $\bar{g}_t$.

\subsection{Equivalences Between Transport, Diffusion, and Reweighting}
\label{sec:pde-equivalences}

The three PDE mechanisms introduced above---continuity, diffusion, and
reweighting---are less distinct than their separate formulations suggest:
provided one has access to the (any) exact score $\nabla \log p_t$, each may be
re-expressed in terms of the others, yielding alternative simulation
schemes that share the same marginal densities. We record the three
relevant identities here, since they underpin both the recasting of the
flow ODE as a marginal-preserving SDE in \cref{app:flow_to_sde} and the
score-imputation step used in the proof of \cref{thm:cost_guided_sde} 
(\cref{thm:gibbs_tilted_full}).

\paragraph{Diffusion as continuity}
Pure Brownian diffusion can be reinterpreted as deterministic transport
under a score-based drift:
\begin{align}
    \deriv{p_t(x)}{t}
    &= \frac{\sigma_t^2}{2}\,\Delta p_t(x)
    = -\inner{\nabla}{p_t(x)\!\left(-\tfrac{\sigma_t^2}{2}\nabla\log p_t(x)\right)} \nonumber \\
    &\implies\quad dx_t = -\tfrac{\sigma_t^2}{2}\nabla\log p_t(x_t)\,dt.
    \label{eq:diffusion_as_continuity}
\end{align}
This is exactly the probability-flow ODE associated with the diffusion
equation: integrating \cref{eq:diffusion_as_continuity} deterministically
yields the same marginals $p_t$ as simulating $dx_t = \sigma_t\,dW_t$.

\paragraph{Continuity as reweighting}
Conversely, deterministic transport can be absorbed entirely into a
reweighting of stationary particles, by factoring $p_t$ out of the
continuity equation,
\begin{align}
    \deriv{p_t(x)}{t}
    &= -\inner{\nabla}{p_t(x)v_t(x)}
    = \left(-\frac{1}{p_t(x)}\inner{\nabla}{p_t(x)v_t(x)}\right) p_t(x) \nonumber \\
    &\implies\quad
    dw_t = -\!\left(\inner{\nabla}{v_t(x_t)} + \inner{\nabla\log p_t(x_t)}{v_t(x_t)}\right) dt.
    \label{eq:continuity_as_rw}
\end{align}
Under this scheme, particle positions $x_t$ are held fixed and the entire
density change is carried by the log-weights $w_t$.

\paragraph{Diffusion as reweighting}
The same factoring, combined with the identity
$\Delta p_t = p_t \bigl(\Delta \log p_t + \|\nabla \log p_t\|^2\bigr)$,
shows that diffusion can also be folded directly into the weights:
\begin{align}
    \deriv{p_t(x)}{t}
    &= \tfrac{\sigma_t^2}{2}\,\Delta p_t(x)
    = \tfrac{\sigma_t^2}{2}\, p_t(x)\!\left(\Delta \log p_t(x) + \|\nabla \log p_t(x)\|^2\right) \nonumber \\
    &\implies\quad
    dw_t = \tfrac{\sigma_t^2}{2}\!\left(\Delta \log p_t(x_t) + \|\nabla \log p_t(x_t)\|^2\right) dt.
    \label{eq:diffusion_as_rw}
\end{align}

Taken together, \cref{eq:continuity_as_rw,eq:diffusion_as_rw} provide a
constructive recipe for converting an arbitrary drift $v_t$ or diffusion
coefficient $\sigma_t$ into an equivalent reweighting term, given access
to the score $\nabla \log p_t$. This interchangeability is precisely
what allows the Feynman--Kac corrector framework to compensate for any
mismatch between the simulated dynamics and the target density evolution
through the weight update, rather than requiring exact agreement in the
drift and diffusion.

%% file: Appendix/flow_to_diffusion.tex
\appsection{From Deterministic Flow to Equivalent SDE}
\label{app:flow_to_sde}

In this appendix we provide complete details on how a deterministic flow 
matching policy can be simulated stochastically while preserving its 
sequence of marginal densities. Our exposition follows 
\citet{singh2024stochastic}.

\subsection{The General Setup}

Consider a generic Itô SDE of the form
\begin{equation}
dx = f_t(x)\,dt + G_t(x)\,dW_t,
\label{eq:app_general_sde}
\end{equation}
where $f_t(x) : \mathbb{R}^d \times [0, 1] \to \mathbb{R}^d$ is the drift, 
$G_t(x) : \mathbb{R}^d \times [0, 1] \to \mathbb{R}^{d \times d}$ is the 
state- and time-dependent diffusion coefficient, and $W_t$ is the standard 
Wiener process. The time-evolution of the marginal density $p_t(x)$ of 
solutions of \cref{eq:app_general_sde} is described by the 
Fokker-Planck-Kolmogorov (FPK) equation
\begin{equation}
\frac{\partial p_t}{\partial t} 
= -\nabla \cdot [f\, p_t] 
+ \tfrac{1}{2}\, \nabla \cdot \big(\nabla \cdot [GG^\top p_t]\big).
\label{eq:app_fpk}
\end{equation}
Note that a deterministic ODE $dx = v_t(x)\,dt$ corresponds to the special 
case of \cref{eq:app_general_sde} with $f \equiv v$ and $G \equiv 0$, in 
which case \cref{eq:app_fpk} reduces to the continuity equation 
$\partial_t p_t = -\nabla \cdot [v\,p_t]$.

\subsection{A Family of SDEs Sharing the Same Marginals}

The central observation is that the FPK equation in \cref{eq:app_fpk} does 
not uniquely determine the drift and diffusion. Many different choices of 
drift and diffusion can yield the same marginal density evolution. The 
following theorem makes this explicit.

\PaperStatementBegin{PaperPlum}
\begin{theorem}[Theorem 1 of \citep{singh2024stochastic}]
\label{thm:singh_general}
Let $p_t(x)$ be the probability density of solutions of the SDE in 
\cref{eq:app_general_sde}, evolving according to \cref{eq:app_fpk}. Then 
for arbitrary functions $\tilde{G_t}(x) : \mathbb{R}^d \times [0,1] \to 
\mathbb{R}^{d\times d}$ and $\gamma_t \geq 0$, the SDE
\begin{equation}
dx = \bar{f_t}(x)\,dt + \bar{G_t}(x)\,dW_t 
\label{eq:app_singh_sde}
\end{equation}
admits the same marginal density $p_t(x)$, where
\begin{align}
\bar{f} &= f - \tfrac{1}{2}\Big( \nabla \cdot [(1-\gamma_t)GG^\top - 
\tilde{G}\tilde{G}^\top] + [(1-\gamma_t)GG^\top - \tilde{G}\tilde{G}^\top]
\nabla \log p_t \Big), \label{eq:app_drift_general}\\
\bar{G} &= \big[\gamma_t GG^\top + \tilde{G}\tilde{G}^\top\big]^{1/2},
\label{eq:app_diff_general}
\end{align}
provided the resulting SDE admits a unique solution.
\PaperStatementEnd
\end{theorem}
\PaperStatementClose

\Cref{thm:singh_general} provides a recipe for constructing entire 
families of SDEs that share the same marginal densities as a given SDE or 
ODE. The free parameters are: (i) the auxiliary diffusion 
$\tilde{G_t}(x)$, which injects additional stochasticity not present in 
the original process; and (ii) the scalar $\gamma_t \geq 0$, which scales 
the contribution of the original diffusion $G$. Setting $\gamma_t = 1$ 
and $\tilde{G} = 0$ trivially recovers the original SDE, while 
$\gamma_t = 0$ and $\tilde{G} = 0$ produces the corresponding probability 
flow ODE (with diffusion replaced by a score-based drift).

\subsection{Specialization to Deterministic Flow}

The case of interest in our work is the deterministic flow 
$dx_t = -v_{1-t}(x_t)\,dt$, which corresponds to taking $f_t(x) = -v_{1-t}(x)$
and $G \equiv 0$ in \cref{eq:app_general_sde}. Substituting these into 
\Cref{thm:singh_general} eliminates all $G$-dependent terms, leaving 
freedom only in the auxiliary diffusion $\tilde{G}$. Choosing $\tilde{G} 
= \tilde{\sigma}(t) I$ for a scalar function $\tilde{\sigma_t} \geq 0$ 
(i.e., isotropic, state-independent noise) yields the following 
specialization.

\PaperStatementBegin{PaperTeal}
\begin{corollary}[Cor. 1.2 of \citep{singh2024stochastic}]
\label{cor:flow_to_sde}
Let $\{q_t(x)\}_{t \in [0,1]}$ denote the marginal densities induced by 
the deterministic flow $dx_t = -v_{1-t}(x_t)\,dt$. Then for any scalar-valued
function $\tilde{\sigma_t} \geq 0$, the SDE
\begin{equation}
dx_t = \left( -v_{1-t}(x_t) + \tfrac{\tilde{\sigma_t}^2}{2}\nabla \log
q_t(x_t) \right)dt + \tilde{\sigma_t}\,dW_t 
\label{eq:app_flow_sde}
\end{equation}
admits the same marginal densities $\{q_t(x)\}_{t\in[0,1]}$.
\PaperStatementEnd
\end{corollary}
\PaperStatementClose

The diffusion term $\tilde{\sigma_t} dW_t$ in \cref{eq:app_flow_sde} 
spreads probability mass via Brownian motion, while the additional term 
$\tfrac{\tilde{\sigma_t}^2}{2}\nabla \log q_t$ in the drift is a 
score-based correction that exactly compensates for this spreading: it 
transports probability mass back toward high-density regions at the rate 
required to leave the marginals $q_t(x)$ unchanged. The two effects are 
in precise balance, so although individual sample trajectories differ 
between the ODE and the SDE, their distributions at every time $t$ remain 
identical.

The function $\tilde{\sigma_t}$ is a free design choice. Setting 
$\tilde{\sigma} \equiv 0$ recovers the original deterministic flow, while 
larger values of $\tilde{\sigma}$ inject more stochasticity into the 
simulation. Common choices from \citet{singh2024stochastic} include:
\begin{itemize}
    \item \textbf{Constant:} $\tilde{\sigma_t} = \alpha$, where 
    $\alpha \geq 0$ controls the overall noise scale.
    \item \textbf{Non-singular:} $\tilde{\sigma_t} = \alpha\sqrt{1-t}$,
    which avoids singularities at the data endpoint $t=1$.
    \item \textbf{Zero-ends:} $\tilde{\sigma_t} = \alpha\sqrt{t(1-t)}$, 
    which vanishes at both endpoints.
\end{itemize}

\subsection{Imputing the Score from the Velocity Field}

Equations~(31)--(33) use model time; in generation time,
$\nabla\log q_t(x)=\nabla\log p_{1-t}(x)$.

Simulating \cref{eq:app_flow_sde} requires evaluating the score function 
$\nabla \log q_t(x)$ of the flow's marginals. Naively, this would require 
training a separate score network. However, for the standard rectified 
flow setup with Gaussian source $p_{\text{src}} = \mathcal{N}(0, I_d)$ and 
the linear interpolation $x_t = (1-t)\,x_{\text{data}} + t\,x_{\text{src}}$, 
the score admits a closed-form expression in terms of the learned velocity 
field itself.

\PaperStatementBegin{PaperPlum}
\begin{proposition}[Score from velocity, \citep{singh2024stochastic}]
\label{prop:score_from_velocity}
For the rectified flow setup with Gaussian source 
$p_{\text{src}} = \mathcal{N}(0, I_d)$ and the linear interpolation, the 
score of the marginals satisfies
\begin{equation}
\nabla \log p_t(x) = -\frac{(1-t)\,v_t(x) + x}{t}, \quad t \in (0, 1].
\label{eq:flow_score}
\end{equation}
\PaperStatementEnd
\end{proposition}
\PaperStatementClose

The derivation begins from denoising score matching, which expresses the 
marginal score as an expectation of the conditional score under the 
posterior $p(x_{\text{data}} \mid x_t)$:
\begin{equation}
\nabla \log p_t(x_t) = \mathbb{E}_{p(x_{\text{data}} \mid x_t)}\!\left[
\frac{\partial \log p_t(x_t \mid x_{\text{data}})}{\partial x_t} \right].
\end{equation}
For the linear interpolation $x_t = (1-t)\,x_{\text{data}} + t\,x_{\text{src}}$ 
with $x_{\text{src}} \sim \mathcal{N}(0, I_d)$, the conditional density 
$p_t(x_t \mid x_{\text{data}})$ is Gaussian with mean $(1-t)\,x_{\text{data}}$ 
and covariance $t^2 I_d$. The conditional score is therefore
\begin{equation}
\frac{\partial \log p_t(x_t \mid x_{\text{data}})}{\partial x_t} 
= -\frac{x_t - (1-t)\,x_{\text{data}}}{t^2}.
\end{equation}
Linearity of expectation, combined with the relation 
$x_{\text{data}} = x_t - t\,\mathbb{E}[x_{\text{src}} - x_{\text{data}} 
\mid x_t] = x_t - t\,v_t(x_t)$ that follows from the interpolation, 
yields \cref{eq:flow_score} after simplification.

The practical upshot of \cref{prop:score_from_velocity} is that the 
pretrained flow model carries all the information needed to simulate 
\cref{eq:app_flow_sde}: both the drift and the score can be evaluated 
through a single forward pass of $v_t(x)$, with no auxiliary score 
network required.

%% file: Appendix/resampling.tex
\appsection{Resampling Methods}
\label{sec:resampling}

In this section, we describe several options for utilizing the FKC weights 
to improve sampling with a batch of $K$ trajectories. While the simplest 
technique would be to simulate the weighted SDE in 
\cref{thm:cost_guided_sde} for $K$ independent trajectories across the full time 
interval $t \in [0,1]$ and reweight using SNIS in \cref{eq:snis}, we 
expect these full-trajectory weights to have high variance in practice 
due to error accumulation along the path. Therefore, we state the resampling methods recommended
by~\citet{skreta2025fkc} for the FKC framework.

\paragraph{Sequential Monte Carlo}
Because our weights constitute a proper weighting scheme for every intermediate 
distribution \citep{naesseth2019elements} (\cref{app:expecation_estimation}), we are able to employ SMC methods 
that reweight trajectories over the course of their simulation. Empirically, 
restricting resampling to an `active interval' $t \in [t_{\text{min}}, 
t_{\text{max}}]$ helps improve sample quality while maintaining diversity, 
and outside this interval we set the weights to zero. Inside the active 
interval, resampling is performed at every step according to the increment 
$w_t^{(k)} = g_t(x_t^{(k)})\,dt$, with systematic sampling carried out in 
proportion to $\exp\{w_t^{(k)}\}$ \citep{douc2005comparison}. With small 
discretizations $dt$, we anticipate that the weights will exhibit relatively 
low variance. Viewed in this light, systematic resampling is an appealing 
selection mechanism, since it retains every trajectory when the weights are 
uniform. That said, even multinomial resampling offers an appealing and 
simpler alternative scheme.

\paragraph{Jump Process Interpretation of Reweighting}
Recasting the reweighting equation as a Markov jump process 
\citep[Ch. 4.2]{ethier2009markov} opens up a range of additional simulation 
algorithms for Feynman-Kac PDEs 
\citep[Ch. 1.2.2, 5]{del2013mean}, \citep{rousset2006equilibrium, 
angeli2020interacting}. A Markov jump process is specified by a rate 
function $\lambda_t(x)$, which controls how often jumps take place, 
together with a Markov transition kernel $J_t(y \mid x)$, from which the 
subsequent state is drawn whenever a jump happens. For a jump process, the 
forward Kolmogorov equation takes the form
\begin{equation}
    \frac{\partial p_t^{\text{jump}}(x)}{\partial t} 
    = \left( \int \lambda_t(y)\, J_t(x \mid y)\, p_t(y)\, dy \right) 
    - p_t(x)\, \lambda_t(x),
\end{equation}
in which the two terms may be intuitively interpreted as quantifying the 
inflow and outflow of probability caused by jumps.

We aim to identify $\lambda_t(x)$ and $J_t(y \mid x)$ for which 
$p_t^{\text{jump}}$ reproduces the evolution of $p_t^w$ in 
\cref{eq:reweighting} under a given choice of $g_t$. Numerous jump 
processes fulfill this requirement 
\citep[Ch. 5]{del2013mean}, \citep{angeli2019rare}; here we follow the 
specific construction of \citet{skreta2025fkc}.

\PaperStatementBegin{PaperPlum}
\begin{proposition}[\citep{skreta2025fkc}]
\label{prop:jump_process}
For a given $g_t$ in \cref{eq:reweighting}, define the jump process rate 
and transition as
\begin{align}
    \lambda_t(x) &= \big( g_t(x) - \mathbb{E}_{p_t}[g_t] \big)^-, \\
    J_t(y \mid x) &= \frac{\big( g_t(y) - \mathbb{E}_{p_t}[g_t] \big)^+ 
    p_t(y)}{\int \big( g_t(z) - \mathbb{E}_{p_t}[g_t] \big)^+ p_t(z)\, dz},
\end{align}
where $(u)^- := \max(0, -u)$ and $(u)^+ := \max(0, u)$. Then,
\begin{equation}
    \frac{\partial p_t^{\text{jump}}(x)}{\partial t} 
    = \frac{\partial p_t^w(x)}{\partial t} 
    = p_t(x) \big( g_t(x) - \mathbb{E}_{p_t}[g_t] \big),
\end{equation}
which matches \cref{eq:reweighting}.
\PaperStatementEnd
\end{proposition}
\PaperStatementClose

In the continuous-time and mean-field limit, this jump-process view of 
reweighting amounts to simulating
\begin{equation}
    x_{t+dt} = \begin{cases} 
    x_t & \text{w.p. } 1 - \lambda_t(x_t)\, dt + o(dt), \\
    \sim J_t(y \mid x_t) & \text{w.p. } \lambda_t(x_t)\, dt + o(dt).
    \end{cases}
    \label{eq:jump_simulation}
\end{equation}
This process should refine the sample population efficiently, because 
jumps are initiated only from states where 
$( g_t(x) - \mathbb{E}_{p_t}[g_t])^- \geq 0 \implies g_t(x) \leq 
\mathbb{E}_{p_t}[g_t]$, and transitions tend to favor destination states 
with large excess weight $( g_t(y) - \mathbb{E}_{p_t}[g_t])^+ > 0$. That 
is, low-cost trajectories take over low-weight ones, thereby directing 
computational resources toward promising regions of the trajectory space.

To estimate the jump rate $\lambda_t(x)$ and transition $J_t(y \mid x)$ 
in practice, we rely on the empirical approximation $p_t^K(z) = \tfrac{1}{K}
\sum_{k=1}^K \delta_z(x^{(k)})$. As an alternative to directly simulating 
\cref{eq:jump_simulation}, an implementation built on birth-death 
`exponential clocks' may also be used \citep[Ch. 5.3-4]{del2013mean}.

%% file: Appendix/algo_discription.tex
\appsection{The BayesFP Algorithm}
\label{app:algorithm}

We collect the general BayesFP sampler in one place as pseudocode (\cref{alg:bayesfp}), implemented as a Sequential Monte Carlo (SMC) procedure with systematic resampling using probabilities proportional to \(\exp\{w_t^{(k)}\}\) or with multinomial resampling. The procedure is identical for diffusion and flow-matching policies expect for the base drift and diffusion ($(f_t,\sigma_t)$) for the reverse-time diffusion SDE of \cref{eq:denoising} and the velocity $v_t$ and auxiliary diffusivity $\tilde\sigma_t$ of the marginal-preserving SDE recast of \cref{app:flow_to_sde}).

\begin{algorithm}[h]
\caption{BayesFP: Feynman--Kac sampling from the Gibbs-tilted posterior of \cref{eq:main_posterior}.}
\label{alg:bayesfp}
\PaperBandBegin{sampler}{PaperTeal}
\begin{algorithmic}[1]
\Require pretrained policy: diffusion $(f_t,\sigma_t,\nabla\log q_t)$ \emph{or} flow $(v_t)$ with auxiliary diffusivity $\tilde\sigma_t$; cost $\mathcal{J}$; inverse-temperature schedule $\beta_t<0$; particles $K$; steps $N$; $0=t_0<t_1<\cdots<t_N=1$; active interval $[t_{\min},t_{\max}]$; resampling period $R$; resampling scheme $\mathsf{S}\in\{\textsc{Systematic},\textsc{Multinomial}\}$
\State Draw $\{x^{(k)}\}_{k=1}^{K} \sim \mathcal{N}(0,I_d)$ and set $w^{(k)} \gets 0$
\For{$n = 0,\dots,N-1$}
  \State $t \gets t_n$, \quad $\Delta t \gets t_{n+1}-t_n$
  \For{$k=1,\dots,K$ \textbf{in parallel}}
    \State Evaluate $\nabla\mathcal{J}(x^{(k)})$
    \If{\textbf{diffusion policy}}
      \State $\mu^{(k)} \gets -f_t(x^{(k)}) + \sigma_t^2\,\nabla\log q_t(x^{(k)}) + \tfrac{1}{2}\beta_t\sigma_t^2\,\nabla\mathcal{J}(x^{(k)})$
      \State $\Delta w^{(k)} \gets \Big[\tfrac{\partial\beta_t}{\partial t}\mathcal{J}(x^{(k)}) + \Big\langle \beta_t\nabla\mathcal{J}(x^{(k)}),\, \tfrac{\sigma_t^2}{2}\nabla\log q_t(x^{(k)}) - f_t(x^{(k)})\Big\rangle\Big]\,\Delta t$
      \State $x^{(k)} \gets x^{(k)} + \mu^{(k)}\,\Delta t + \sigma_t\sqrt{\Delta t}\,\xi^{(k)}, \quad \xi^{(k)}\sim\mathcal{N}(0,I_d)$
    \ElsIf{\textbf{flow policy}}
      \State Impute the score: $\nabla\log q_t(x^{(k)}) \gets -\frac{x^{(k)}+t\,v_{1-t}(x^{(k)})}{1-t}$ \Comment{Prop.~\ref{prop:score_from_velocity}}
      \State $\mu^{(k)} \gets -v_{1-t}(x^{(k)}) + \tfrac{\tilde\sigma_t^2}{2}\nabla\log q_t(x^{(k)}) + \tfrac{1}{2}\beta_t\tilde\sigma_t^2\,\nabla\mathcal{J}(x^{(k)})$
      \State $\Delta w^{(k)} \gets \Big[\tfrac{\partial\beta_t}{\partial t}\mathcal{J}(x^{(k)}) - \beta_t\,\big\langle \nabla\mathcal{J}(x^{(k)}),\, v_{1-t}(x^{(k)})\big\rangle\Big]\,\Delta t$
      \State $x^{(k)} \gets x^{(k)} + \mu^{(k)}\,\Delta t + \tilde\sigma_t\sqrt{\Delta t}\,\xi^{(k)}, \quad \xi^{(k)}\sim\mathcal{N}(0,I_d)$
    \EndIf
  \EndFor
  \If{$t \in [t_{\min},t_{\max}]$}
    \State $\Delta w^{(k)} \gets \Delta w^{(k)} - \tfrac{1}{K}\sum_{j}\Delta w^{(j)}$ \Comment{center across the $K$ particles}
    \State $w^{(k)} \gets w^{(k)} + \Delta w^{(k)}$
    \If{$n \bmod R = 0$}
      \State $\pi^{(k)} \gets \exp(w^{(k)})\big/\sum_{j=1}^{K}\exp(w^{(j)})$
      \If{$\mathsf{S}=\textsc{Systematic}$}
        \State Systematically resample $\{\tilde x^{(k)}\}_{k=1}^{K}$ from $\{x^{(k)}\}_{k=1}^{K}$ with probabilities $\pi^{(1:K)}$
      \ElsIf{$\mathsf{S}=\textsc{Multinomial}$}
        \State Draw $a^{(k)} \overset{\mathrm{i.i.d.}}{\sim} \mathrm{Categorical}\!\left(\pi^{(1:K)}\right)$, $k=1,\dots,K$
        \State $\tilde x^{(k)} \gets x^{(a^{(k)})}$, $k=1,\dots,K$
      \EndIf
      \State $x^{(k)} \gets \tilde x^{(k)}$ and reset $w^{(k)} \gets 0$, \quad $k=1,\dots,K$
    \EndIf
  \EndIf
\EndFor
\State Draw $k^\star \sim \mathrm{Categorical}\big(\mathrm{softmax}(w^{(1:K)})\big)$
\State \Return $x^{(k^\star)}$
\end{algorithmic}
\PaperBandEnd{sampler}
\end{algorithm}

Our framework is agnostic to many of the numerical and algorithmic components used to instantiate the sampler. In particular, practitioners may choose the numerical integrator, resampling scheme and frequency, inverse-temperature schedule, active resampling interval, and related discretization choices based on the computational budget and requirements of the application. These choices do not change the underlying Feynman--Kac posterior-sampling formulation, although they may affect computational complexity, discretization and finite-particle error, variance, and empirical performance. Accordingly, some theoretical guarantees or constants may depend on the particular implementation; for example, our finite-particle result is stated for a fixed numerical grid and resampling schedule. \cref{alg:bayesfp} should therefore be viewed as one practical instantiation of the general framework rather than a prescriptive choice of numerical components, and alternative components may be substituted when they offer favorable computational or empirical properties.

%% file: Appendix/proofs.tex
\appsection{Theoretical Results and Proofs}
\label{app:proofs}

In this appendix, we provide the technical details supporting the main theoretical claims used in the paper. The first result recalls the Feynman--Kac expectation representation from \citet{skreta2025fkc} underlying our use of weighted particles, establishing how expectations under the target distribution can be expressed in terms of exponentially weighted trajectories and thereby motivating the self-normalized importance sampling estimator used in our method. The second result derives the Feynman--Kac weighted dynamics for sampling from the cost-tilted marginals introduced in the main text. This establishes that the proposed guidance procedure is not an ad hoc modification of the sampling dynamics, but instead corresponds to a well-defined density evolution whose weighted particle system targets the desired Gibbs-tilted distribution. The third result justifies the use of this tilted distribution for constrained optimization by showing that, under sufficiently large penalty weights and inverse temperature, its probability mass concentrates on feasible and near-optimal trajectories.
And finally, we establish finite-particle guarantees for the practical BayesFP approximation, showing that the additional error introduced by a finite particle budget decays as \(K^{-1}\), while preserving concentration on feasible, near-optimal trajectories.

\subsection{Expectation Estimation under Feynman--Kac PDEs} 
\label{app:expecation_estimation}
We recall the following Feynman--Kac expectation representation, adapted from \citet{skreta2025fkc}, 
which provides the theoretical basis for estimating expectations under the target 
distribution using weighted samples.

\PaperStatementBegin{PaperPlum}
\begin{proposition}[Expectation Representation under Feynman--Kac Dynamics \cite{skreta2025fkc}]\label{prop:unbiased}
Let $\phi : \cX \rightarrow \bbR$ be a bounded test function, and suppose that
$p_t$ evolves according to the Feynman--Kac PDE
\begin{align}
    \frac{\partial}{\partial t} p_t(x_t) = - \inner{\nabla}{p_t(x_t) v_t(x_t)} + \frac{\sigma_t^2}{2}\Delta p_t(x_t) + p_t(x_t) \left( g_t(x_t) - \int g_t(x_t) p_t(x_t) dx_t \right)  \label{eq:fk_pde_appendix0}
\end{align}

Then the expectation of $\phi$ under the terminal distribution $p_T$ admits the
weighted representation
\begin{align}
    \mathbb{E}_{p_T(x_T)}\!\left[\phi(x_T)\right]
    &=
    \frac{1}{Z_T}
    \mathbb{E}\!\left[
        \exp\!\left(\int_0^T g_s(x_s)\,ds\right)
        \phi(x_T)
    \right],
    \label{eq:unbiased_result}\\
    &\text{where}\qquad
    dx_t = v_t(x_t)\,dt + \sigma_t\,dW_t,
    \qquad x_0 \sim p_0,
    \nonumber
\end{align}
and $Z_T$ denotes a normalization constant that does not depend on $x$.

Consequently, \cref{eq:unbiased_result} provides the Feynman--Kac
representation underlying the self-normalized importance sampling estimator
in \cref{eq:snis}, which converges to the desired expectation as
$K \rightarrow \infty$.
\PaperStatementEnd
\end{proposition}
\PaperStatementClose

\subsubsection{Posterior Expectations under BayesFP}
\label{app:bayesfp_expectations}

We specialize the Feynman--Kac expectation representation of \cref{prop:unbiased}
to the Gibbs-tilted posterior targeted by BayesFP.
We retain $x_t$ for the simulated trajectory, $w_t$ for its log-weight,
and $q_t$ for the pretrained policy's marginals. The observation
conditioning the policy is fixed and suppressed. To allow a general
trajectory space $\mathcal{X}$, we state the result for probability
measures and express the Feynman--Kac evolution in weak form.

\PaperStatementBegin{PaperPlum}
\begin{proposition}[Posterior expectation representation under BayesFP]
\label{prop:bayesfp_expectations}
Let $\mathcal{X}$ be a measurable trajectory space, let $q_t$ be
probability measures on $\mathcal{X}$, and let
$\mathcal{J}:\mathcal{X}\to\mathbb{R}$ be a measurable penalty-augmented
cost. Let $\beta_t\leq0$, $t\in[0,T]$, be a continuously differentiable
inverse-temperature schedule. Define the unnormalized and normalized
Gibbs-tilted measures by
\begin{align}
  \tilde p_t(dx)&:=e^{\beta_t\mathcal{J}(x)}q_t(dx),
  &Z_t&:=\int_{\mathcal{X}}e^{\beta_t\mathcal{J}(x)}q_t(dx),
  \label{eq:bayesfp_tilted_measure}\\
  p_t(dx\mid O=1)&:=Z_t^{-1}\tilde p_t(dx),
  &0&<Z_t<\infty.
  \label{eq:bayesfp_posterior_measure}
\end{align}

Let $(x_t)_{t\in[0,T]}$ be a well-defined, nonexplosive Markov proposal
initialized with $x_0\sim q_0$, and let $\mathcal{A}_t$ denote its
generator on a common domain $\mathcal{D}$ of bounded measurable test
functions. Assume that $\mathcal{D}$ determines finite measures,
contains the constant function $1$, and $\mathcal{A}_t1=0$.
For each $\psi\in\mathcal{D}$, assume the proposal admits the
semimartingale decomposition
\begin{equation}
  \psi(x_t)=\psi(x_0)
    +\int_0^t(\mathcal{A}_s\psi)(x_s)\,ds+M_t^\psi,
  \qquad M_0^\psi=0,
  \label{eq:bayesfp_proposal_generator}
\end{equation}
where $M_t^\psi$ is a local martingale.
Suppose the measurable reweighting function $g_t$ is chosen so that the
Gibbs-tilted measures satisfy the weak Feynman--Kac equation
\begin{equation}
  \begin{aligned}
  \int_{\mathcal{X}}\psi(x)\tilde p_t(dx)
  &=\int_{\mathcal{X}}\psi(x)\tilde p_0(dx)\\
  &\quad+\int_0^t\int_{\mathcal{X}}
      \bigl[(\mathcal{A}_s\psi)(x)+g_s(x)\psi(x)\bigr]
      \tilde p_s(dx)\,ds,
  \qquad\psi\in\mathcal{D}.
  \end{aligned}
  \label{eq:bayesfp_linear_fk}
\end{equation}
Assume uniqueness of this equation in a class of finite-measure curves
containing both $\tilde p_t$ and the weighted proposal measures below.

Assign each proposal the uncentered log-weight
\begin{equation}
  w_t=\beta_0\mathcal{J}(x_0)+\int_0^t g_s(x_s)\,ds.
  \label{eq:bayesfp_expectation_weight}
\end{equation}
Assume that $\int_0^T|g_s(x_s)|\,ds<\infty$ almost surely,
$\mathbb{E}[e^{w_t}]<\infty$ for every $t\in[0,T]$, and, for every
$\psi\in\mathcal{D}$,
\begin{equation}
  \int_0^T\mathbb{E}\!\left[
    e^{w_s}\bigl(|(\mathcal{A}_s\psi)(x_s)|
                  +|g_s(x_s)\psi(x_s)|\bigr)\right]ds<\infty,
\end{equation}
with $\int_0^t e^{w_s}\,dM_s^\psi$ a true martingale.

Then, for every bounded measurable test function
$\phi:\mathcal{X}\to\mathbb{R}$ and every $t\in[0,T]$,
\begin{align}
  \mathbb{E}\!\left[e^{w_t}\phi(x_t)\right]
    &=Z_t\,\mathbb{E}_{p_t(\cdot\mid O=1)}[\phi],
  &\mathbb{E}[e^{w_t}]&=Z_t.
  \label{eq:bayesfp_proper_weighting}
\end{align}
Consequently, for $K$ independent copies of this weighted proposal
and each fixed $t\in[0,T]$ and bounded measurable $\phi$,
the self-normalized estimator satisfies
\begin{equation}
  \sum_{k=1}^K
    \frac{\exp(w_t^{(k)})}{\sum_{j=1}^K\exp(w_t^{(j)})}
    \phi(x_t^{(k)})
  \xrightarrow[K\to\infty]{\mathrm{a.s.}}
    \mathbb{E}_{p_t(\cdot\mid O=1)}[\phi].
  \label{eq:bayesfp_snis_consistency}
\end{equation}
In particular, if $q_T=p_{\mathrm{data}}$ and $\beta_T=\beta$,
the terminal target is
$p_T(dx\mid O=1)\propto e^{\beta\mathcal{J}(x)}p_{\mathrm{data}}(dx)$.
\PaperStatementEnd
\end{proposition}
\PaperStatementClose

\begin{proof}
The weight $e^{w_t}$ is a random continuous process of finite variation,
with $d(e^{w_t})=e^{w_t}g_t(x_t)\,dt$.
Its quadratic covariation with the semimartingale $\psi(x_t)$ is zero.
The product rule applied to~\cref{eq:bayesfp_proposal_generator} gives
\begin{equation}
  d\!\left(e^{w_t}\psi(x_t)\right)
  =e^{w_t}\bigl[(\mathcal{A}_t\psi)(x_t)
                   +g_t(x_t)\psi(x_t)\bigr]dt
    +e^{w_t}\,dM_t^\psi.
\end{equation}
Taking expectations is justified by the assumed integrability and
martingale conditions. Thus the finite measures
$\mathbb{E}[e^{w_t}\delta_{x_t}(dx)]$, where $\delta_{x_t}$ is the
Dirac measure at $x_t$, solve~\cref{eq:bayesfp_linear_fk}.
Their initial measure is
\begin{equation}
  \mathbb{E}[e^{w_0}\delta_{x_0}(dx)]
  =e^{\beta_0\mathcal{J}(x)}q_0(dx)=\tilde p_0(dx).
\end{equation}
Uniqueness therefore implies
$\mathbb{E}[e^{w_t}\delta_{x_t}(dx)]=\tilde p_t(dx)$.
Integrating a bounded measurable $\phi$, and then $1$, establishes
\cref{eq:bayesfp_proper_weighting}.
Finally,
$\mathbb{E}[e^{w_t}|\phi(x_t)|]\leq\|\phi\|_\infty Z_t<\infty$.
Applying the strong law of large numbers separately to the numerator
and denominator, with $Z_t>0$, proves
\cref{eq:bayesfp_snis_consistency}.
\end{proof}

\paragraph{Relation to the diffusion and flow formulations}
Equation~\cref{eq:bayesfp_linear_fk} is the compatibility condition
between the proposal, its reweighting function, and the desired
Gibbs-tilted marginals. In the smooth density setting of the paper,
the weighted SDEs give the following choices.
For the reverse-time diffusion policy,
\begin{align}
  dx_t
  &=\left[-f_t(x_t)+\sigma_t^2\nabla\log q_t(x_t)
       +\frac{\beta_t\sigma_t^2}{2}\nabla\mathcal{J}(x_t)\right]dt
       +\sigma_t\,dW_t,
  \label{eq:bayesfp_diffusion_proposal}\\
  g_t(x)
  &=\frac{\partial\beta_t}{\partial t}\mathcal{J}(x)
    +\beta_t\left\langle\nabla\mathcal{J}(x),
        \frac{\sigma_t^2}{2}\nabla\log q_t(x)-f_t(x)\right\rangle.
  \label{eq:bayesfp_diffusion_potential}
\end{align}
For the flow policy, in the generative-drift convention $-v_t$ used in
the weighted sampler,
\begin{align}
  dx_t
  &=\left[-v_t(x_t)+\frac{\tilde\sigma_t^2}{2}\nabla\log q_t(x_t)
       +\frac{\beta_t\tilde\sigma_t^2}{2}\nabla\mathcal{J}(x_t)\right]dt
       +\tilde\sigma_t\,dW_t,
  \label{eq:bayesfp_flow_proposal}\\
  g_t(x)
  &=\frac{\partial\beta_t}{\partial t}\mathcal{J}(x)
      -\beta_t\langle\nabla\mathcal{J}(x),v_t(x)\rangle.
  \label{eq:bayesfp_flow_potential}
\end{align}
Here $W_t$ denotes the standard Wiener process used in the paper;
$w_t$ continues to denote the particle log-weight.
The generator $\mathcal{A}_t$ is that of the corresponding proposal SDE.
For example, in the diffusion case it acts as
\begin{equation}
  \begin{aligned}
  (\mathcal{A}_t\psi)(x)
  &=\left\langle
      -f_t(x)+\sigma_t^2\nabla\log q_t(x)
        +\frac{\beta_t\sigma_t^2}{2}\nabla\mathcal{J}(x),
      \nabla\psi(x)\right\rangle\\
  &\quad+\frac{\sigma_t^2}{2}\Delta\psi(x).
  \end{aligned}
\end{equation}
The density calculation for the Gibbs-tilted dynamics establishes
\cref{eq:bayesfp_linear_fk} for these choices under its regularity
assumptions. The proposition then supplies their posterior expectation
representation and consistent self-normalized estimator.
The unnormalized equation and weights use $g_t$; its centered counterpart
$\bar g_t=g_t-\mathbb{E}_{p_t(\cdot\mid O=1)}[g_t]$ appears in the
normalized Feynman--Kac evolution.

\paragraph{State space and time conventions}
The expectation result does not require coordinates, a reference
Lebesgue measure, or a finite-dimensional state space. It therefore
also applies to infinite-dimensional trajectory spaces whenever the
proposal, generator, potential, and weak Feynman--Kac evolution satisfy
the stated assumptions. The displayed score, gradient, Laplacian, and
scalar-noise formulas are the paper's smooth density specializations;
they are not definitions of these objects on an arbitrary measurable
space. An infinite-dimensional SDE specialization requires an appropriate
Wiener process, a well-defined noise operator and stochastic integral,
and a corresponding generator satisfying the same compatibility condition.

The proposition uses increasing generative time, starting at $0$ and
ending at $T$. In the paper's flow interpolation clock, the source is
at time $1$ and the data at time $0$. When using increasing generative
time $t\in[0,1]$, every interpolation-indexed quantity must therefore
be evaluated at $1-t$; in particular, the generative ODE drift is
$-v_{1-t}$ when $v$ retains its original interpolation-time indexing.
Equations~\cref{eq:bayesfp_flow_proposal}--\cref{eq:bayesfp_flow_potential}
write $v_t$ after this time reindexing, and the same convention applies
to $q_t$, $\beta_t$, and $\tilde\sigma_t$.
In particular, $\partial\beta_t/\partial t$ is taken with respect to
increasing generative time, including the chain-rule sign when a
nonconstant interpolation-time schedule is reindexed.

\paragraph{Initialization and normalization}
Starting from $q_0$ requires $w_0=\beta_0\mathcal{J}(x_0)$.
Allowing $\beta_0=0$ permits $w_0=0$ and $Z_0=1$, with a negative
terminal inverse temperature $\beta_T=\beta<0$.
Alternatively, if $x_0\sim p_0(\cdot\mid O=1)$ and $w_0=0$, then
$w_t=\int_0^t g_s(x_s)\,ds$ and
\begin{equation}
  \mathbb{E}[e^{w_t}]=\frac{Z_t}{Z_0},
  \qquad
  \mathbb{E}_{p_t(\cdot\mid O=1)}[\phi]
  =\frac{Z_0}{Z_t}\mathbb{E}[e^{w_t}\phi(x_t)].
\end{equation}

\paragraph{Deterministic flow limit}
For the smooth flow specialization, setting $\tilde\sigma_t\equiv0$
gives $dx_t=-v_t(x_t)\,dt$. The chain rule yields
\begin{equation}
  \frac{d}{dt}\bigl[\beta_t\mathcal{J}(x_t)\bigr]
  =\frac{\partial\beta_t}{\partial t}\mathcal{J}(x_t)
     -\beta_t\langle\nabla\mathcal{J}(x_t),v_t(x_t)\rangle
  =g_t(x_t).
\end{equation}
With $x_0\sim q_0$ and $w_0=\beta_0\mathcal{J}(x_0)$, this gives
$w_t=\beta_t\mathcal{J}(x_t)$: importance weighting of prior-flow
samples remains valid in the deterministic limit.

\paragraph{Scope of the guarantee}
The self-normalized estimator is generally biased for finite $K$.
Its stated consistency concerns independent proposal trajectories;
intermediate resampling requires the separate particle-system analysis.
For the score-based specializations, exactness requires the correct
prior score and the full weight potential. Numerical discretization,
approximate scores, and omitted initial or intermediate weights may
contribute to the target discrepancy $d_N$ already considered in the
finite-particle analysis. Common shifts of all log-weights preserve
normalized particle estimates, but the displayed identities for $Z_t$
use uncentered weights before any reset at resampling.
Taking $\phi=\mathbf{1}_{G_\delta}$ estimates the posterior probability
of the feasible, near-optimal event used in the concentration results,
under the separate assumptions imposed by those results.

\subsection{Feynman-Kac PDEs for Gibbs Tiled Sampling}

We begin by deriving the weighted SDE associated with the Gibbs-tilted marginals. The goal is to start from a base process with marginals $q_t(x)$ and construct a Feynman--Kac evolution whose weighted density is proportional to $q_t(x)\exp(\beta_t J(x))$. The proof proceeds by differentiating the normalized tilted density, identifying the additional reweighting term required by the Feynman--Kac PDE, and then specializing the resulting expression to the reverse-time diffusion dynamics used by the pretrained generative policy. The following result is special case of \citet{skreta2025fkc} adapted to sample from the posterior distribution~\cref{eq:main_posterior}.

\PaperStatementBegin{PaperPlum}
\begin{theorem}[Gibbs-Tilted Posterior Sampling for Constrained Trajectory Generation]
\label{thm:gibbs_tilted_full}
Let $q_t(x)$ denote the time-$t$ marginals of a pretrained diffusion policy
(or, via the marginal-preserving SDE recast of \cref{app:flow_to_sde}, a
pretrained rectified-flow policy), generated by the base SDE
\begin{align}
    dx_t = v_t(x_t)\,dt + \sigma_t\,dW_t ,
\end{align}
so that $q_t$ evolves according to the Fokker--Planck equation associated with
this process. Let $\mathcal{J}:\mathbb{R}^d \to \mathbb{R}$ be the
penalty-augmented cost of \cref{eq:agumented_cost_function} encoding the inference-time
constrained optimization problem~\cref{eq:optimzation_prob}, and let
$\beta_t < 0$ be an inverse-temperature schedule controlling the strength of
the cost-induced tilt. Then the cost-tilted posterior marginals
\begin{equation}
    p_t(x):= p_t(x \mid O=1) = \frac{q_t(x)\exp\bigl(\beta\mathcal{J}(x)\bigr)}{\int q_t(x)\exp\bigl(\beta\mathcal{J}(x)\bigr)\,dx},
\end{equation}
i.e.\ the Bayesian posterior over trajectories that treats the learned
demonstration distribution as a prior and $\exp(\beta_t \mathcal{J})$ as a
likelihood on optimality, can be sampled via the following Feynman--Kac
weighted SDE:
\begin{align}
    dx_t
    &= \bigl(v_t(x_t) + a\,\nabla \mathcal{J}(x_t)\bigr)\,dt
       + \sigma_t\,dW_t ,
    \\[0.5em]
    dw_t
    &= \Bigg[
    \left\langle
    \nabla \mathcal{J}(x_t),
    \begin{array}{@{}l@{}}
    \beta_t\Bigl(
        v_t(x_t)
        - \sigma_t^2\nabla\log q_t(x_t)
        - \dfrac{\sigma_t^2}{2}\beta_t\nabla \mathcal{J}(x_t)
    \Bigr)
    \\[0.25em]
    \quad
    + a\Bigl(
        \nabla\log q_t(x_t)
        + \beta_t\nabla \mathcal{J}(x_t)
    \Bigr)
    \end{array}
    \right\rangle
    \\
    &\qquad
    + \left(a - \beta_t\frac{\sigma_t^2}{2}\right)\!\Delta \mathcal{J}(x_t)
    + \deriv{\beta_t}{t}\mathcal{J}(x_t)
    \Bigg]\,dt ,
\end{align}
where $a \in \mathbb{R}$ is a free design parameter, 
and $w_t$ is the log-importance-weight carried by each FK particle.

\medskip
\noindent\textbf{Specialization to the reverse-time generative SDE.}
For the reverse-time SDE used by the pretrained diffusion policy (or by the
flow policy after the marginal-preserving SDE recast of
\cref{app:flow_to_sde}), namely
$v_t(x_t) = -f_t(x_t) + \sigma_t^2 \nabla \log q_t(x_t)$, and with the
canonical choice $a = \beta_t \sigma_t^2/2$, the weighted SDE simplifies to
\begin{align}
    dx_t
    &= \Bigl(-f_t(x_t) + \sigma_t^2 \nabla \log q_t(x_t)
        + \beta_t \tfrac{\sigma_t^2}{2}\,\nabla \mathcal{J}(x_t)\Bigr)\,dt
        + \sigma_t\,dW_t ,
    \\
    dw_t
    &= \left[\,
        \deriv{\beta_t}{t}\,\mathcal{J}(x_t)
        + \left\langle
            \beta_t\,\nabla \mathcal{J}(x_t),\
            \tfrac{\sigma_t^2}{2}\,\nabla \log q_t(x_t) - f_t(x_t)
        \right\rangle
    \right]\,dt .
\end{align}

\PaperStatementEnd
\end{theorem}    
\PaperStatementClose
\begin{proof}
We begin from the dynamics of the base process. Let $q_t(x)$ denote the time-$t$ marginal of the pretrained generative policy, which evolves according to the Fokker--Planck equation
\begin{align}
    \deriv{q_t(x)}{t} = -\inner{\nabla}{q_t(x)v_t(x)} + \frac{\sigma_t^2}{2}\Delta q_t(x)\,.
\end{align}
Our goal is to characterize the time-evolution of the Gibbs-tilted density
\begin{align}
    p_t(x) = \frac{q_t(x) e^{(\beta_t \mathcal J(x))}}{\int dx\;q_t(x) e^{(\beta_t \mathcal J(x))}}\,.
\end{align}
By directly differentiating the logarithm of $p_t(x)$ with respect to time and recognizing that the normalization constant in the denominator contributes a term equal to the expectation under $p_t$ of the time-derivative of the log-numerator, we obtain
\begin{align}
    \deriv{}{t}\log p_t(x) =~& \deriv{}{t}\log q_t(x) + \deriv{\beta_t}{t} \mathcal J(x) - \int dx\;p_t(x)\left[\deriv{}{t}\log q_t(x) + \deriv{\beta_t}{t} \mathcal J(x)\right]
\end{align}
Notice that the first term, $\deriv{}{t}\log q_t(x)$, can be expanded by applying the Fokker-Planck equation for $q_t$ and using standard logarithmic identities, and that one can further re-express the resulting score $\nabla \log q_t$ in terms of $\nabla \log p_t$ by exploiting the relation $\nabla \log p_t(x) = \nabla \log q_t(x) + \beta_t \nabla \mathcal J(x)$ that follows directly from the definition of the tilted density, yielding
\begin{align}
    \deriv{}{t}\log q_t(x) =~& -\inner{\nabla}{v_t(x)} - \inner{\nabla \log q_t(x)}{v_t(x)} + \frac{\sigma_t^2}{2}\Delta \log q_t(x) + \frac{\sigma_t^2}{2}\norm{\nabla \log q_t(x)}^2\\
    =~& -\inner{\nabla}{v_t(x)} - \inner{\nabla \log p_t(x)}{v_t(x)} + \frac{\sigma_t^2}{2}\Delta \log p_t(x) + \frac{\sigma_t^2}{2}\norm{\nabla \log p_t(x)}^2 + \\
    ~&+\inner{\beta_t\nabla \mathcal J(x)}{v_t(x)-\sigma_t^2\nabla\log q_t(x) - \frac{\sigma_t^2}{2}\beta_t\nabla \mathcal J(x)} - \beta_t\frac{\sigma_t^2}{2}\Delta \mathcal J(x)\,.\nonumber
\end{align}
Substituting this expansion back into the time-derivative of $\log p_t$ and re-collecting the terms into transport, diffusion, and reweighting contributions, we arrive at the Feynman--Kac PDE
\begin{align}
    \deriv{p_t(x)}{t} =~& -\inner{\nabla}{p_t(x)v_t(x)} + \frac{\sigma_t^2}{2}\Delta p_t(x) + p_t(x)\left(g_t(x)-\mean_{p_t(x)}g_t(x)\right)\\
    g_t(x) =~& \inner{\beta_t\nabla \mathcal J(x)}{v_t(x)-\sigma_t^2\nabla\log q_t(x) - \frac{\sigma_t^2}{2}\beta_t\nabla \mathcal J(x)} - \beta_t\frac{\sigma_t^2}{2}\Delta \mathcal J(x) + \deriv{\beta_t}{t}\mathcal J(x)\,.
\end{align}
At this stage the entire cost-induced tilt is being carried by the reweighting term $g_t$, but we have full freedom to redistribute part of this tilt into a cost-aware drift correction: by exploiting the transport/reweighting equivalence, we can introduce the gradient of the cost as an additional drift term $a\nabla \mathcal J(x)$ with $a\in\mathbb{R}$ a free design parameter, which leaves the marginal $p_t$ unchanged but modifies $g_t$ accordingly, i.e.
\begin{align}
    \deriv{p_t(x)}{t} =~& -\inner{\nabla}{p_t(x)(v_t(x) + a\nabla \mathcal J(x))} + \frac{\sigma_t^2}{2}\Delta p_t(x) + p_t(x)\left(g_t(x)-\mean_{p_t(x)}g_t(x)\right)\\
    g_t(x) =~& a\Delta \mathcal J(x) + a\inner{\nabla \log p_t(x)}{\nabla \mathcal J(x)} - \beta_t\frac{\sigma_t^2}{2}\Delta\mathcal J(x) + \deriv{\beta_t}{t}\mathcal J(x) +\\
    ~& + \inner{\beta_t\nabla \mathcal J(x)}{v_t(x)-\sigma_t^2\nabla\log q_t(x) - \frac{\sigma_t^2}{2}\beta_t\nabla \mathcal J(x)}\,. \nonumber
\end{align}
We now specialize this general expression to the reverse-time generative dynamics used by the pretrained policy and choose the free parameter so as to obtain a maximally compact weight update: by setting the base drift to the standard reverse-time form $v_t(x) := -f_t(x) + \sigma_t^2 \nabla \log q_t(x)$ and the cost-drift coefficient to the canonical value $a := \beta_t \sigma_t^2/2$ (so that the Laplacian terms $a\Delta\mathcal J$ and $\beta_t \frac{\sigma_t^2}{2}\Delta\mathcal J$ in $g_t$ cancel exactly and the score $\nabla \log p_t$ that appears in $g_t$ collapses to the simpler $\nabla \log q_t$), we obtain the simplified system
\begin{align}
    \deriv{p_t(x)}{t} =~& -\inner{\nabla}{p_t(x)(-f_t(x) + \sigma_t^2 \nabla \log q_t(x) + \beta_t \frac{\sigma_t^2}{2}\nabla \mathcal J(x))}\\
    &+ \frac{\sigma_t^2}{2}\Delta p_t(x) + p_t(x)\left(g_t(x)-\mean_{p_t(x)}g_t(x)\right) \nonumber\\
    g_t(x) =~& \inner{\beta_t\nabla \mathcal J(x)}{\frac{\sigma_t^2}{2}\nabla \log p_t(x)} + \deriv{\beta_t}{t}\mathcal J(x) + \inner{\beta_t\nabla \mathcal J(x)}{-f_t(x) - \frac{\sigma_t^2}{2}\beta_t\nabla \mathcal J(x)} \\
    =~&\deriv{\beta_t}{t}\mathcal J(x) + \inner{\beta_t\nabla \mathcal J(x)}{\frac{\sigma_t^2}{2}\nabla\log q_t(x) -f_t(x)} \,.
\end{align}
Because the resulting PDE has the canonical Feynman--Kac form, it admits a particle-based simulation in which each trajectory carries both a state $x_t$, evolved by the cost-aware SDE drift and the unchanged diffusion, and a log-weight $w_t$, accumulated according to the uncentered reweighting function $g_t$; concretely, this can be simulated as
\begin{align}
    dx_t =~& (-f_t(x_t) + \sigma_t^2 \nabla \log q_t(x_t) + \beta_t \frac{\sigma_t^2}{2}\nabla \mathcal J(x_t))dt + \sigma_t dW_t\,,\\
    dw_t =~& \left[\deriv{\beta_t}{t}\mathcal J(x) + \inner{\beta_t\nabla \mathcal J(x)}{\frac{\sigma_t^2}{2}\nabla\log q_t(x) -f_t(x)}\right]dt
\end{align}
\end{proof}

\subsection{Concentration on the Constrained Optimum}

The previous result shows how to sample from the Gibbs-tilted distribution for a fixed penalty-augmented objective. We now show that this distribution is meaningful for the original constrained optimization problem. In particular, as the constraint penalties and the magnitude of the inverse temperature increase, the tilted measure assigns vanishing probability to trajectories that are either infeasible or suboptimal. This establishes that the posterior targeted by our sampler concentrates around feasible, near-optimal solutions of the constrained problem.

\PaperStatementBegin{PaperPlum}
\begin{proposition}[Concentration on the Constrained Optimum by Gibbs tilting]
\label{prop:approx-constrained-full}
Let $\mathcal{X}\subset \mathbb{R}^d$ be a compact set with nonempty interior (or at least positive Lebesgue measure),
and let $\mathcal L,\rho,h_1,h_2:\mathcal{X}\to\mathbb{R}$ be continuous. Define the feasible set
\[
F := \{x\in \mathcal{X}:\ h_1(x)=0,\ h_2(x)\le 0\} \ne \emptyset,
\qquad
\mathcal L^\star := \min_{x\in F} \mathcal L(x),
\]
and for penalties $c=(c_1,c_2)$ with $c_1,c_2>0$ define the penalized objective
\[
\mathcal J_c(x) := \mathcal L(x)+\rho(x)+c_1 h_1(x)^2 + c_2 \max\{0,h_2(x)\}.
\]
Let $q_0:\mathcal{X}\to[0,\infty)$ be a continuous density (with respect to Lebesgue measure); in our case ($q_0 = p_\text{data}$) and, for $\beta<0$, define
the Gibbs-tilted probability distribution
\[
p^{(\beta,c)}(dx)
:=
\frac{q_0(x)e^{\beta \mathcal J_c(x)}}{\int_{\mathcal{X}} q_0(u)e^{\beta \mathcal J_c(u)}\,du}\,dx.
\]

Assume the following:

\begin{enumerate}
\item[(A1)] (\textbf{Compactness}) $\mathcal{X}$ is compact.
\item[(A2)] (\textbf{Continuity}) $\mathcal J_c$ is continuous on $\mathcal{X}$ for each fixed $c$.
\item[(A3)] (\textbf{Positivity near a supported minimizers}) For each fixed \(c\), let \(S_c:=\arg\min_{\mathcal{X}} \mathcal J_c(x)\). There exist \(x_c^\star\in S_c\) and constants \(r_c,m_c>0\) such that
\[
q_0(x)\ge m_c
\qquad\text{for all }x\in \mathcal{X}\cap B_{r_c}(x_c^\star),
\]
and
\[
\mathrm{Leb}~\!\left(\mathcal{X}\cap B_r(x_c^\star)\right)>0
\qquad\text{for every }0<r\le r_c.
\]
Here \(B_r(x)\) denotes the open Euclidean ball of radius \(r\) centered at \(x\).
\item[(A4)] (\textbf{Penalty correctness for global minimizers}) For every $\delta>0$, there exists $c=(c_1,c_2)$ such that
\emph{every} global minimizer $x_c\in \arg\min_{\mathcal{X}} \mathcal J_c$ satisfies
\[
|h_1(x_c)|\le \delta/2,\qquad h_2(x_c)\le \delta/2,\qquad \mathcal L(x_c)\le \mathcal L^\star+\delta/2.
\]
This is the standard penalty-method consistency condition in nonlinear programming (for the right class of functions \(\rho\)); see \citet{doi:10.1137/1.9781611971316, book}.
\end{enumerate}

Then for every $\delta>0$ and $\eta\in(0,1)$, there exist penalties $c=(c_1,c_2)$ (large enough) and
$\beta<0$ (with $|\beta|$ large enough) such that, for $X\sim p^{(\beta,c)}$,
\[
\mathbb{P}\Big(|h_1(X)|\le \delta,\ \ h_2(X)\le \delta,\ \ \mathcal L(X)\le \mathcal L^\star+\delta\Big)\ \ge\ 1-\eta.
\]
\PaperStatementEnd
\end{proposition}
\PaperStatementClose

\begin{proof}
Fix arbitrary $\delta>0$ and $\eta\in(0,1)$. Define the ``good'' set
\begin{equation}
G_\delta
:=
\Big\{x\in \mathcal{X}:\ |h_1(x)|\le \delta,\ \ h_2(x)\le \delta,\ \ \mathcal L(x)\le \mathcal L^\star+\delta\Big\},
\label{eq:Gdelta}
\end{equation}
and its complement (the ``bad'' set)
\[
A_\delta := \mathcal{X}\setminus G_\delta.
\]

\paragraph{Choose penalties so that all minimizers of $J_c$ are $\delta/2$-good}
By assumption (A4), there exists a choice of penalties $c=(c_1,c_2)$ such that every global minimizer
$x_c\in S_c:=\arg\min_{\mathcal{X}} \mathcal J_c$ satisfies
\[
|h_1(x_c)|\le \delta/2,\qquad h_2(x_c)\le \delta/2,\qquad \mathcal L(x_c)\le \mathcal L^\star+\delta/2.
\]
Equivalently,
\begin{equation}
S_c \subseteq G_{\delta/2} \subseteq G_\delta.
\label{eq:Sc-in-Gdelta}
\end{equation}
Fix such a $c$ for the remainder of the proof, and write $\mathcal J:=\mathcal J_c$ and $S:=S_c$ to simplify notation.
Let
\[
\mathcal J^\star := \min_{x\in \mathcal{X}} \mathcal J(x),
\]
so that $S=\{x\in \mathcal{X}:\ \mathcal J(x)=\mathcal J^\star\}$ and $S\neq\emptyset$ by compactness (A1) and continuity (A2).

\paragraph{There is a strictly positive energy gap outside $G_\delta$}
We claim there exists $\varepsilon>0$ such that
\begin{equation}
A_\delta \subseteq \{x\in \mathcal{X}:\ \mathcal J(x)\ge \mathcal J^\star+\varepsilon\}.
\label{eq:gap-inclusion}
\end{equation}

To prove this, note first that $G_\delta$ is closed in $\mathcal{X}$ because it is defined by finitely many inequalities
of continuous functions; hence $A_\delta=\mathcal{X}\setminus G_\delta$ is open in $\mathcal{X}$. Since we will use compactness,
we instead work with the closed superset
\[
\widetilde{A}_\delta := \Big\{x\in \mathcal{X}:\ |h_1(x)|\ge \delta \ \text{or}\ h_2(x)\ge \delta\ \text{or}\ \mathcal L(x)\ge \mathcal L^\star+\delta\Big\}.
\]
Each of the sets $\{|h_1|\ge \delta\}$, $\{h_2\ge \delta\}$, and $\{\mathcal L\ge \mathcal L^\star+\delta\}$ is closed in $\mathcal{X}$ by continuity,
so $\widetilde{A}_\delta$ is closed as a finite union of closed sets. Moreover, by definition we have the inclusion
$A_\delta \subseteq \widetilde{A}_\delta$. Since $\mathcal{X}$ is compact, it follows that $\widetilde{A}_\delta$ is compact.

Now define the continuous nonnegative function
\[
\Delta(x):=\mathcal J(x)-\mathcal J^\star \ge 0.
\]
Since $\widetilde{A}_\delta$ is compact and $\Delta$ is continuous, $\Delta$ attains its minimum on $\widetilde{A}_\delta$:
\[
\varepsilon := \min_{x\in \widetilde{A}_\delta} \Delta(x)
           = \min_{x\in \widetilde{A}_\delta} \big(\mathcal J(x)-\mathcal J^\star\big).
\]
We show $\varepsilon>0$. Suppose for contradiction that $\varepsilon=0$. Then there exists $x_0\in \widetilde{A}_\delta$ with
$\mathcal J(x_0)-\mathcal J^\star=0$, i.e.\ $\mathcal J(x_0)=\mathcal J^\star$, so $x_0\in S$. But by \cref{eq:Sc-in-Gdelta}, $S\subseteq G_{\delta/2}$, and hence
$|h_1(x_0)|\le \delta/2$, $h_2(x_0)\le \delta/2$, and $\mathcal L(x_0)\le \mathcal L^\star+\delta/2$, which implies $x_0\notin \widetilde{A}_\delta$.
This is a contradiction. Therefore $\varepsilon>0$, and by definition of $\varepsilon$ we have
\[
\mathcal J(x)-\mathcal J^\star \ge \varepsilon \quad \text{for all } x\in \widetilde{A}_\delta.
\]
Using $A_\delta\subseteq \widetilde{A}_\delta$, we conclude that
\[
\mathcal J(x)-\mathcal J^\star \ge \varepsilon \quad \text{for all } x\in A_\delta,
\]
which is exactly \cref{eq:gap-inclusion}.

Consequently,
\begin{equation}
p^{(\beta,c)}(A_\delta)
\le
p^{(\beta,c)}\big(\{x:\ \mathcal J(x)\ge \mathcal J^\star+\varepsilon\}\big)
\quad \text{for every } \beta<0.
\label{eq:reduce-to-energy-tail}
\end{equation}

\paragraph{Exponential bound on the energy tail as $\beta\to -\infty$}
Define the ``high-energy'' set
\[
B := \{x\in \mathcal{X}:\ \mathcal J(x)\ge \mathcal J^\star+\varepsilon\}.
\]
We will show that $p^{(\beta,c)}(B)\to 0$ as $\beta\to -\infty$, and in fact we will obtain an explicit bound.



By (A3), choose a minimizer \(x^\star\in S\) and constants \(r_0,m>0\) such that \(q_0\ge m\) on \(\mathcal{X}\cap B_{r_0}(x^\star)\), and every smaller ball centered at \(x^\star\) has an intersection with \(\mathcal{X}\) of positive Lebesgue measure. Since \(J(x^\star)=J^\star\), continuity of \(J\) gives a radius \(0<r\le r_0\) such that, with \(U':=\mathcal{X}\cap B_r(x^\star)\),
\begin{equation} \label{eq:local-sublevel}
    J(x)\le J^\star+\varepsilon/2
    \qquad\text{for all }x\in U'.
\end{equation}
Consequently,
\begin{equation} \label{eq:q-lower}
q_0(x)\ge m\quad\text{for all }x\in U',
\qquad
\operatorname{Leb}(U')>0.
\end{equation}

Also, since $q_1$ is continuous on compact $\mathcal{X}$, it is bounded above:
\begin{equation}
\exists\,M<\infty\ \text{ such that } \ q_0(x)\le M \quad \forall x\in \mathcal{X}.
\label{eq:q-upper}
\end{equation}

Now compute, for $\beta<0$,
\begin{equation}
p^{(\beta,c)}(B)
=
\frac{\int_B q_0(x)e^{\beta \mathcal J(x)}\,dx}{\int_{\mathcal{X}} q_0(x)e^{\beta \mathcal J(x)}\,dx}.
\label{eq:prob-ratio}
\end{equation}

\emph{Upper bound the numerator.}
For all $x\in B$, $\mathcal J(x)\ge \mathcal J^\star+\varepsilon$. Since $\beta<0$, the function $t\mapsto e^{\beta t}$ is decreasing, so
\[
e^{\beta \mathcal J(x)} \le e^{\beta (\mathcal J^\star+\varepsilon)} \quad \forall x\in B.
\]
Using \cref{eq:q-upper},
\[
\int_B q_0(x)e^{\beta \mathcal J(x)}\,dx
\le
\int_B M\, e^{\beta (\mathcal J^\star+\varepsilon)}\,dx
=
M\,\mathrm{Leb}(B)\,e^{\beta (\mathcal J^\star+\varepsilon)}
\le
M\,\mathrm{Leb}(\mathcal{X})\,e^{\beta (\mathcal J^\star+\varepsilon)}.
\]
Thus
\begin{equation}
\int_B q_0(x)e^{\beta \mathcal J(x)}\,dx
\le
M\,\mathrm{Leb}(\mathcal{X})\,e^{\beta (\mathcal J^\star+\varepsilon)}.
\label{eq:num-bound}
\end{equation}

\emph{Lower bound the denominator.}
Restrict the denominator integral to the set $U'$. For all $x\in U'$, by \cref{eq:local-sublevel} we have
$\mathcal J(x)\le \mathcal J^\star+\varepsilon/2$, and again since $\beta<0$,
\[
e^{\beta \mathcal J(x)} \ge e^{\beta (\mathcal J^\star+\varepsilon/2)} \quad \forall x\in U'.
\]
Using \cref{eq:q-lower},
\[
\int_{\mathcal{X}} q_0(x)e^{\beta \mathcal J(x)}\,dx
\ge
\int_{U'} q_0(x)e^{\beta \mathcal J(x)}\,dx
\ge
\int_{U'} m\, e^{\beta (\mathcal J^\star+\varepsilon/2)}\,dx
=
m\,\mathrm{Leb}(U')\, e^{\beta (\mathcal J^\star+\varepsilon/2)}.
\]
Hence
\begin{equation}
\int_{\mathcal{X}} q_0(x)e^{\beta \mathcal J(x)}\,dx
\ge
m\,\mathrm{Leb}(U')\, e^{\beta (\mathcal J^\star+\varepsilon/2)}.
\label{eq:den-bound}
\end{equation}

\emph{Combine bounds.}
Substituting \cref{eq:num-bound} and \cref{eq:den-bound} into \cref{eq:prob-ratio} yields
\[
p^{(\beta,c)}(B)
\le
\frac{M\,\mathrm{Leb}(\mathcal{X})\,e^{\beta (\mathcal J^\star+\varepsilon)}}{m\,\mathrm{Leb}(U')\, e^{\beta (\mathcal J^\star+\varepsilon/2)}}
=
\frac{M\,\mathrm{Leb}(\mathcal{X})}{m\,\mathrm{Leb}(U')}\, e^{\beta\,\varepsilon/2}.
\]
Therefore, for all $\beta<0$,
\begin{equation}
p^{(\beta,c)}(B)
\le
C\, e^{\beta\,\varepsilon/2},
\qquad
\text{where } C:=\frac{M\,\mathrm{Leb}(\mathcal{X})}{m\,\mathrm{Leb}(U')}\in(0,\infty).
\label{eq:exp-tail}
\end{equation}
The following estimate is the standard compact-space Gibbs concentration/Laplace-principle argument; see \citet{doi:10.1137/1036078, 10.1214/aop/1176994579}.
Since $\varepsilon>0$ and $\beta\to -\infty$, we have $e^{\beta\varepsilon/2}\to 0$, hence
$p^{(\beta,c)}(B)\to 0$. In particular, choose $\beta<0$ sufficiently negative so that
\begin{equation}
C\, e^{\beta\,\varepsilon/2} \le \eta.
\label{eq:choose-beta}
\end{equation}
For example, it suffices to take any $\beta<0$ satisfying
\[
\beta \le \frac{2}{\varepsilon}\log\!\Big(\frac{\eta}{C}\Big)
\quad
\]

\paragraph{Conclude the desired high-probability guarantee on $G_\delta$}
By \cref{eq:reduce-to-energy-tail} and \cref{eq:choose-beta},
\[
p^{(\beta,c)}(A_\delta)
\le
p^{(\beta,c)}(B)
\le
\eta.
\]
Therefore,
\[
p^{(\beta,c)}(G_\delta) = 1 - p^{(\beta,c)}(A_\delta) \ge 1-\eta.
\]
By the definition \cref{eq:Gdelta} of $G_\delta$, this is exactly
\[
\mathbb{P}\Big(|h_1(X)|\le \delta,\ \ h_2(X)\le \delta,\ \ \mathcal L(X)\le \mathcal L^\star+\delta\Big)\ \ge\ 1-\eta,
\quad \text{for } X\sim p^{(\beta,c)}.
\]
\end{proof}

\subsection{Finite-particle guarantee}

\cref{prop:approx-constrained-full} concerns the Gibbs posterior itself. We next transfer this
guarantee to the trajectory returned by a finite-particle sampler. We consider the multinomial-resampling variant of \cref{alg:bayesfp}. All other numerical
choices enter through the discrete Feynman--Kac target defined below.

\PaperStatementBegin{PaperPlum}
\begin{theorem}[Finite-particle BayesFP with multinomial resampling]
\label{thm:bayesfp-finite-particle}
Retain the notation of \cref{prop:approx-constrained-full}, including $G_\delta$ and
$A_\delta$. Fix $\delta>0$ and penalties $c$ for which \cref{prop:approx-constrained-full} holds,
and let $C,\varepsilon>0$ be its constants, so that
\begin{equation}
 p^{(\beta,c)}(A_\delta)
 \leq C\exp(\beta\varepsilon/2),
 \qquad \beta<0.
 \label{eq:bayesfp-gibbs-tail}
\end{equation}
Fix $\beta<0$, the numerical grid $(t_n)_{n=0}^{N}$ of \cref{alg:bayesfp} in
integration order, and a deterministic resampling schedule, independently
of $K$.

Let $N_{\mathrm{blk}}\geq1$ be the number of weighting blocks into which
propagation and weight accumulation are partitioned, with block index
$j\in\{1,\ldots,N_{\mathrm{blk}}\}$.
Each block ends at a scheduled
intermediate resampling or at the final particle selection. The first
block includes any initial importance weight, and the last includes all
remaining weights up to the terminal selection. An empty block has weight
one. Let $\mathsf E_j$ be the measurable space of numerical proposal
histories through the end of block $j$, and let $Y_j\in\mathsf E_j$ denote
such a history. Its scalar unnormalized incremental potential
$G_j:\mathsf E_j\to(0,\infty)$ is the product of the uncentered
weight factors actually retained within block $j$; it is a block weight,
distinct from the diffusion matrix $G_t$ used in \cref{app:flow_to_sde}.

Assume that the initial proposals are independent and identically
distributed (i.i.d.), that propagation uses independent randomness
conditional on the selected ancestors, and that every intermediate
resampling step is multinomial. Specifically, multinomial resampling
selects each offspring's ancestor independently, conditional on the current
weighted population, with probabilities given by its normalized weights.
Assume also that each block has deterministic bounds
$\underline g_j$ and $\overline g_j$ satisfying
\begin{equation}
 0<\underline g_j\leq G_j(y)\leq\overline g_j<\infty,
 \qquad y\in\mathsf E_j,\quad j=1,\ldots,N_{\mathrm{blk}}.
 \label{eq:bayesfp-bounded-potentials}
\end{equation}
Here $y$ is a generic block history. The potentials and mutation kernels,
meaning the conditional probability laws used to extend a selected
history through the next block, are fixed independently of $K$.

Let $\mathsf Q_N$ be the probability law of one numerical proposal path
$(x_{t_0},\ldots,x_{t_N})$ simulated without resampling. Each $Y_j$ is the
corresponding history through the end of block $j$. For a Borel measurable
trajectory event $A\subseteq\mathbb R^d$, let $\mathbf1_A(x)$ equal one
if $x\in A$ and zero otherwise. Define the numerical Feynman--Kac terminal
probability measure $\pi_N$ by
\begin{equation}
 \pi_N(A):=
 \frac{\mathbb E_{\mathsf Q_N}
 [\mathbf1_A(x_{t_N})\prod_{j=1}^{N_{\mathrm{blk}}}G_j(Y_j)]}
 {\mathbb E_{\mathsf Q_N}[\prod_{j=1}^{N_{\mathrm{blk}}}G_j(Y_j)]}.
 \label{eq:bayesfp-discrete-target}
\end{equation}
The bounds in \cref{eq:bayesfp-bounded-potentials} ensure that the
denominator is finite and strictly positive.

Using the terminal states $x_{t_N}^{(k)}$ and residual log-weights
$w_{t_N}^{(k)}$ of \cref{alg:bayesfp}, define the normalized residual weights
\[
 W_N^{(k)}:=
 \frac{e^{w_{t_N}^{(k)}}}{\sum_{\ell=1}^{K}e^{w_{t_N}^{(\ell)}}},
 \qquad k=1,\ldots,K,
\]
where $\ell$ is a summation index over the particles. These weights are
nonnegative and sum to one. The terminal random weighted empirical
measure and its probability assigned to an event $A$ are
\[
 \widehat\pi_N^K:=\sum_{k=1}^K W_N^{(k)}\delta_{x_{t_N}^{(k)}},
 \qquad
 \widehat\pi_N^K(A)=\sum_{k=1}^K W_N^{(k)}\mathbf1_A(x_{t_N}^{(k)}).
\]
Let $X_{\mathrm{out}}^K:=x_{t_N}^{(k^\star)}$ denote the trajectory returned by the
multinomial-resampling variant of Algorithm~1 considered here. Its final
selection rule is
\[
 X_{\mathrm{out}}^K\mid\widehat\pi_N^K\sim\widehat\pi_N^K,
\]
meaning that, conditional on the weighted population, it selects
$x_{t_N}^{(k)}$ with probability $W_N^{(k)}$. For probability measures
$\mu$ and $\nu$ on a common measurable space, define
\[
 \|\mu-\nu\|_{\mathrm{TV}}:=\sup_{A\text{ measurable}}|\mu(A)-\nu(A)|,
\]
where the supremum is over all measurable events in that space.
For any random element $Z$, write $\operatorname{Law}(Z)$ for its
probability distribution: for every measurable event $A$ in its state
space,
\[
 \operatorname{Law}(Z)(A):=\mathbb P(Z\in A).
\]
For a tuple of random elements, $\operatorname{Law}$ denotes their joint
probability distribution.
When comparing the Gibbs posterior with numerical measures on
$\mathbb R^d$, regard $p^{(\beta,c)}$ as assigning zero mass outside
$\mathcal X$. Define the deterministic target discrepancy
\[
 d_N:=\|\pi_N-p^{(\beta,c)}\|_{\mathrm{TV}}.
\]
Then there is a finite nonnegative constant $B_N$, depending on the fixed
Feynman--Kac model but not on $K$, such that, for every integer $K\geq1$,
\begin{equation}
 \mathbb P(X_{\mathrm{out}}^K\in G_\delta)
 \geq 1-C\exp(\beta\varepsilon/2)
          -d_N-\frac{B_N}{K}.
 \label{eq:bayesfp-finite-output}
\end{equation}
Unless a measure is specified in the subscript, probabilities and
expectations below include initialization, propagation, resampling, and
final particle selection, as well as auxiliary resampling draws when
introduced in the proofs. In particular, $d_N=0$ when the Feynman--Kac model
targets $p^{(\beta,c)}$ exactly.
\PaperStatementEnd
\end{theorem}
\PaperStatementClose

\begin{proof}
Group all propagation and weight accumulation between consecutive
resampling times into the path-valued states $Y_j$ and potentials $G_j$
defined above. The initial proposed paths are i.i.d.; subsequent stages
select ancestors multinomially and independently extend their paths. This
is the Feynman--Kac particle model defined in Section~1.1,
equations~(1.4)--(1.6), of \citet{delmoral2007sharp}, with their selection
parameter set to zero. Path-valued states are explicitly allowed in
equation~(1.2) of the same reference. Subtracting a common scalar from all log-weights,
as in \cref{alg:bayesfp}, leaves the normalized weights unchanged.

For analysis only, append a final multinomial resampling of $K$ particles
from $\widehat\pi_N^K$, with identity mutation, meaning that the selected
states are left unchanged. Its first particle has exactly the law of
$X_{\mathrm{out}}^K$. Under \cref{eq:bayesfp-bounded-potentials}, the
finite-particle propagation-of-chaos estimate in equation~(1.9) of
\citet{delmoral2007sharp}, applied to the distribution of a single
particle, gives
\begin{equation}
 \|\operatorname{Law}(X_{\mathrm{out}}^K)-\pi_N\|_{\mathrm{TV}}
 \leq \frac{B_N}{K}.
 \label{eq:bayesfp-marginal-smc-bound}
\end{equation}
If the particle state contains an entire path, apply the terminal-coordinate
map, which extracts the terminal trajectory $x_{t_N}$ from that path; total
variation cannot increase under a measurable map. Therefore,
\begin{align*}
 \mathbb P(X_{\mathrm{out}}^K\notin G_\delta)
 &\leq 1-\pi_N(G_\delta)+B_N/K\\
 &\leq p^{(\beta,c)}(A_\delta)+d_N+B_N/K\\
 &\leq C\exp(\beta\varepsilon/2)
          +d_N+B_N/K,
\end{align*}
where the last inequality is \cref{eq:bayesfp-gibbs-tail}. Taking complements proves the
claim.
\end{proof}

\paragraph{Choice of particle budget}
For an exactly targeted Gibbs posterior and any $\eta\in(0,1)$, first
choose the penalties as in Proposition~F.2, then choose $\beta<0$ so that
$C\exp(\beta\varepsilon/2)\leq\eta/2$, and finally
choose
\[
 K\geq\max\{1,\lceil2B_N/\eta\rceil\}.
\]
Then
$\mathbb P(X_{\mathrm{out}}^K\in G_\delta)\geq1-\eta$.
The order of these choices matters: $B_N$ may depend on $\beta$, $c$,
the grid, and the resampling schedule. For an approximate target, provided
$d_N\leq\eta/3$, the same conclusion follows by making each of the Gibbs
tail and particle terms at most $\eta/3$. No uniform bound in inverse
temperature or numerical resolution is asserted.

\PaperStatementBegin{PaperTeal}
\begin{corollary}[Concentration of the weighted particle population]
\label{cor:bayesfp-finite-population}
Under the conditions of ~\cref{thm:bayesfp-finite-particle}, the
constant $B_N$ can be chosen so that, for every integer $K\geq1$ and every
confidence failure probability $\alpha\in(0,1)$,
\begin{equation}
 \mathbb P\!\left(
 \widehat\pi_N^K(G_\delta)\geq
 1-C\exp(\beta\varepsilon/2)-d_N
 -\sqrt{\frac{1+8B_N}{K\alpha}}
 \right)\geq1-\alpha.
 \label{eq:bayesfp-finite-population}
\end{equation}
Thus $1-\alpha$ is the confidence level for the lower bound on the random
weighted mass $\widehat\pi_N^K(G_\delta)$ assigned to acceptable trajectories.
The outer probability is over construction of the weighted particle
population.
\PaperStatementEnd
\end{corollary}
\PaperStatementClose

\begin{proof}
For $i\in\{1,\ldots,K\}$, let $Z_i\in\mathbb R^d$ be the terminal
trajectory of the $i$th particle after the auxiliary final multinomial
resampling used above. Conditional on $\widehat\pi_N^K$, these $K$ draws
are independent and each has probability distribution $\widehat\pi_N^K$.
Let $q$ denote the number of particles considered jointly, and let
$\pi_N^{\otimes q}$ denote the product probability measure of $q$
independent draws from $\pi_N$, on $(\mathbb R^d)^q$.
The same propagation-of-chaos estimate, now applied to one or two
particles, allows $B_N$ to be chosen so that
\[
 \|\operatorname{Law}(Z_1,\ldots,Z_q)-\pi_N^{\otimes q}\|_{\mathrm{TV}}
 \leq B_Nq^2/K,\qquad q\in\{1,\ldots,\min\{2,K\}\}.
\]
Fix a Borel measurable event $A\subseteq\mathbb R^d$ and define the
centered indicators $U_i:=\mathbf1_A(Z_i)-\pi_N(A)$ for $i=1,\ldots,K$.
For $K\geq2$, the auxiliary particles are exchangeable, meaning that
their joint law is unchanged by any permutation of the particle indices.
Exchangeability and the $q=2$ bound imply
\[
 \mathbb E\!\left[\left(K^{-1}\sum_{i=1}^K U_i\right)^2\right]
 \leq \frac1K+\frac{K-1}{K}|\mathbb E[U_1U_2]|
 \leq \frac{1+8B_N}{K}.
\]
Indeed, $|U_1U_2|\leq1$, its expectation under
$\pi_N^{\otimes2}$ is zero, and differences of expectations of functions
bounded by one in absolute value are at most twice total variation.
Conditional on the weighted population,
\[
 \mathbb E\!\left[
 K^{-1}\sum_{i=1}^K\mathbf1_A(Z_i)\mid\widehat\pi_N^K
 \right]=\widehat\pi_N^K(A).
\]
Conditional Jensen's inequality consequently gives
\[
 \mathbb E[(\widehat\pi_N^K(A)-\pi_N(A))^2]
 \leq (1+8B_N)/K.
\]
For $K=1$ this inequality holds trivially. Apply Markov's inequality to
the squared error with $A=G_\delta$, using the threshold
$(1+8B_N)/(K\alpha)$, and use
$\pi_N(G_\delta)\geq1-C\exp(\beta\varepsilon/2)-d_N$.
\end{proof}

\paragraph{Scope for the implemented sampler}
The theorem separates finite-particle error from the discrepancy between
the numerical Feynman--Kac target and the Gibbs posterior. In particular,
time discretization, approximate scores, omitted weight increments, and
an incorrect initial tilted law can contribute to $d_N$; increasing $K$
alone does not remove these errors. Restricting the times at which
resampling is performed is compatible with the theorem, provided the
corresponding weights are retained. The stated particle bounds apply to
multinomial intermediate resampling. They do not automatically extend to
the systematic resampling used in our experiments; systematic resampling
can be inconsistent for some input orderings
\citep{gerber2019negative}.

\citet[Section~2.3]{naesseth2019elements} give a related bound on the
Kullback--Leibler divergence between the expected SMC empirical measure
and its target. The direct total-variation estimate used above yields an
$O(K^{-1})$ bound for the selected trajectory, where $O(K^{-1})$ means
bounded by a constant times $K^{-1}$ as $K$ increases with the entire
numerical model fixed. This does not assert an $O(K^{-1})$ error for each
realized particle population, whose fluctuation is instead controlled by
\cref{cor:bayesfp-finite-population}.

\subsubsection{Dependence of the target discrepancy on algorithmic parameters.}
For fixed
particle-update rules and retained weight factors, $d_N$ is independent of $K$ and of
the resampling period $R$, provided that changing the resampling times
preserves all accumulated weights and does not alter the particle
updates in \cref{alg:bayesfp}. The following results describe individual contributions to
$d_N$. Throughout this subsection, the cost $\mathcal J$ used in
\cref{alg:bayesfp} has the fixed penalties $c$, is defined on the proposal
states, and agrees with $\mathcal J_c$ on $X$. The
inverse-temperature schedule is fixed with terminal value
$\beta_{t_N}=\beta$.

\paragraph{Initial weighting and the active interval}
Consider \cref{alg:bayesfp}'s initialization from $p_{\mathrm{src}}$ with zero
log-weights. On a proposal path distributed according to $\mathsf Q_N$,
let $\Delta w_n$ denote the scalar log-weight increment computed at
update $n\in\{0,\ldots,N-1\}$ by the relevant branch of \cref{alg:bayesfp},
before centering and before checking the active interval. Thus
$\Delta w_n$ is the single-path counterpart of $\Delta w^{(k)}$ in
that algorithm; it includes the integration step used by the update.
This definition follows the integration order of $(t_n)_{n=0}^N$ and
does not impose a different time convention on either policy class.
Define the active update indices by
\[
 \mathcal I_N:=\{n\in\{0,\ldots,N-1\}:
 t_n\in[t_{\min},t_{\max}]\}.
\]
Define the two cumulative uncentered log-weights
\begin{align}
 w_{t_N}^{\mathrm{act}}
 &:=\sum_{n\in\mathcal I_N}\Delta w_n,
 \label{eq:f5-active-logweight}\\
 w_{t_N}^{\mathrm{full}}
 &:=\beta_{t_0}\mathcal J(x_{t_0})
       +\sum_{n=0}^{N-1}\Delta w_n.
 \label{eq:f5-full-logweight}
\end{align}
The superscripts $\mathrm{act}$ and $\mathrm{full}$ distinguish the
weights actually retained from those retaining every increment and the
initial Gibbs factor. These are cumulative weights of an unresampled
proposal path, not the residual particle log-weights after intermediate
resampling. For the initialization considered here,
$e^{w_{t_N}^{\mathrm{act}}}=\prod_{j=1}^{N_{\mathrm{blk}}}G_j(Y_j)$.
The initial factor in \cref{eq:f5-full-logweight} presumes that
$q_{t_0}=p_{\mathrm{src}}$; a mismatch between these source laws is a
separate discrepancy.

Let $\Delta w_N^{\mathrm{omit}}$ denote the total omitted log-weight:
\begin{equation}
 \Delta w_N^{\mathrm{omit}}
 :=w_{t_N}^{\mathrm{full}}-w_{t_N}^{\mathrm{act}}
 =\beta_{t_0}\mathcal J(x_{t_0})
   +\sum_{n\in\{0,\ldots,N-1\}\setminus\mathcal I_N}\Delta w_n.
 \label{eq:f5-omitted-logweight}
\end{equation}
Let $\pi_N^{\mathrm{full}}$ be the comparison terminal probability
measure using the same proposal law and the complete log-weight:
\begin{equation}
 \pi_N^{\mathrm{full}}(A):=
 \frac{\mathbb E_{\mathsf Q_N}
 [\mathbf1_A(x_{t_N})e^{w_{t_N}^{\mathrm{full}}}]}
 {\mathbb E_{\mathsf Q_N}[e^{w_{t_N}^{\mathrm{full}}}]}.
 \label{eq:f5-full-target}
\end{equation}
Assume that both cumulative log-weights have finite expected absolute
values under $\mathsf Q_N$ and finite essential upper bounds. Define
\[
 \overline w_N^{\mathrm{full}}
 :=\operatorname*{ess\,sup}_{\mathsf Q_N}w_{t_N}^{\mathrm{full}},
 \qquad
 \overline w_N^{\mathrm{act}}
 :=\operatorname*{ess\,sup}_{\mathsf Q_N}w_{t_N}^{\mathrm{act}},
\]
where the essential supremum is the least upper bound after discarding
sets of proposal paths with $\mathsf Q_N$-probability zero. Introduce
the shifted normalization constant $Z_N^{\mathrm{full}}$ and the
stability factor $\kappa_N$ by
\begin{align*}
 Z_N^{\mathrm{full}}
 &:=\mathbb E_{\mathsf Q_N}
 [e^{w_{t_N}^{\mathrm{full}}-\overline w_N^{\mathrm{full}}}],\\
 \kappa_N
 &:=\frac{e^{\max\{0,\overline w_N^{\mathrm{act}}
                        -\overline w_N^{\mathrm{full}}\}}}
              {Z_N^{\mathrm{full}}}.
\end{align*}
The assumptions imply $0<Z_N^{\mathrm{full}}\leq1$ and
$\kappa_N<\infty$. Theorem~8 of \citet{sprungk2020}, applied on the
proposal-path space with negative log-weights
$\overline w_N^{\mathrm{full}}-w_{t_N}^{\mathrm{full}}$ and
$\overline w_N^{\mathrm{full}}-w_{t_N}^{\mathrm{act}}$, gives
\begin{equation}
 \|\pi_N-\pi_N^{\mathrm{full}}\|_{\mathrm{TV}}
 \leq\kappa_N\,
       \mathbb E_{\mathsf Q_N}|\Delta w_N^{\mathrm{omit}}|.
 \label{eq:f5-omitted-weight-bound}
\end{equation}
Here the bound for path measures also bounds their terminal marginals,
because total variation cannot increase under the terminal-coordinate
map. Define the remaining discrepancy with complete weighting by
\[
 e_N^{\mathrm{full}}
 :=\|\pi_N^{\mathrm{full}}-p^{(\beta,c)}\|_{\mathrm{TV}}.
\]
The triangle inequality then gives
\begin{equation}
 d_N\leq e_N^{\mathrm{full}}
       +\kappa_N\mathbb E_{\mathsf Q_N}|\Delta w_N^{\mathrm{omit}}|.
 \label{eq:f5-target-discrepancy}
\end{equation}
Retaining every increment and supplying the initial Gibbs factor removes
the corresponding omissions, but does not eliminate numerical
propagation or learned-model error in $e_N^{\mathrm{full}}$.
The increment $\Delta w_n$, the proposal law $\mathsf Q_N$, and hence
the bound can depend on $\beta_t$, $c$, the grid, and the diffusion
schedule. In particular, the inverse-normalizer dependence in
$\kappa_N$ permits amplification of weight errors; no uniformity in
these parameters or monotonicity of $d_N$ is asserted.

\paragraph{Perturbation of the initial law}
An alternative comparison isolates initial tilting. Assume that the
required initial Gibbs marginal has a finite, strictly positive
normalization constant, and let
$\pi_N^{\mathrm{init}}$ denote the terminal probability measure of the
same numerical model when the initial law $p_{\mathrm{src}}$ is replaced
by $p_{t_0}(\cdot\mid O=1)$, while all subsequent particle-update rules and
retained weight factors remain fixed. Under $q_{t_0}=p_{\mathrm{src}}$,
this is the required initial Gibbs marginal,
$p_{t_0}(\cdot\mid O=1)\propto
p_{\mathrm{src}}\exp(\beta_{t_0}\mathcal J)$.
Its initial Gibbs factor is supplied by the change of initial law and
is not also included in the first block potential.
The stability estimates of \citet[Lemma~6 and Section~2.1 following
Lemma~11]{giraud2017}, specialized using the block-potential bounds of
the finite-particle theorem, give
\begin{equation}
 \|\pi_N-\pi_N^{\mathrm{init}}\|_{\mathrm{TV}}
 \leq
 \left(\prod_{j=1}^{N_{\mathrm{blk}}}
       \frac{\overline g_j}{\underline g_j}\right)
 \|p_{\mathrm{src}}-p_{t_0}(\cdot\mid O=1)\|_{\mathrm{TV}}.
 \label{eq:f5-initial-law-bound}
\end{equation}
This bound concerns the same zero-initial-weight numerical model as
above. It is an alternative to counting the initial-tilt omission in
\cref{eq:f5-omitted-weight-bound}; that contribution should not be
counted twice.

\paragraph{Discretization of the FK potential integral}
Let $\pi_N^{\mathrm{quad}}$ denote the terminal FK probability measure
obtained with exact reference dynamics, initialization from the exact
initial Gibbs marginal, and left-endpoint quadrature of the continuous
log-weight integral using $N$ equally spaced steps over the unit
generation interval. All weight increments are retained in this
comparison. To apply the cited continuous-time result, the reference
dynamics and the log-weight rate are expressed in increasing elapsed
generation time; for the flow branch, this entails consistently
reversing its decreasing model-time coordinate.
Suppose that the exact continuous FK terminal law is $p^{(\beta,c)}$
and that the regularity assumptions in case~(C), Section~2.1, of
\citet{delmoral2013geometric} hold. Their Theorem~3.1 yields
\begin{equation}
 \|\pi_N^{\mathrm{quad}}-p^{(\beta,c)}\|_{\mathrm{TV}}
 \leq\frac{C_{\mathrm{quad}}}{N},
 \label{eq:f5-potential-quadrature-bound}
\end{equation}
where $C_{\mathrm{quad}}$ is a finite nonnegative constant independent
of $N$ for the fixed continuous model and physical horizon. It may
depend on $\beta_t$, $c$, and the dynamics. Case~(C) requires a bounded,
time-differentiable potential in a common generator domain, together
with the domain-invariance and semigroup regularity conditions stated
there. This estimate concerns potential quadrature with exact
propagation; it does not by itself cover Euler--Maruyama or DDPM
propagation, or errors in the learned score or velocity. We therefore
retain $d_N$ explicitly without asserting a rate for the full numerical
implementation.

%% file: Appendix/experimental_deatils.tex
\appsection{Experimental Details}
\label{sec:experimental_details}

This appendix collects the experimental details that were deferred from \cref{sec:experiments}: the precise form of the obstacle-avoidance cost we use throughout, the task definitions for each benchmark, and the training/evaluation protocol.

\subsection{Obstacle-Avoidance Cost Function}
\label{sec:cost_function}

All collision-avoidance experiments in the paper use the same generic cost; the only thing that changes from task to task is the obstacle geometry (or, in the case of $\pi_{0.5}$, the SDF that represents the scene). We describe the cost here once for a generic obstacle and then specialize it in the relevant subsections below.

\paragraph{Robot body discretization}
Given an action waypoint $a$ predicted by the policy (an end-effector pose, an end-effector delta, or a joint-position command, depending on the experiment), we map it to a set of $N_{\mathrm{body}}$ sampled body points
\begin{equation}
\mathcal{B}(a) \;=\; \bigl\{p_1(a),\, p_2(a),\, \ldots,\, p_{N_{\mathrm{body}}}(a)\bigr\} \;\subset\; \mathbb{R}^3
\end{equation}
in the world frame. For policies that predict end-effector poses or end-effector deltas (Diffusion Policy, GR00T-N1.6), $\mathcal{B}(a)$ samples points on the gripper and the last few links of the arm. For policies that predict joint positions ($\pi_{0.5}$), $\mathcal{B}(a)$ is obtained by differentiable forward kinematics applied to the predicted joint configuration and sampling points densely along each link.

\paragraph{Per-point violation}
Let $\mathcal{O} \subset \mathbb{R}^3$ denote the obstacle region and let
\begin{equation}
s_{\mathcal{O}}(p) \;:=\; \operatorname{sdf}_{\mathcal{O}}(p)
\end{equation}
be its signed distance function, with the convention that $s_{\mathcal{O}}(p) > 0$ outside the obstacle and $s_{\mathcal{O}}(p) < 0$ inside. We define the \emph{clearance} of body point $p$ with respect to $\mathcal{O}$ as
\begin{equation}
c_{\mathcal{O}}(p) \;:=\; s_{\mathcal{O}}(p) \;-\; \varepsilon ,
\end{equation}
where $\varepsilon > 0$ is a safety margin that inflates the obstacle slightly to keep the robot from grazing its boundary. The corresponding per-point violation is the inward penetration depth into the inflated obstacle,
\begin{equation}
\mathrm{viol}_{\mathcal{O}}(p) \;:=\; \max\bigl\{0,\, -c_{\mathcal{O}}(p)\bigr\} \;=\; \max\bigl\{0,\, \varepsilon - s_{\mathcal{O}}(p)\bigr\} \;\ge\; 0 .
\label{eq:violation}
\end{equation}
$\mathrm{viol}_{\mathcal{O}}(p)$ is zero whenever $p$ lies outside the inflated obstacle and grows linearly with how far it penetrates inside.

\paragraph{Trajectory-level penalty}
Given a trajectory $x = (a_1, \ldots, a_H)$ of $H$ predicted action waypoints, the collision penalty aggregates \cref{eq:violation} over all body points sampled along the trajectory and over all obstacles $\mathcal{O}_k$ specified at inference time:
\begin{equation}
\mathcal{J}_{\mathrm{coll}}(x) \;=\; \sum_{h=1}^{H}\;\sum_{p \in \mathcal{B}(a_t)}\;\sum_{k}\, \mathrm{viol}_{\mathcal{O}_k}(p).
\label{eq:total_collision_cost}
\end{equation}
This is the $\rho(x)$ term in \cref{eq:agumented_cost_function} of the main paper; we set $\mathcal{L} \equiv 0$ and have no equality/inequality constraints $h_1, h_2$ in the collision-avoidance experiments, so $\mathcal{J}(x) = \mathcal{J}_{\mathrm{coll}}(x)$.

In practice $\mathrm{viol}_{\mathcal{O}}$ is non-smooth at the boundary $s_{\mathcal{O}}(p) = \varepsilon$. As noted in the footnote to \cref{eq:agumented_cost_function}, our theory uses a $C^2$ softplus relaxation $\tfrac{1}{\beta_{\mathrm{sp}}}\operatorname{softplus}\!\bigl(-\beta_{\mathrm{sp}}\, c_{\mathcal{O}}(p)\bigr)$, but empirically the hard clamp in \cref{eq:violation} works without issue and is what we use to produce the reported numbers. Throughout, $\varepsilon$ is on the order of a few centimeters (e.g.\ $\varepsilon = 0.02$\,m for the cylindrical-obstacle experiments).

\paragraph{Specialization to the cylindrical obstacle (Diffusion Policy, GR00T-N1.6)}
For the experiments on Diffusion Policy and GR00T-N1.6, the inference-time obstacle is a vertical cylinder of radius $r$ centered at $(c_x, c_y)$ and extruded along the world $z$-axis. Its signed distance function reduces to the planar radial distance,
\begin{equation}
s_{\mathcal{O}}(p) \;=\; \bigl\| p_{xy} - (c_x, c_y) \bigr\|_2 \;-\; r ,
\end{equation}
so the per-point violation in \cref{eq:violation} becomes
\begin{equation}
\label{eq:viol}
\mathrm{viol}_{\mathcal{O}}(p) \;=\; \max\!\Bigl\{0,\; (r + \varepsilon) - \bigl\| p_{xy} - (c_x, c_y) \bigr\|_2 \Bigr\}.
\end{equation}
The V-shaped obstacle used in the non-convex LIBERO experiments is implemented analogously as the union of two thick rectangular prisms joined at an angle, with its SDF computed as the minimum of the two prism SDFs.

\paragraph{Specialization to the online SDF ($\pi_{0.5}$)}
For the $\pi_{0.5}$ experiments we do not have access to obstacle geometry. Instead, we follow \cref{sec:sim_exps} of the main paper and build an online signed distance function $\hat{s}(p)$ of the scene with \texttt{nvblox}~\citep{millane2024nvbloxgpuacceleratedincrementalsigned} from the simulator's depth observation. All objects that are not the current manipulation/grasp target, together with the supporting table surface, are marked as collidable. We then use $s_{\mathcal{O}}(p) \equiv \hat{s}(p)$ inside \cref{eq:violation} and \cref{eq:total_collision_cost}.

\subsection{Constraint Violation Metric Used in Ablations}
\label{app:constraint_violation_metric}
Let $N$ be the number of evaluation episodes. For each episode
$i \in \{1, \ldots, N\}$ we roll out the policy in simulation and record the
executed end-effector position $\mathbf{e}_t^{(i)} \in \mathbb{R}^3$ at every
simulator step $t \in \{0, 1, \ldots, T_i - 1\}$. Let
$\mathcal{O} \subset \mathbb{R}^3$ denote the inference-time obstacle region
and let $s_{\mathcal{O}}: \mathbb{R}^3 \to \mathbb{R}$ be its signed distance
function, with the convention $s_{\mathcal{O}}(\mathbf{p}) > 0$ outside
$\mathcal{O}$ and $s_{\mathcal{O}}(\mathbf{p}) < 0$ inside, exactly as before. 
We define the instantaneous \emph{penetration depth} at
step $t$ of episode $i$ as
\begin{equation}
    p_t^{(i)} \;\triangleq\;
    \bigl[\,-s_{\mathcal{O}}\bigl(\mathbf{e}_t^{(i)}\bigr)\,\bigr]_+
    \;=\;
    \max\!\Bigl(0,\;-\,s_{\mathcal{O}}\bigl(\mathbf{e}_t^{(i)}\bigr)\Bigr),
\end{equation}
i.e.\ the depth by which the executed end-effector has crossed the obstacle
boundary $\partial\mathcal{O} = \{\mathbf{p} : s_{\mathcal{O}}(\mathbf{p}) = 0\}$
($0$ whenever the end-effector lies outside $\mathcal{O}$). Note that unlike
the per-point violation $\mathrm{viol}_{\mathcal{O}}$ in~\cref{eq:viol}
used inside the cost $\mathcal{J}_{\mathrm{coll}}$, the metric uses zero
inflation ($\varepsilon = 0$) so that $p_t^{(i)}$ reports geometric
penetration into $\mathcal{O}$ rather than penetration into the inflated
obstacle $\{\mathbf{p} : s_{\mathcal{O}}(\mathbf{p}) < \varepsilon\}$. We
aggregate along the executed trajectory to obtain the per-episode
\emph{total penetration},
\begin{equation}
    P^{(i)} \;\triangleq\; \sum_{t=0}^{T_i - 1} p_t^{(i)},
\end{equation}
and average across episodes to obtain the \emph{mean total penetration},
\begin{equation}
    \bar{P} \;\triangleq\; \frac{1}{N}\sum_{i=1}^{N} P^{(i)}.
    \label{eq:mean-total-penetration}
\end{equation}
By construction $\bar{P} \ge 0$, with $\bar{P} = 0$ iff no executed
trajectory ever enters $\mathcal{O}$. Larger values of $\bar{P}$ indicate
either deeper incursions, longer dwell times inside the obstacle, or both.
Because $p_t^{(i)}$ is summed over discrete simulator steps of fixed
duration $\Delta t$, $\bar{P}$ is proportional, up to the constant factor
$\Delta t$, to the per-episode time-integrated penetration
$\int_0^{T_i \Delta t} p^{(i)}(\tau)\, d\tau$ averaged across episodes. The
obstacle region $\mathcal{O}$ and its SDF $s_{\mathcal{O}}$ are specialized
per task as described in \cref{sec:cost_function} (e.g., a vertical cylinder of
radius $R$ centered at $\mathbf{c}$ for the Push-T and RoboMimic
experiments, a V-shaped union of two prisms for the non-convex LIBERO
experiment, and the online nvblox SDF for the $\pi_{0.5}$ experiments).
Throughout we use $\Delta t = 0.05\,\mathrm{s}$ and $N = 50$ episodes per
configuration.

\subsection{Task Descriptions and Datasets}

\paragraph{Diffusion Policy on RoboMimic (Can, Transport)}
We use the pretrained image-based Diffusion Policy checkpoints released by \citet{chi2024diffusionpolicy} for the \emph{Can} and \emph{Transport} tasks from the RoboMimic benchmark~\citep{mandlekar2021what}, both trained on the \emph{Proficient-Human (PH)} subset (200 successful demonstrations collected by a single experienced teleoperator). In \emph{Can}, a single 7-DoF Franka Panda arm must pick a soda can from one bin and place it into an adjacent target bin. In \emph{Transport}, two Franka Panda arms must cooperatively transfer a hammer from a starting bin to a target bin, requiring a mid-air handoff between the two arms. The policy outputs a horizon of absolute end-effector position waypoints in the world frame. At inference time we introduce an additional cylindrical obstacle, centered between the bins for \emph{Can} and at the bimanual handoff location for \emph{Transport}; the obstacle is never seen during training and the policy must avoid it while still completing the task. We evaluate over 50 rollouts per condition with randomized initial object poses.

\paragraph{GR00T-N1.6 on LIBERO-Object}
We fine-tune the pretrained GR00T-N1.6 foundation model~\citep{nvidia2025gr00tn1openfoundation} on the LIBERO-Object suite~\citep{liu2023liberobenchmarkingknowledgetransfer} using the released \emph{Proficient-Human (PH)} demonstrations (50 teleoperated trajectories per task). LIBERO-Object consists of 10 pick-and-place tasks in which a Franka Panda arm must pick up a designated grocery item from a kitchen tabletop and place it into a white basket; tasks differ only in the identity of the target object (alphabet soup, butter, cream cheese, tomato sauce, ketchup, salad dressing, milk, orange juice, BBQ sauce, chocolate pudding). For the inference-time collision experiments we restrict evaluation to the \emph{``Pick up the butter and place it in the basket''} task. We do this because, once a cylindrical or V-shaped translucent obstacle is inserted between the table and the basket, most of the other LIBERO-Object tasks become \emph{geometrically infeasible}: the target object can only be lifted and transported along a path that forces the gripper through the obstacle, regardless of which collision-aware planner is used. The butter task is the only one for which the object pose and basket pose admit a collision-free pick-and-place trajectory with the obstacle in place, isolating constraint satisfaction from feasibility. The policy outputs end-effector action deltas in the world frame; we evaluate over 100 rollouts per condition with randomized initial scene configurations.

\paragraph{$\pi_{0.5}$ on RoboLab}
We use the pretrained $\pi_{0.5}$~\citep{intelligence2025pi05visionlanguageactionmodelopenworld} checkpoint that has been fine-tuned on the DROID dataset~\citep{khazatsky2025droidlargescaleinthewildrobot} and evaluate it zero-shot (no further fine-tuning) on four RoboLab~\citep{yang2026robolab} tasks that lie outside the $\pi_{0.5}$ training distribution. RoboLab is an Isaac Lab-based benchmark of 120+ tabletop manipulation tasks with photorealistic rendering and automated success detection via composable predicates. The four tasks we use are:
\begin{itemize}
\item \textbf{FoodPacking2CansTask.} A Franka Panda arm must pick up two soup-can-style food items from a cluttered tabletop and place them upright inside a designated packing bin.
\item \textbf{BBQSauceInBinTask.} The arm must pick bottle(s) of BBQ sauce from amongst a set of distractor grocery items and deposit ithem into a target bin without knocking over the neighboring items.
\item \textbf{HammersInLeftBinTask.} The arm must select hammer-shaped tool objects from a workshop scene with multiple bins and place each one into the specifically designated left bin.
\item \textbf{TakeMugsOffOfShelfTask.} The arm must pick up mugs resting on a multi-tier shelf and place them on the table surface in front of the shelf, requiring reaching into a constrained, partially-enclosed workspace.
\end{itemize}
Across all four tasks $\pi_{0.5}$ outputs 7-DoF joint-position commands rather than Cartesian poses; we therefore differentiate through forward kinematics to map joint trajectories into sampled body points for evaluation of \cref{eq:total_collision_cost}. Because RoboLab scenes contain many objects whose poses are not directly exposed to the policy, we use the online SDF construction described above rather than analytic obstacle geometry. We treat \emph{every} non-target object (including the table) as a collidable obstacle. For each task we 
report the success rate of completing the task as defined by RoboLab's built-in predicates and the rate of any sampled body point penetrating the SDF zero level set during the rollout.




\paragraph{Real-world SO101 (PickAndPlace with mug obstacle; BimodalMugSelection)}
We additionally validate our approach on a low-cost SO101 arm~\cite{so101_github}
on two real-world tasks, by fine-tuning
$\pi_{0.5}$~\cite{intelligence2025pi05visionlanguageactionmodelopenworld}  with
only $20$ teleoperated demonstrations per task. Demonstrations are
collected via leader-follower teleoperation, and the policy outputs
joint-position commands. As in our $\pi_{0.5}$ experiments, we map joint
trajectories to sampled body points by differentiating through forward
kinematics, and we use two cameras an vanilla RGB and a RBG-D (Intel RealSense LiDAR Camera L515) camera. Every object that is not the current
manipulation target is treated as a collidable obstacle.

\begin{itemize}
    \item \textbf{Mug as obstacle PickAndPlace.} The training data consists of
     $20$ demonstrations in which the arm picks a yellow play-toy from a fixed
     region of the workspace and places it inside a saucer at a fixed goal
     location, with no obstacles present along the path. At inference time we
     introduce a ceramic mug, never seen during training, placed in the
     straight-line path between the pick and place locations.
     The mug therefore acts as a purely inference-time obstacle: the
     pretrained policy has no demonstrations of avoidance behavior, and the
     FKC sampler must reshape the learned trajectory distribution online to
     route around it.
     
     \item \textbf{BimodalMugSelection} The training data
     consists of $20$ demonstrations in which the arm picks up a small yellow
     play-toy from the table and places it into one of two ceramic mugs
     positioned to the left and right of the workspace; the (20) demonstrations
     are split evenly between the two destinations, so the pretrained policy
     is bimodal over the placement target. At inference we condition the
     language input on ``place the yellow object in the right mug'' and
     additionally introduce an inference-time constraint forbidding the
     region around the left mug, encoded through the same SDF-based collision
     cost as above. This experiment probes a complementary use of our
     framework: rather than avoiding an object absent from training, the
     cost is used to \emph{suppress one of two trained-in modes} whose
     selection cannot be reliably forced through language conditioning alone.
\end{itemize}

\subsection{Implementation Details}

For reweighting we use the Sequential Monte Carlo scheme of \cref{sec:resampling} throughout---systematic
resampling over the active interval $[t_{\min},t_{\max}]$---and \emph{not} the Markov
jump-process interpretation; we found SMC simpler to implement and sufficient on every task.
We additionally center the per-step log-weight increment across the $K$ particles before
accumulation: an unbiased baseline that keeps the effective sample size well conditioned. We
discretize the weighted SDE with a first-order Euler--Maruyama integrator for the two flow
policies. For the diffusion policy this is an implementation choice rather than a mathematical one:
the released Diffusion Policy checkpoint ships with the DDPM noise scheduler and a trained
$\epsilon$-network rather than a separately learned score model, so the score enters
\cref{alg:bayesfp}'s drift only through the identity
$\nabla\log q_t(x) = -\epsilon_\theta(x,t)/\sigma_t'$ supplied by the $\epsilon$-network. And this is done as follows:
Rather than converting $\epsilon_\theta$ back into a score and running a separate Euler--Maruyama
loop, we feed the scheduler's native DDPM ancestral update a cost-tilted noise prediction
\[
  \tilde\epsilon_\theta(x,t) \;:=\; \epsilon_\theta(x,t) \;-\; \tfrac{1}{2}\,\beta_t\,
  \sigma_t'\,\nabla\mathcal{J}(x),
  \qquad \sigma_t' := \sqrt{1-\bar\alpha_t},
\]
which substitutes the cost-tilted score $s+\tfrac{\beta_t}{2}\nabla\mathcal{J}$ into the
$\epsilon$-parametrization via the identity $s = -\epsilon_\theta/\sigma_t'$.
This is equivalent to running Euler--Maruyama on the cost-tilted VP-SDE in the continuous-time
limit---the DDPM ancestral step is a specific discretization of that SDE with per-step
variance increment $\sigma_t^2\,\Delta t = 1-\bar\alpha_t/\bar\alpha_{t-1}$, and the FK
log-weight in \cref{alg:bayesfp} uses exactly this matching increment. At the end of $N$
steps we draw a single multinomial sample from the residual log-weights to return one action
per batch element.

\paragraph{Sigma schedule} The Brownian diffusivity $\sigma_t$ is what lifts a deterministic
flow into the marginal-preserving SDE that admits Feynman--Kac tilting. For both flow
policies (\(\pi_{0.5}\) and GR00T-N1.6) we use the \citet{singh2024stochastic} ``zero-ends''
schedule
\[
  \sigma_t \;=\; \alpha\sqrt{t(1-t)}, \qquad \alpha = 0.25,
\]
which peaks mid-denoising and vanishes at the data and noise endpoints; this matched the best
ImageNet schedule in  \citet{singh2024stochastic} and we found it sufficient on every task. For the
DDPM diffusion policy we inherit the variance-preserving (VP) diffusivity of the trained
scheduler, $\sigma_t^2\,\Delta t = 1-\bar\alpha_t/\bar\alpha_{t-1}$, with no additional knob;
this keeps the proposal identical to the standard DDPM ancestral sampler under $\beta=0$. In
both cases we use a constant inverse-temperature schedule
$\beta_t \equiv -\gamma$ with $\gamma>0$, so the $(\beta_{t+\Delta t}-\beta_t)\mathcal{J}$
term in the FK weight vanishes and only the inner-product term survives; the strength
$\gamma$ is the dominant guidance knob and is varied in the main paper.


\paragraph{Active interval} We use the same window
$[t_{\min},t_{\max}] = [0.05,\,0.95]$ across all three policies, clipping only the very first
and last steps where the imputed score / velocity is least reliable (the score expression
$1/t$ or $1/(1-t)$ blows up at the endpoints). 

\paragraph{Number of particles} We use $K=32$ particles for the diffusion policy, $K=8$ for
GR00T-N1.6, and $K=4$ for \(\pi_{0.5}\) (limited by the larger per-step backbone /
VLM cost). The particle-count ablation of \cref{fig:ex:A}, where $K$ is varied,
is the only exception.

%% file: Appendix/notation.tex
%
%
%

\appsection{Notation}
\label{app:notation}

This appendix collects the notation used throughout the paper. Symbols are
grouped by role. A short list of intentionally overloaded / context-dependent
symbols is given at the end of the section.

\renewcommand{\arraystretch}{1.25}
\setlength{\LTcapwidth}{\textwidth}

\begin{longtable}{@{}>{\columncolor{white}[0pt][0.55em]\raggedright\arraybackslash}p{0.30\textwidth}@{\hspace{1.1em}}>{\columncolor{white}[0.55em][0pt]\raggedright\arraybackslash}p{0.63\textwidth}@{}}
\caption{Summary of notation used in the paper.}\label{tab:notation}\\
\toprule
\PaperTableHead
\textbf{Symbol} & \textbf{Description} \\
\midrule
\endfirsthead

\multicolumn{2}{@{}l}{\small\itshape \tablename~\thetable\ (continued)}\\
\toprule
\PaperTableHead
\textbf{Symbol} & \textbf{Description} \\
\midrule
\endhead

\midrule
\multicolumn{2}{r@{}}{\small\itshape continued on next page}\\
\endfoot

\bottomrule
\endlastfoot

\addlinespace[0.4em]
\multicolumn{2}{@{}>{\columncolor{PaperTeal!4}[0pt][0pt]}l@{}}{\textbf{Sets, spaces, and operators}}\\
\midrule
$\mathbb{R}^d$ & $d$-dimensional Euclidean space in which trajectories / actions live. \\
$d$ & Dimensionality of the trajectory (control-input) space. \\
$\mathcal{X}\subset\mathbb{R}^d$ & Feasible domain of the optimization; in the proofs, a compact set with nonempty interior. \\
$\mathcal{F}$ & Feasible set $\{x\in\mathcal{X}:h_1(x)=0,\ h_2(x)\le 0\}$. \\
$I_d,\ I$ & $d\times d$ identity matrix. \\
$\nabla$ & Gradient operator (with respect to the state $x$). \\
$\nabla\!\cdot,\ \langle\nabla,\,\cdot\,\rangle$ & Divergence operator; the two notations are used interchangeably. \\
$\Delta$ & Laplacian operator, $\Delta=\nabla\!\cdot\!\nabla$. \\
$\langle\,\cdot\,,\,\cdot\,\rangle$ & Euclidean inner product. \\
$\lVert\,\cdot\,\rVert_2$ & Euclidean ($\ell_2$) norm. \\
$\mathbb{E}[\cdot]$ & Expectation; a subscript denotes the distribution. \\
$\mathbb{P}(\cdot)$ & Probability of an event. \\
$\mathrm{Leb}(\cdot)$ & Lebesgue measure of a set. \\
$(u)^{+}$ & Positive part, $\max(0,u)$. \\
$(u)^{-}$ & Negative part, $\max(0,-u)$. \\
$\delta_z(\cdot)$ & Dirac measure centered at $z$. \\

\addlinespace[0.4em]
\multicolumn{2}{@{}>{\columncolor{PaperTeal!4}[0pt][0pt]}l@{}}{\textbf{Time variables}}\\
\midrule
$\tau\in[0,1]$ & Forward (noising) time; $\tau=0$ is clean data and $\tau=1$ is noise. \\
$t\in[0,1]$ & Reverse (denoising / generation) time, $t=1-\tau$. See the note on time conventions below. \\

\addlinespace[0.4em]
\multicolumn{2}{@{}>{\columncolor{PaperTeal!4}[0pt][0pt]}l@{}}{\textbf{States and trajectories}}\\
\midrule
$x$ & A trajectory / sequence of control inputs (action deltas, joint positions, or end-effector poses) drawn from or generated by the policy. \\
$x_\tau$ & State along the forward noising SDE at time $\tau$. \\
$x_t$ & State along the reverse generation process at time $t$. \\
$x_{\mathrm{data}}$ & A clean expert demonstration, $x_{\mathrm{data}}\sim p_{\mathrm{data}}$. \\
$x_{\mathrm{src}}$ & A source-noise sample, $x_{\mathrm{src}}\sim p_{\mathrm{src}}=\mathcal{N}(0,I_d)$. \\
$\tilde{x}$ & Dummy integration variable (e.g., in normalization integrals). \\
$x_t^{(k)}$ & State of the $k$-th Feynman--Kac particle at time $t$. \\
$h(x_{\mathrm{data}},x_{\mathrm{src}},t)$ & Time-differentiable interpolation between data and source endpoints; default linear interpolation $x_t=(1-t)x_{\mathrm{data}}+t\,x_{\mathrm{src}}$. \\

\addlinespace[0.4em]
\multicolumn{2}{@{}>{\columncolor{PaperTeal!4}[0pt][0pt]}l@{}}{\textbf{Distributions and densities}}\\
\midrule
$p_{\mathrm{data}}(x)$ & Expert demonstration distribution; the Bayesian prior. Shorthand for $p_{\mathrm{data}}(x\mid o)$ (always observation-conditioned). \\
$p_{\mathrm{src}}(x)$ & Tractable source distribution, typically $\mathcal{N}(0,I_d)$. \\
$\mathcal{N}(0,I_d)$ & Standard $d$-dimensional Gaussian. \\
$\nu(x_{\mathrm{data}},x_{\mathrm{src}})$ & Coupling distribution over (data, source) pairs; the independent coupling $\nu\equiv p_{\mathrm{data}}\,p_{\mathrm{src}}$ is a common choice. \\
$p_t(x)$ & Marginal density at reverse-time $t$, $p_t=p_{1-\tau}$ (overloaded; see note below). \\
$q_t(x)$ & Time-$t$ marginal density of the pretrained diffusion / flow policy (the base process). \\
$p(x\mid O=1)$ & Target posterior over optimal trajectories, $\propto p_{\mathrm{data}}(x)\exp(\beta\mathcal{J}(x))$. \\
$p_t(x\mid O=1)$ & Time-$t$ posterior (Gibbs-tilted) marginals, $\propto q_t(x)\exp(\beta\mathcal{J}(x))$. \\
$p^{(\beta,c)}(\mathrm{d}x)$ & Gibbs-tilted distribution with inverse temperature $\beta$ and penalties $c$ (proof notation). \\
$p_t^{\mathrm{FK}}(x)$ & Feynman--Kac weighted density (the density of the reweighted particle population). \\
$p_t^{\mathrm{ode}}(x)$ & Density evolving purely under deterministic transport (continuity equation). \\
$p_t^{\mathrm{diff}}(x)$ & Density evolving purely under Brownian diffusion (heat equation). \\
$p_t^{\mathrm{sde}}(x)$ & Density of an (unweighted) SDE, governed by the Fokker--Planck equation. \\
$p_t^{w}(x)$ & Weighted density induced by the log-weights $w_t$ (reweighting equation). \\
$p_t^{\mathrm{jump}}(x)$ & Density evolving under the Markov jump-process reweighting. \\
$p_t^{K}(z)$ & Empirical particle approximation $\tfrac{1}{K}\sum_k\delta_z(x^{(k)})$ of $p_t$. \\
$p(x_{\mathrm{data}}\mid x_t)$ & Posterior over the clean sample given a noised state (used in score imputation). \\
$p_t(x_t\mid x_{\mathrm{data}})$ & Conditional (Gaussian) density of the interpolant given the data endpoint; mean $(1-t)x_{\mathrm{data}}$, covariance $t^2I_d$. \\
$Z_t$ & Normalization constant (partition function) of the tilted marginal, $Z_t=\int q_t(\tilde{x})\exp(\beta\mathcal{J}(\tilde{x}))\,\mathrm{d}\tilde{x}$. \\

\addlinespace[0.4em]
\multicolumn{2}{@{}>{\columncolor{PaperTeal!4}[0pt][0pt]}l@{}}{\textbf{Diffusion and flow dynamics}}\\
\midrule
$f_\tau(x_\tau),\ f_t(x_t)$ & Drift of the noising / denoising SDE; often linear, $f_\tau=\alpha_\tau x_\tau$. \\
$\alpha_\tau$ & Coefficient of the linear drift. \\
$\sigma_\tau,\ \sigma_t$ & Diffusion coefficient (noise scale) of the forward / reverse SDE. \\
$W_\tau,\ W_t$ & Standard Wiener process; $\mathrm{d}W$ its increment. \\
$v_t(x),\ v(x_t,t)$ & Learned velocity field of the (rectified) flow; $v_t(x)=\mathbb{E}[x_{\mathrm{src}}-x_{\mathrm{data}}\mid x,t]$. \\
$\nabla\log p_t(x)$ & Score function of the marginal $p_t$. \\
$\nabla\log q_t(x)$ & Score of the policy / flow marginals; for Gaussian source available in closed form from $v_t$ (Prop.~C.3): $\nabla\log p_t(x)=-\big((1-t)v_t(x)+x\big)/t$. \\
$\tilde{\sigma}(t),\ \tilde{\sigma}_t$ & Auxiliary (free) diffusion coefficient used to recast the deterministic flow as a marginal-preserving SDE. \\
$\gamma_t\ge 0$ & Scalar scaling the contribution of the original diffusion in the SDE family (Thm.~C.1). \\
$G_t(x)$ & (Matrix) diffusion coefficient of a generic It\^o SDE. \\
$\tilde{G}_t(x)$ & Auxiliary diffusion matrix injecting additional stochasticity (Thm.~C.1). \\
$\bar{f}$ & Modified drift of an equivalent (same-marginal) SDE. \\
$\bar{G}$ & Modified diffusion of an equivalent SDE. \\

\addlinespace[0.4em]
\multicolumn{2}{@{}>{\columncolor{PaperTeal!4}[0pt][0pt]}l@{}}{\textbf{Bayesian posterior, cost, and constrained optimization}}\\
\midrule
$O\in\{0,1\}$ & Binary optimality indicator; $O=1$ flags trajectories meeting the optimality criterion. \\
$p(O=1\mid x)$ & Likelihood of optimality given a trajectory, $\exp(\beta\mathcal{J}(x))$. \\
$\beta<0$ & Inverse temperature controlling the strength of the cost-induced tilt. \\
$\beta_t$ & Inverse-temperature schedule (time-dependent $\beta$); $\partial\beta_t/\partial t$ enters the weight update. \\
$\mathcal{J}(x)$ & Penalty-augmented cost encoding the inference-time constrained problem. \\
$\mathcal{J}_c(x)$ & Penalized objective with explicit penalty weights $c$ (proof notation). \\
$\mathcal{L}(x)$ & Task objective to be minimized. \\
$\rho(x)$ & Additional penalty term 
\\
$h_1(x)$ & Equality-constraint function, $h_1(x)=0$. \\
$h_2(x)$ & Inequality-constraint function, $h_2(x)\le 0$. \\
$c_1,c_2>0$ & Penalty weights for the equality and inequality terms; $c=(c_1,c_2)$. \\
$L^{\star}$ & Optimal objective value over the feasible set, $\min_{x\in\mathcal{F}}\mathcal{L}(x)$. \\
$\mathcal{J}^{\star},\ J^{\star}$ & Minimum of the penalized objective over $\mathcal{X}$. \\
$\nabla\mathcal{J}(x)$ & Gradient of the cost (the cost-aware drift correction). \\
$\Delta\mathcal{J}(x)$ & Laplacian of the cost (appears in the general weight update). \\
$\mathrm{softmax}\{0,\cdot\}$ & A $C^2$ smooth-max relaxation of $\max\{0,\cdot\}$ (softplus); the hard clamp is used empirically. \\
$\beta_{\mathrm{sp}}$ & Sharpness parameter of the softplus relaxation, $\tfrac{1}{\beta_{\mathrm{sp}}}\mathrm{softplus}(-\beta_{\mathrm{sp}}\,c_{\mathcal{O}})$. \\
$\lambda_t$ & Guidance strength of the linear-combination baseline ($\nabla\log q_t+\lambda_t\nabla\mathcal{J}$, or $v_t+\lambda_t\nabla\mathcal{J}$); distinct from the jump rate $\lambda_t(x)$. \\

\addlinespace[0.4em]
\multicolumn{2}{@{}>{\columncolor{PaperTeal!4}[0pt][0pt]}l@{}}{\textbf{Feynman--Kac weights and SNIS}}\\
\midrule
$w_t$ & Log-importance-weight carried by a particle / trajectory; $\mathrm{d}w_t$ its increment. \\
$g_t(x)$ & Reweighting function: the discrepancy between the simulated dynamics and the target density evolution. \\
$\bar{g}_t(x)$ & Centered reweighting function, $\bar{g}_t=g_t-\int g_t\,p_t^{\mathrm{FK}}\,\mathrm{d}x$, ensuring normalization is preserved. \\
$a\in\mathbb{R}$ & Free design parameter redistributing the tilt between drift and weights; canonical choice $a=\beta_t\sigma_t^2/2$. (Distinct from the action waypoint $a$.) \\
$K$ & Number of Feynman--Kac particles (parallel trajectories). \\
$\{x_t^{(k)},w_t^{(k)}\}_{k=1}^{K}$ & Batch of simulated particles with their log-weights. \\
$w_T^{(k)}$ & Terminal log-weight of particle $k$ (used in SNIS). \\
$\phi(x)$ & Arbitrary test function whose expectation under the target is estimated (often $\phi(x)=x$). \\
$\mathbb{E}_{p_t^{\mathrm{FK}}}[\phi(x_t)]$ & Target expectation estimated by self-normalized importance sampling (SNIS). \\

\addlinespace[0.4em]
\multicolumn{2}{@{}>{\columncolor{PaperTeal!4}[0pt][0pt]}l@{}}{\textbf{Resampling and jump-process interpretation}}\\
\midrule
$\lambda_t(x)$ & Jump-process rate function (frequency of jump events), $\lambda_t(x)=(g_t(x)-\mathbb{E}_{p_t}[g_t])^{-}$. \\
$J_t(y\mid x)$ & Markov transition kernel for the next state when a jump occurs (a sampling kernel, \emph{not} a cost). \\
$\mathbb{E}_{p_t}[g_t]$ & Population mean of the reweighting function under $p_t$. \\
$w_t^{(k)}=g_t(x_t^{(k)})\,\mathrm{d}t$ & Per-step weight increment used for systematic resampling. \\
$[t_{\min},t_{\max}]$ & ``Active interval'' over which resampling is performed; weights are zeroed outside it. \\

\addlinespace[0.4em]
\multicolumn{2}{@{}>{\columncolor{PaperTeal!4}[0pt][0pt]}l@{}}{\textbf{Robot body, obstacles, and collision cost}}\\
\midrule
$o$ & Current sensor observation (typically RGB cameras); always conditioned on, suppressed in notation. \\
$a$ & A predicted action waypoint (end-effector pose / delta or joint command). Distinct from the FK design parameter $a$. \\
$\mathcal{B}(a)$ & Set of $N_{\mathrm{body}}$ sampled robot-body points in the world frame for action $a$. \\
$p_i(a)$ & The $i$-th sampled body point, $p\in\mathbb{R}^3$. \\
$N_{\mathrm{body}}$ & Number of sampled robot-body points. \\
$H$ & Trajectory horizon (number of predicted action waypoints); $x=(a_1,\dots,a_H)$. \\
$a_h$ & Action waypoint at horizon step $h$. \\
$\mathcal{O}\subset\mathbb{R}^3$ & Obstacle region (calligraphic; distinct from the optimality indicator $O$). \\
$\mathcal{O}_k$ & The $k$-th obstacle specified at inference time. \\
$\partial\mathcal{O}$ & Obstacle boundary, the zero-level set $\{p:s_{\mathcal{O}}(p)=0\}$. \\
$s_{\mathcal{O}}(p),\ \mathrm{sdf}_{\mathcal{O}}(p)$ & Signed distance function of $\mathcal{O}$ ($>0$ outside, $<0$ inside). \\
$\hat{s}(p)$ & Online SDF of the scene built from depth (nvblox) for the $\pi_{0.5}$ experiments. \\
$c_{\mathcal{O}}(p)$ & Clearance of a body point, $s_{\mathcal{O}}(p)-\varepsilon$. \\
$\varepsilon>0$ & SDF safety margin inflating the obstacle (a few cm). Distinct from the proof's energy gap $\varepsilon$. \\
$\mathrm{viol}_{\mathcal{O}}(p)$ & Per-point inward penetration, $\max\{0,\varepsilon-s_{\mathcal{O}}(p)\}$. \\
$\mathcal{J}_{\mathrm{coll}}(x)$ & Trajectory-level collision penalty summing $\mathrm{viol}$ over body points, waypoints, and obstacles. \\
$p_{xy}$ & Planar ($xy$) projection of a body point (cylinder SDF). \\
$(c_x,c_y)$ & Center of a cylindrical obstacle. \\
$r,\ R$ & Radius of the cylindrical obstacle. \\

\addlinespace[0.4em]
\multicolumn{2}{@{}>{\columncolor{PaperTeal!4}[0pt][0pt]}l@{}}{\textbf{Evaluation metric (ablations)}}\\
\midrule
$N$ & Number of evaluation episodes. \\
$e_t^{(i)}\in\mathbb{R}^3$ & Executed end-effector position at simulator step $t$ of episode $i$. \\
$T_i$ & Number of simulator steps in episode $i$. \\
$p_t^{(i)}$ & Instantaneous penetration depth $\max\{0,-s_{\mathcal{O}}(e_t^{(i)})\}$ (zero inflation, $\varepsilon=0$). \\
$P^{(i)}$ & Per-episode total penetration, $\sum_t p_t^{(i)}$. \\
$\bar{P}$ & Mean total penetration over episodes, $\tfrac{1}{N}\sum_i P^{(i)}$. \\
$\Delta t$ & Simulator step duration (e.g., $0.05$\,s). \\

\end{longtable}

\paragraph{Notes on overloaded / context-dependent notation}
A handful of symbols carry different meanings depending on context; we collect
them here to avoid confusion.
\begin{itemize}
  \item \emph{Time conventions for $t$.} For diffusion (\cref{sec:back_diff_policy}) we set $t=1-\tau$,
        so generation starts at $x_{t=0}\sim\mathcal{N}(0,I)$ (noise) and is integrated
        forward to $x_{t=1}$ (data). For the flow setup (Sec.~2.2) the endpoints are
        reversed: $x_{t=0}=x_{\mathrm{data}}$ and $x_{t=1}=x_{\mathrm{src}}\sim\mathcal{N}(0,I)$,
        and inference integrates from $t=1$ to $t=0$.
        The sampling SDEs and Algorithm~1 instead use increasing
        generation time and evaluate the flow model at $1-t$.
  \item \emph{$p_t(\cdot)$} denotes a generic marginal density; in the appendix it is
        overloaded to mean any marginal, and may denote the posterior marginal where
        stated explicitly.
  \item \emph{$a$} is the free design parameter of the Feynman--Kac drift correction
         and, separately, a predicted action waypoint in the experiments.
  \item \emph{$\varepsilon$} is the SDF safety margin in the collision cost and,
        separately, the strictly positive energy gap in the concentration proof.
  \item \emph{$\lambda_t$} is the scalar guidance strength of the linear-combination
        baseline, whereas $\lambda_t(x)$ is the rate function of the jump process.
  \item \emph{$O$ vs.\ $\mathcal{O}$:} $O$ is the binary optimality indicator;
        $\mathcal{O}$ (calligraphic) is the obstacle region.
  \item \emph{$\mathcal{J}$ vs.\ $J_t$:} $\mathcal{J}$ (and $\mathcal{J}_c$,
        $\mathcal{J}_{\mathrm{coll}}$) is the penalty-augmented cost; $J_t(y\mid x)$ is the
        Markov jump-transition kernel.
  \item \emph{$r$ vs.\ $R$} both denote the radius of a cylindrical obstacle in
        different subsections; \emph{$c$} appears both as the penalty vector
        $(c_1,c_2)$ and as the clearance $c_{\mathcal{O}}$.
\end{itemize}